%% file: Ranking_OR.tex
\documentclass[opre,nonblindrev]{informs3}

\DoubleSpacedXI 

\usepackage{endnotes}
\let\footnote=\endnote
\let\enotesize=\normalsize
\def\notesname{Endnotes}%
\def\makeenmark{\hbox to1.275em{\theenmark.\enskip\hss}}
\def\enoteformat{\rightskip0pt\leftskip0pt\parindent=1.275em
  \leavevmode\llap{\makeenmark}}

\usepackage{graphicx} 

\usepackage{smile}

\newcommand{\bu}{\bm{u}}

\newcommand{\by}{\bm{y}}
\newcommand{\bz}{\bm{z}}

\newcommand{\Rb}{\mathbf{R}}

\newcommand{\cC}{\mathcal{C}}

\newcommand{\cE}{\mathcal{E}}

\newcommand{\cG}{\mathcal{G}}
\newcommand{\cH}{\mathcal{H}}
\newcommand{\cI}{\mathcal{I}}

\newcommand{\cL}{\mathcal{L}}
\newcommand{\cM}{\mathcal{M}}

\newcommand{\cO}{\mathcal{O}}

\newcommand{\cR}{\mathcal{R}}

\newcommand{\cV}{\mathcal{V}}

\newcommand{\EE}{\mathbb{E}}

\newcommand{\PP}{\mathbb{P}}

\newcommand{\btheta}{\bm{\theta}}

\newcommand{\bSigma}{\bm{\Sigma}}

\newcommand{\tr}{\mathop{\mathrm{tr}}}

\usepackage[dvipsnames]{xcolor}
\usepackage{bm}
\usepackage{amsmath}
\usepackage{amssymb}
\usepackage{float}
\RequirePackage[colorlinks,
            linkcolor=red,
            anchorcolor=blue,
            citecolor=blue,
            urlcolor=blue
            ]{hyperref}

       \let\hat\widehat
       \let\tilde\widetilde

\def\given{\,|\,}

\def\Biggiven{\,\Big{|}\,}
\def\tr{\mathop{\text{tr}}\kern.2ex}

\def\P{{\mathbb P}}
\def\E{{\mathbb E}}

\def\supp{\mathop{\text{supp}}}

\long\def\comment#1{}

\def\tr{\mathop{\text{Tr}}}

\newcommand{\bel}{\begin{eqnarray}\label}
\newcommand{\eel}{\end{eqnarray}}
\newcommand{\bes}{\begin{eqnarray*}}
\newcommand{\ees}{\end{eqnarray*}}

\def\real{{\mathbb{R}}}
\def\R{{\real}}

\def\mid{\,|\,}

\def\P{{\mathbb P}}
\def\E{{\mathbb E}}

\def\supp{\mathop{\text{supp}\kern.2ex}}

\def\given{{\,|\,}}

\def\T{{\mathcal{T}}}

\newcommand{\op}{o_{\raisemath{-1.5pt}\PP}}

\usepackage{natbib}
 \bibpunct[, ]{(}{)}{,}{a}{}{,}%
 \def\bibfont{\small}%
 \def\bibsep{\smallskipamount}%
\TheoremsNumberedBySection  
\ECRepeatTheorems

\EquationsNumberedBySection 

\usepackage{bookmark}
\bookmarksetup{startatroot}

\begin{document}


\RUNAUTHOR{Liu, Fang and Lu}

\RUNTITLE{Lagrangian Inference for Ranking Problems}

\TITLE{Lagrangian Inference for Ranking Problems}


\ARTICLEAUTHORS{%

\AUTHOR{Yue Liu}
\AFF{Department of Statistics, Harvard University, Boston, MA 02138,  \EMAIL{yueliu@fas.harvard.edu}}\AUTHOR{Ethan X. Fang}
\AFF{Department of Statistics, Pennsylvania State University, University Park, PA 16802,  \EMAIL{xxf13@psu.edu}}
\AUTHOR{Junwei Lu}
\AFF{Department of Biostatistics, Harvard T.H. Chan School of Public Health, Boston, MA 02130,  \EMAIL{junweilu@hsph.harvard.edu}}
} 

\ABSTRACT{%
We propose a novel combinatorial inference framework to conduct general uncertainty quantification in ranking problems. We consider  the widely adopted  Bradley-Terry-Luce (BTL) model, where each item is assigned a positive preference score that determines the Bernoulli distributions of pairwise comparisons' outcomes. Our proposed method aims to infer general ranking properties of the BTL model. 
The general ranking properties include the ``local'' properties such as if an item is preferred over another and the ``global'' properties such as if an item is among the top $K$-ranked items. We further generalize our inferential framework to multiple testing problems where we control the false discovery rate (FDR), and apply the method to infer the top-$K$ ranked items. We also derive the information-theoretic lower bound to justify the minimax optimality of the proposed method. We conduct extensive numerical studies using both synthetic and real datasets to back up our theory. 
}%



\maketitle

%


\input{introduction}
\input{method}
\input{theory}
\input{simulation}

\input{discussion}
\begin{APPENDICES}

\input{appendix}

\end{APPENDICES}



\bibliographystyle{informs2014} 
\bibliography{ranking_reference} 



\end{document}

%% file: introduction.tex

\section{Introduction}\label{sec:intro}

Ranking problems aim to study relative orderings of some set of items, and 
find many applications such as  sports competition  \citep{massey1997statistical,pelechrinis2016sportsnetrank, price2017scores,xia2018network}, online gamers ranking (e.g., Microsoft TrueSkill ranking system \cite{minka2007trueskill, minka2018trueskill}), web search and information retrieval \citep{dwork2001rank, bouadjenek2013sopra, guo2020deep}, 
 recommendation systems \citep{baltrunas2010group, he2018adversarial,  geyik2019fairness},
  crowdsourcing \citep{chen2013pairwise, suh2017adversarial, liang2020online}, gene ranking \citep{boulesteix2009stability, kolde2012robust, kim2015hydra}, {assortment optimization \citep{li2018integrating, aouad2018approximability}}, among many others. Due to the practical importance, ranking problems draw significant attention from different communities such as operations research \citep{mcfadden1973conditional, mohammadi2020ensemble}, statistics \citep{hunter2004mm, chen2020partial}, machine learning \citep{richardson2006beyond, guo2020deep}, and sociology \citep{brown2003ranking, subochev2018ranking}.

In ranking problems, given some comparisons among pairs of  $n$ items, we aim to infer the relative ranking of these items. Many models {are} proposed to study this problem, and one  of the most widely used parametric models is the Bradley-Terry-Luce (BTL) model \citep{bradley1952rank, luce1959individual}. In the BTL model, each item is assigned a latent positive preference score that determines its rank, and the latent scores determine the relative preference among the $n$ items.

 
Based on the BTL  model, there are several tracks of works that estimate the ranks of the items by estimating the latent scores. The first track is the rank centrality method   \citep{negahban2017rank, dwork2001rank, maystre2015fast, vigna2016spectral, jang2016top}, which is also known as the spectral method. This class of methods connects pairwise comparisons with random walk over the comparison graph. In particular, each node in the graph represents an  item, and  the probability of moving from node~$i$ to node $j$ equals the probability that item $j$ is preferred over item $i$. Based on this approach, \cite{negahban2017rank} { show} that the preference scores of items {coincide} the stationary distribution under the random walk, and derive a fast rate of convergence of the estimator for the scores. {\cite{chen2019spectral} further { improve} the  convergence rate in \cite{negahban2017rank} by removing the logarithmic factor.}
The second track is based on considering the regularized maximum likelihood estimator (MLE) \citep{ford1957solution, hunter2004mm, lu2015individualized}. This approach estimates the latent scores by maximizing the regularized likelihood function, and \cite{chen2019spectral} derive the rate of convergence of the estimator under the $\ell_2$ regularization.  {\cite{negahban2018learning} also consider nuclear norm  regularization.} In addition, \cite{azari2013generalized} consider the method of moments for the Plackett-Luce model, and \cite{mosteller2006remarks, jiang2011statistical, neudorfer2018predicting} consider general least square methods, and their estimation consistency for preference scores are established by various works~\citep{chen2015spectral,maystre2015fast,jang2016top,duchi2010consistency, rajkumar2014statistical, rajkumar2016can}.

Despite the aforementioned significant progress of rank estimation, the uncertainty quantification in ranking problems remains largely unexplored, which is of crucial importance in practice. For example,  saying that player $i$ is ranked higher than player $j$ without a confidence score is not very informative in practice. In this paper, we propose a novel combinatorial inferential framework for testing ranking properties. In particular, given $n$ items, we define a ranking list $\gamma$ as a permutation over the set of $n$ items $[n] = \{1,...,n\}$, and let $\cR$  be the set of all possible rankings (i.e., all possible permutation over $[n]$). Let ranking list $\gamma^{*}$  be the true underlying ranking of $n$ items.  We aim to test whether  $\gamma^{*}$ satisfies certain ranking properties based on partial pairwise comparison observations. For example, let $\cR_i$ be a subset of $\cR$ representing the ranking property with respect to item $i$.
We test the general  ranking property for a given item $i$,  i.e., whether item $i$ has certain properties, i.e., $${H}_{0}:  \text{item } i \text{ does not satisfy the property}   \text{ \ \ v.s. \ \  } {H}_{a}: \text{item } i \text{  satisfies the property},$$ which is equivalent to
$${H}_{0}:  \gamma^* \notin \cR_i
\text{ \ \ v.s. \ \  } 
{H}_{a}: \gamma^* \in \cR_i.$$

\subsection{Motivating Applications}

The inference in ranking problems finds many applications. 
For instance, it is of practical interest to test whether movie
$A$ is preferred over movie $B$ on average, and test whether chess player~$C$ is stronger than player $D$. Such problems are pairwise ranking inference problems that fit into our framework as defined in the following example.

\begin{example}[Pairwise ranking inference] \label{eg:intro-1}
	Consider testing whether  item $i$ is ranked higher than item $j$. Let $\mathcal{R}_i$ be the set of all possible rankings that item $i$ is ranked higher than item $j$. We consider the following hypothesis testing problem that
	$${H}_{0}: \text{Item } j \text{ is ranked higher than item } i  \text{ \ \ v.s. \ \  }  { H}_{a}: \text{Item } i \text{ is ranked higher than item } j.$$
\end{example}  

Another important application of inference in ranking problems is the top-$K$ {inference}.
 For instance, in recommendation systems, one important goal is to find a few most appealing items for the users \citep{cremonesi2010performance}.  In biomedical studies, only a small subset of top-ranked genes is informative, and it is crucial for the investigators to identify this set of genes to perform detailed analysis \citep{boulesteix2009stability}. {In assortment optimization, the challenge is to identify a subset of items that maximize revenue based on customer preferences \citep{li2018integrating, aouad2018approximability}.
 } We first summarize the single top-$K$ inference problem in the following example.

\begin{example}[Single top-$K$ inference] 
	\label{eg:intro-2}
	Consider testing whether  item $i$ is among the top-$K$ items (a special case is $K=1$). Here $\mathcal{R}_i$ is the set of all possible rankings that item $i$ is among the top-$K$ items. We consider the following hypothesis testing problem that
	$${H}_{0}: \text{Item } i \text{ is not among the top-$K$ items}  \text{ \ \  v.s. \ \  }  {H}_{a}: \text{Item } i \text{ is among the top-$K$ items}.$$
\end{example}

We then extend the problem to the multiple testing setup, where the goal is to infer the set of all top-$K$ items.

\begin{example}[Top-$K$ inference] 
	\label{eg:intro-3}
	Consider the problem of identifying the set of top-$K$ items. Here $\mathcal{R}_i$ is the set of all possible rankings that item $i$ is among the top-$K$ items, {$i \in [n]$}. We consider the following multiple testing problem that 
	$${H}_{0}: \text{Item $i$ is not among the top-$K$  items} \text{ \ v.s.  \  }  {H}_{a}: \text{Item $i$ is among the top-$K$ items,}   \text{ \ for\  } i \in [n].$$
\end{example}

\subsection{Major Contributions}
To the best of our knowledge, this paper provides the first inferential framework for ranking problems. Our proposed method can test a broad class of hypotheses for ranking problems. Theoretically, we show that the p-values are valid, and our procedures are powerful. We summarize the major contributions below.
 
\begin{itemize}
\item We are among the first to study general inferential approaches for ranking problems beyond estimation, and we propose a novel general framework for inferring different ranking properties. 
 We  show that our proposed methods are asymptotically valid and powerful. Furthermore, we generalize the method to the more challenging multiple testing setup to widen the applicability.

\item In our inferential framework, we  propose a novel general Lagrangian debiasing procedure to handle the constrained parameter space. Our Lagrangian debiasing procedure addresses the challenge raised  by the non-identifiability of the BTL model. Most existing works on high-dimensional inference, such as \cite{zhang2014confidence,
van2014asymptotically,ning2017general}, focus on inferring the parameter under the unconstrained space, and do not apply to the BTL model due to the constraints. By considering the optimality condition of the Lagrangian dual problem, our proposed approach provides a new tool for high-dimensional inference under general constraints. We also  derive the asymptotic distribution of our debiased estimator.
{We point out that this new Lagrangian debiasing procedure can be applied to general high-dimensional constrained inferential problems beyond ranking problems, which itself is of great interest.}
  
\item We provide a new framework to derive the minimax lower bound for multiple testing in ranking problems, which provides new theoretical insights.
To the best of our knowledge, this is the first time that such a lower bound is derived. In particular, {let the preference score vector be $\theta \in \mathbb{R}^n$, which represents the scores of all $n$ items that determine the ranks of all items.} We first define a new minimax risk for the multiple testing problems that 
\begin{equation} \nonumber
\mathfrak{R} = \inf_{\psi}  \sup_{ \theta} \P(\symbol{"0023} \text{ false positives}+ \symbol{"0023} \text{ false negatives}\geq 1),
\end{equation} 
where the infimum is taken over all possible selection procedure $\psi$. Here the risk is the probability of making at least one Type I or Type II error. If the minimax risk $\mathfrak{R} \geq\ 1-\epsilon$ for some constant $\epsilon>0$, we say that any  procedure fails for the multiple testing problem since they cannot control the Type I error or Type II error in the minimax sense. 
{
To derive the necessary conditions for controlling minimax risk, we further define a novel distance
$\Delta({\theta})$ in \eqref{equ:distance}, 
and a divider set $\mathcal{M}({\theta})$ in Definition~\ref{def:divider}, which capture the combinatorial structures of ranking properties. Intuitively, the distance is a  signal strength for  selecting the items of interest; the divider set is the set of items that are crucial for  selecting the items, and the size of the divider set increases as the distance decreases. 
We  show that the numerical signal strength $\Delta({\theta})$ and combinatorial signal strength $|\mathcal{M}({\theta})|$ together measure the difficulty in the multiple testing problems. 
We also give two concrete examples where  $\mathfrak{R}$ is arbitrarily close to 1 if $ \Delta({\theta}) \lesssim  
\sqrt{\frac{\log n}{npL}}$.
In addition, we show that our lower bound matches our upper bound to justify the optimality of the proposed method. 
}
\end{itemize} 

\subsection{Literature Review}

\noindent{\bf Ranking Problem.}
 There has been a long history of works on ranking problems \citep{mallows1957non, keener1993perron,  altman2005ranking, jiang2011statistical,  osting2013enhanced, vigna2016spectral, ding2018new, filiberto2018new, guo2020deep, pujahari2020aggregation}. 
Some ranking systems are based on explicit preference scores or ratings provided by individuals, which is closely related to the matrix completion problem \citep{candes2009exact, negahban2012restricted}. In these problems, an individual only provides scores for a subset of items, and we estimate the individual's preference scores for other items.   
However, users' explicit scores can be inconsistent and noisy, or even not available in some cases. This motivates researchers to develop methods for ranking aggregations from comparison results or partial rankings provided by users \citep{saaty2003decision, ailon2008aggregating, ailon2010aggregation, gleich2011rank, ammar2011ranking, farnoud2012novel,   volkovs2012flexible,  ammar2012efficient, swain2017consensus, napoles2017weighted, jang2017optimal, jang2018top, chen2018nearly, zhang2020ranking}. 

Another important track of works on ranking problems are based on pairwise comparison data \citep{kendall1940method, kendall1955further,  adler1994selection, talluri2006theory,chen2017asymptotically, beutel2019fairness, jain2020spectral, chen2021optimal}. 
For instance, \cite{lu2011learning} study the Mallows model from pairwise comparisons. \cite{chen2017asymptotically} study the sequential design with pairwise comparisons.  \cite{chen2015spectral} propose a new two-step method called the {spectral MLE}, and prove that it is minimax optimal. 
\cite{jang2016top} show that the spectral method itself is optimal for identifying the top-$K$ items in the sense of achieving the minimal sample size. \cite{chen2019spectral}
further study the sample complexity of regularized MLE and spectral method in a sparse pairwise comparison setting. 

There are other general frameworks on ranking problems such as the
Thurstone model \citep{thurstone1927law, vojnovic2017parameter, orban2019generalization, Jin_Xu_Gu_Farnoud_2020} and Plackett-Luce model \citep{guiver2009bayesian,hajek2014minimax}.  For instance, \cite{Jin_Xu_Gu_Farnoud_2020} propose a heterogeneous Thurstone model capturing heterogeneity of different individuals, and propose an algorithm to estimate the preference score vector and heterogeneity.  
Beyond parametric models, there are also  nonparametric  methods for ranking problems. For instance, \cite{shah2017simple} analyze a simple counting algorithm proposed by \cite{copeland1951reasonable}, which counts the numbers of wins of each item, and show its optimality and robustness. \cite{shah2016stochastically}, \cite{chen2018optimal}, and \cite{pananjady2017worst} consider the strong stochastically transitive (SST) model for pairwise comparisons. 
Furthermore, some other works consider ranking problems under specific settings such as active-ranking  \citep{ jamieson2011active, busa2013top, heckel2019active}, and crowd-sourcing \citep{chen2013pairwise, chen2016bayesian,  suh2017adversarial, liang2020online}. 
However, we point out the all above works focus on the estimation problem, and do not consider uncertainty quantification and inferential methods in ranking.  One exception is \cite{hall2009using}. This work focuses on using $m$-out-of-$n$ bootstrap to estimate the distribution of an empirical rank, which  requires empirical choice of $m$ and is of less practical interest. In contrast, we provide a more general framework that solves the problems of practical interest.

\vspace{4pt}

\noindent{\bf Constrained Inference.}  The inference under equality or inequality constraints is of great interest in literature. The low-dimensional constrained inference dates back to \cite{chernoff1954distribution}, which proves that the likelihood ratio weakly converges to a weighted chi-square distribution for constrained testings. Under the low-dimensional setting, the constrained inference has been further studied in \cite{gourieroux1982likelihood, kodde1986wald, rogers1986modified, shapiro1988towards, wolak1989testing, molenberghs2007likelihood, susko2013likelihood}, among many others.  
 For the high-dimensional constrained inference,  \cite{yu2019constrained} assume the existence of natural constraint on parameters, and test whether the parameters lie on the boundary of the constraint. By applying the debiasing approach in \cite{ning2017general}, the authors study the asymptotic distribution of test statistics under constraints. 

We note that the above mentioned methods cannot be applied to solve our problem. This is mainly due to the unique challenge that in our setting, the Fisher information matrix is singular due to the non-identifiability issue.


\vspace{4pt}

\noindent{\bf Paper Organization.}
The rest of our paper is organized as follows. In Section~\ref{sec:setup}, we introduce some preliminaries of ranking problems and  some ranking properties.
In Section~\ref{sec:method}, we present our debiased estimator with constraints. We then provide the { general hypothesis testing}. In Section~\ref{sec:multi test}, we extend our method to handle  multiple testing problems. In Section \ref{sec:LB_FWER}, we present the lower bound theory with applications to several examples.
We provide numerical results in Section \ref{sec:simulation} and some discussions in Section \ref{sec:discussion}.

\vspace{4pt}

\noindent{\bf Notations.}
Let $|A|$ represent the cardinality of set $A$, and $[n]$ represent the set of $\{1, \cdots, n\}$ for $n \in \mathbb{Z}^{+}$. For vector $v=\left(v_{1}, \ldots, v_{d}\right)^{T} \in \mathbb{R}^{d}$, and $1 \leq q \leq \infty,$ we define norm of $v$ as $\|v\|_{q}=\big(\sum_{i=1}^{d}\left|v_{i}\right|^{q}\big)^{1 / q}$. In particular,
 $\|v\|_{\infty}=\max _{1 \leq i \leq d}\left|v_{i}\right|$. For a matrix $M=\left[M_{ij}\right]$, 
let $\ell_1$-norm $\|M\|_{1}=\max _{j} \sum_{i}\left|M_{ij}\right|$, $\ell_\infty$-norm $\|M\|_{\infty}=\max _{i} \sum_{j}\left|M_{i j}\right|$, and the operator norm
$\|M\|_{2}=\sigma_{\max }(M)$
where $\sigma_{\max }(M)$ represents the largest singular value of matrix $M$. In addition, $a_n=O(b_n)$ or $a_n \lesssim b_n$ means there exists a constant $C>0$ such that $a_n \leq Cb_n$, and $a_n=o(b_n)$ means
$\lim_{n \rightarrow \infty} \frac{a_n}{b_n}=0$. we write $a_{n} \asymp b_{n}$ if $C \leq a_{n} / b_{n} \leq C^{\prime}$ for some $C, C^{\prime}>0$. {For a sequence of random
variables $\{X_n\}$, we write $X_n \overset{d}{\longrightarrow} X$ if $X_n$ converges in distribution to the random variable $X$. }Throughout
the paper, we let $C, C_1, C_2, \cdots,c, c_1, c_2, \cdots$ be generic constants which may change in different places.


%% file: method.tex

\section{Preliminaries and Problem Setup}\label{sec:setup}
 In this section, we provide  some preliminaries to facilitate our discussions. We first briefly review the Bradley-Terry-Luce model and introduce our data generating scheme. Then, we provide the definitions of rankings and ranking properties.

\subsection{Bradley-Terry-Luce Model} 
We consider the Bradley-Terry-Luce (BTL) parametric model \citep{bradley1952rank, luce1959individual}. This model assumes  a  hidden preference score $\omega^*_i>0$ for each item $i$, $1 \leq i \leq n$. The scores determine the ranking and {the distributions of comparison results}. Let ${\omega}^*=(\omega^*_1,\cdots,\omega^*_n)^\top \in \mathbb{R}^n$ be the true preference score vector,  
 and its log-transformation is 
  $$
  {\theta}^*=(\theta^*_1,\cdots,\theta^*_n)^\top, \text{ where } \theta^*_i=\log \omega^*_i.
  $$ 
 Here $\omega^*_j>\omega^*_i$ or $\theta^*_j>\theta^*_i$ means that item $j$ is ranked higher (preferred) than item $i$. 
In this paper, for ease of presentation, we consider the case that all scores are in a bounded domain that $w_{i}^{*} \in\left[w_{\min }, w_{\max }\right]$ for all $i\in[n]$, where $ w_{\min},\ w_{\max}>0$. We let $\kappa=\omega_{\max} / \omega_{\min}$ be the condition number, and $\kappa$ is a constant which does not depend on $n$.

When we collect data, we compare the items pairwisely. To model the random pairs for comparisons, we adopt the Erd\"{o}s-R\'{e}nyi random graph. In particular, suppose we have an 
 undirected graph  $\cG = \{\cV,\cE\}$ where $\cV = [n]$ is the vertex set, and {$\cE \subseteq [n]\times [n]$} is the edge set.   In the Erd\"{o}s-R\'{e}nyi random graph $\cG (n,p)$, each edge is drawn independently from a Bernoulli distribution with probability $p$. Here we assume that for each pair $(i,j)\in\cE$, we observe the comparisons $L$ times. Note that we assume for all pairs in $\cE$, we have {a same} number of observations for ease of presentation, but our proposed method can be easily generalized to handle the general setting  where we have different numbers of observations of different pairs. Denote by $y_{i,j}^{(\ell)}$ the $\ell$-th comparison between items $i$ and $j$ for some $i<j$, which depends only on the relative scores of the  two items. 
 We assume that each  $y_{i,j}^{(\ell)}$ is generated independently from a Bernoulli distribution that 
\begin{equation}\label{eqn:y}
y_{i,j}^{(\ell)} \stackrel{\text { ind. }}{=}\left\{\begin{array}{ll}
1, & \text { with probability } \frac{w_{j}^{*}}{w_{i}^{*}+w_{j}^{*}}=\frac{e^{\theta_{j}^{*}}}{e^{\theta_{i}^{*}}+e^{\theta_{j}^{*}}} \\
0, & \text { otherwise,}
\end{array}\right.
\end{equation}
where $y_{i,j}^{(\ell)}=1$ means item $j$ is preferred over item $i$. Here we assume that all $y_{i,j}^{(\ell)}$'s are independent for all $i, j$, $(i<j)$, and $\ell$.


We point out that the BTL  model is invariant that if we multiply $\omega_i^*$, or increase $\theta^*$, by a constant~$c$,  the distribution of ${y}_{i,j}^{(\ell)}$ does not change.  That is, ${\theta}^*$ and ${\theta}^*+c=(\theta^*_1+c,\cdots,\theta^*_n+c)^\top$ are observationally equivalent. Hence, we regard a score vector ${\theta}^* \in \mathbb{R}^n$ as an equivalence class $[{\theta}^*]=\{{\theta}':{\theta}^*+c, c \in \mathbb{R}\}$, and regard the parameter space as the set of equivalence classes of $\mathbb{R}^n$
\citep{hunter2004mm,negahban2017rank}.  For the identifiability of the parameters, we impose a  constraint on the parameter space $\mathcal{C}$ that we let $\mathcal{C}=\{\theta: f(\theta)=0\}$, and propose our inferential framework under this constraint, where the function $f$ is smooth, and ensures the identifiability of the parameter $\theta$. Specific examples of $f$ include $\bm{1}^\top{\theta}=0$, ${\theta}_1=1$ (${\theta}_1$ is the preference score of the first item), among others \citep{negahban2017rank,Jin_Xu_Gu_Farnoud_2020,chen2013pairwise}.

\subsection{Ranking and Ranking Property}
As discussed in the introduction, our goal is to infer some general ranking properties { based on the BTL model using} samples of pairwise comparisons among all items. We first provide the formal definition of the ranking and its properties.
\begin{definition}[Ranking]
	\label{def:ranking}
	Assume there are $n$ items. Let $\cR$  be all  bijections from the set $[n]$ onto itself. 
	Then each $\gamma \in \cR$ is a possible rank of the $n$ items. Let $\gamma_i$ be the rank of item $i$ in ranking~$\gamma$, where $\gamma_i < \gamma_j$ if item $i$ is ranked higher (preferred) than item $j$.  Let $\gamma^* \in \cR$  be the  true ranking of these $n$ items. Finally, we let $\gamma({\theta})$ be the induced ranking from the underlying preference score vector ${\theta}$.
\end{definition}

When we are interested in some ranking  property of a given item, we are essentially interested in testing if the ranking satisfies some properties as we discussed in the introduction. Thus, we infer if the true ranking belongs to some set of rankings. To facilitate our discussion,  we define the equivalent rankings and ranking properties with respect to a single item below.

\begin{definition}[Equivalent rankings with respect to a single item]
	\label{def:equi_ranking}
	Rankings $\gamma$ and $\gamma' \in \cR$  are equivalent  with respect to item $i$ if 
	\[
	\{j \in [n]: \gamma_j < \gamma_i \} = \{j \in [n]: \gamma'_j < \gamma'_i \},
	\]
	or equivalently, 
	\[
	\{j \in [n]: \gamma_j > \gamma_i \} = \{j \in [n]: \gamma'_j >\gamma'_i \}.
	\]
	Furthermore, we  let the equivalent class of a ranking $\gamma$ with respect to item $i$ be $\cR_{\gamma, i} = \{\gamma' \in \cR: \gamma' \text{ is equivalent to } \gamma {\text{ with respect to item } i} \}$.
\end{definition}

\begin{definition}[Ranking property with respect to a single item]
	\label{def:ranking property}
	 A ranking property $\cR_i$ with respect to a single item $i$ is a set of rankings such that $\cR_i\subset \cR$, and for any ranking $\gamma \in\ \cR_i$, its equivalent class $\cR_{\gamma,i}$ satisfies $\cR_{\gamma,i} \subseteq \cR_i$ .
\end{definition}

{Essentially, the ranking property $\cR_i$ with respect to item $i$ is a subset of all possible rankings $\cR$, and a collection of disjoint equivalent classes.} 
Specific examples of ranking property with respect to item $i$ include Examples \ref{eg:intro-1} and \ref{eg:intro-2} where we are interested in testing if item $i$ is preferred over another given item, or item $i$ is ranked within top-$K$. 

{\begin{itemize}
	\item Example \ref{eg:intro-1}: (Pairwise preference between item $i$ and item $j$). We aim to test if item $i$ is ranked higher than item $j$, which means $\gamma^*_i <\gamma^*_j$  or  $\theta^*_i>\theta^*_j$. Letting 	$$\cR_i = \{\gamma: \gamma_i < \gamma_j\} = \{\gamma(\theta): \theta_i > \theta_j\},$$ we show that $\cR_i$ is a ranking property as defined in Definition~\ref{def:ranking property}. 	
	If $\gamma \in \cR_i$, which means $\gamma_i <\gamma_j$, then for any ranking $\gamma'$ equivalent to $\gamma$ ({ i.e., $\gamma' \in \cR_{\gamma,i}$}), we have \[
	\{k \in [n]: \gamma_k < \gamma_i \} = \{k \in [n]: \gamma'_k < \gamma'_i \}.
	\]
	Thus, $\gamma'_i <\gamma'_j$, which means that $\gamma' \in \cR_i$ and further gives the equivalent class $\cR_{\gamma,i} \subseteq \cR_i$. We have that $\cR_i$ in this example satisfies  Definition \ref{def:ranking property}.

	\item Example \ref{eg:intro-2}: (Top-$K$ test). If item $i$'s preference score is larger than $n-K$ items, i.e., $\theta_i > \theta_{(K+1)}$, where $\theta_{(K+1)}$ denotes the $(K+1)$-th largest preference score, or equivalently, $\gamma_i \le K$. We aim to test if item $i$ is ranked among top-$K$ items. Thus, we have that $\cR_i$ is
	$$\cR_i=\{\gamma: \gamma_i \leq K\} = \{\gamma(\theta): \theta_i > \theta_{(K+1)}\}.$$	
To see that $\cR_i$ is a ranking property as defined in Definition~\ref{def:ranking property}, we have that If $\gamma \in \cR_i$, which means $\gamma_i \leq K$, then for any ranking $\gamma'$ equivalent to $\gamma$ ({i.e., $\gamma' \in \cR_{\gamma,i}$}), and we have
	$\{k \in [n]: \gamma_k < \gamma_i \} = \{k \in [n]: \gamma'_k < \gamma'_i \}$. Thus,  the ranking of item $i$ does not change, and we still have $\gamma'_i \leq K$, which means $\gamma' \in \cR_i$. We  have that $\cR_i$ is  a  ranking property satisfying  Definition~\ref{def:ranking property}.

\end{itemize}
}

In the next section, given a ranking property $\cR_i$, we propose a novel approach to test 
whether item $i$ satisfies this property that
$${H}_{0}:  \gamma^* \notin \cR_i
\text{ \ \ v.s. \ \  } 
{H}_{a}: \gamma^* \in \cR_i.$$

%

\section{Inference}
\label{sec:method}
{ In this section, we propose our inferential framework to test general ranking properties. The first step in our inferential framework is a novel Lagrangian debiasing method, which handles the general constrained inference with penalization, and we apply the method to infer the latent scores. We then adopt a Gaussian multiplier bootstrap approach to test general ranking properties. We conclude this section by showing that our method controls the Type I error, and is asymptotically powerful. }
\subsection{Lagrangian Debiasing Approach}\label{sec:debias}
We first propose a novel Lagrangian debiased estimator of the preference scores. 
Our proposed method is motivated from the regularized maximum likelihood estimator (MLE) approach. Assuming the BTL model, the MLE approach \citep{ford1957solution, hunter2004mm} provides an estimator for the latent preference scores by solving the following convex optimization problem
 \begin{equation}\label{eqn:thetahat}
 \hat\theta = \operatorname{argmin}_{{\theta} \in \mathbb{R}^{n}}  \cL_{\lambda_0}({\theta}):=\cL({\theta})+ \lambda_0\|{\theta}\|_{2}^{2},
 \end{equation}
where $\lambda_0>0$ is a tuning parameter, and the negative log-likelihood  function $\cL({\theta})$ is
\begin{equation}\label{eqn:likelihood}
\cL({\theta}) =\sum_{(i, j) \in \cE, i>j} \Big\{-y_{j, i}(\theta_{i}-\theta_{j})+\log \big(1+e^{\theta_{i}-\theta_{j}}\big)\Big\},
\end{equation}
 where $y_{j,i}=\sum_{\ell=1}^{L} y_{j,i}^{(\ell)}/L$.
 
{Note that the regularization guarantees that the obtained estimator satisfying $\bm{1}^T \hat{\theta} = 0$ \citep{chen2019spectral}, and the deduction of optimal rate of the obtained estimator relies on the strong convexity of $\cL_{\lambda_0}({\theta})$}. Different works study the theoretical guarantees of the MLE approach \citep{shah2015estimation, negahban2017rank, chen2019spectral, wang2020stretching}. In particular,  \cite{chen2019spectral} study the convergence rate of the estimator \eqref{eqn:thetahat} in terms of $\ell_\infty$-norm. For self-completeness, we provide the result below.

\begin{lemma} \label{lem:consistency}
Under the BTL model, suppose that $\kappa:=\frac{w_{\max }}{w_{\min }}< C$ for some constant $C>0$. If {the pairwise comparison probability $p$ in Erd\"{o}s-R\'{e}nyi graph satisfies } $p \geq \frac{C_{0} \log n}{n}$ for some sufficiently large constant $C_{0}>0$, and the regularization parameter $\lambda_0 = c_{\lambda_0}\sqrt{np\log n/L}$ for some constant $c_{\lambda_0}>0$, we have that the estimator $\hat\theta$ derived from the regularized MLE achieves the optimal  rate
	\begin{equation}\nonumber
	\|\hat{{\theta}}-{{\theta}^*}\|_\infty\lesssim \sqrt{\frac{\log n}{npL}}
	\end{equation} with probability at least $1-\cO(n^{-5})$. 
\end{lemma}
\begin{remark}
We point out that the same rate can also be achieved by the spectral method \cite{negahban2017rank}, and we provide the proof in Appendix Section \ref{sec:pf:consistency}.
\end{remark}

Since $\hat\theta$ is derived from a regularized MLE, conducting inference based on $\hat\theta$ is challenging. Over the past few years, the debiasing approach achieves great successes for penalized regression. For examples,  \cite{zhang2014confidence,
van2014asymptotically,
javanmard2014confidence,
javanmard2014hypothesis} study the debiasing approach based on linear or generalized linear models, and \cite{ning2017general} provide a decorrelation approach to { inferring} estimators derived from penalized MLE methods. 

However, these existing debiasing methods cannot be directly used in our problem. This is because that the parameter of interest is non-identifiable in the BTL model, and the Fisher information matrix is singular. To ensure the identifiability, as discussed in Section 2.1, we impose a constraint of the parameter that we let $\theta$ belongs to the set $\cC=\{\theta:f(\theta) = 0\}$ for some smooth function $f$. 
To handle the challenges raised by the constraint, we propose a general Lagrangian debiasing method in the next part.

 \subsubsection{Lagrangian Debiasing method} \label{sec:311}
We propose a general Lagrangian debiasing method for inference based on penalized MLE with constraints. Our method is motivated by \cite{ning2017general}, where the authors consider a one-step estimator by solving the first-order approximation of the score function $\nabla \cL(\hat{{\theta}}) + \nabla^2 \cL(\hat{{\theta}})({\theta}-\hat{{\theta}}) = 0$. To handle the constraint on the parameters that $f({\theta})=0$, we consider the Lagrangian dual function. In particular, under the constraint $f(\theta)=0$, the MLE method aims to solve the problem that
$$
\min_{\theta} {\cL}(\theta),\text{ subject to }f(\theta) = 0.
$$
The corresponding Lagrangian dual problem is 
$$
\max_{\lambda}\min_{\theta} \cL(\theta) + \lambda f(\theta),
$$
where $\lambda \in \R$ is the Lagrangian multiplier. Considering the { first-order optimality} condition of the Lagrangian dual problem, we have that an optimal dual solution pair $(\theta,\lambda)$ satisfies
\begin{equation}\label{eqn:solve}
\nabla\cL(\theta) + \lambda\nabla f(\theta) = 0, \text{ and }f(\theta) = 0.
\end{equation}

Based on the above equations, we propose our debiasing approach. In particular, given a penalized estimator $\hat\theta$ from \eqref{eqn:thetahat}, we obtain a debiased estimator $\hat\theta^d$ by solving the following system of equations of $\theta$ and $\lambda$, which are first-order approximations of \eqref{eqn:solve},
\begin{equation}
\label{equ:gen-debias}
\left\{\begin{array}{ll}
\nabla \cL(\hat{{\theta}})+ \nabla^2 \cL(\hat{{\theta}})({\theta}-\hat{{\theta}})+\lambda \nabla f(\hat{{\theta}})=0  \\
f(\hat{{\theta}})+ 
\nabla f(\hat{{\theta}})^\top ({\theta}-\hat{{\theta}})=0 
\end{array}\right.
\end{equation}
or equivalently,
\begin{equation}
\Bigg(\begin{array}{cc}
\nabla^{2} \cL(\hat{\theta}) & \nabla f(\hat{\theta}) \\
\nabla f(\hat{\theta})^\top & 0
\end{array}\Bigg)
\Bigg(\begin{array}{c}
{\theta}-\hat{\theta} \\
\lambda
\end{array}\Bigg) 
= \Bigg(\begin{array}{c}
-\nabla \cL(\hat{\theta}) \\
-f(\hat{\theta})
\end{array}\Bigg)
\end{equation}
See Figure \ref{fig:LDE} for  illustration. 
 	
 \begin{figure}[h]
 	\begin{center}
 		\begin{tabular}{c}		
 			\includegraphics[width=.7\textwidth,angle=0]{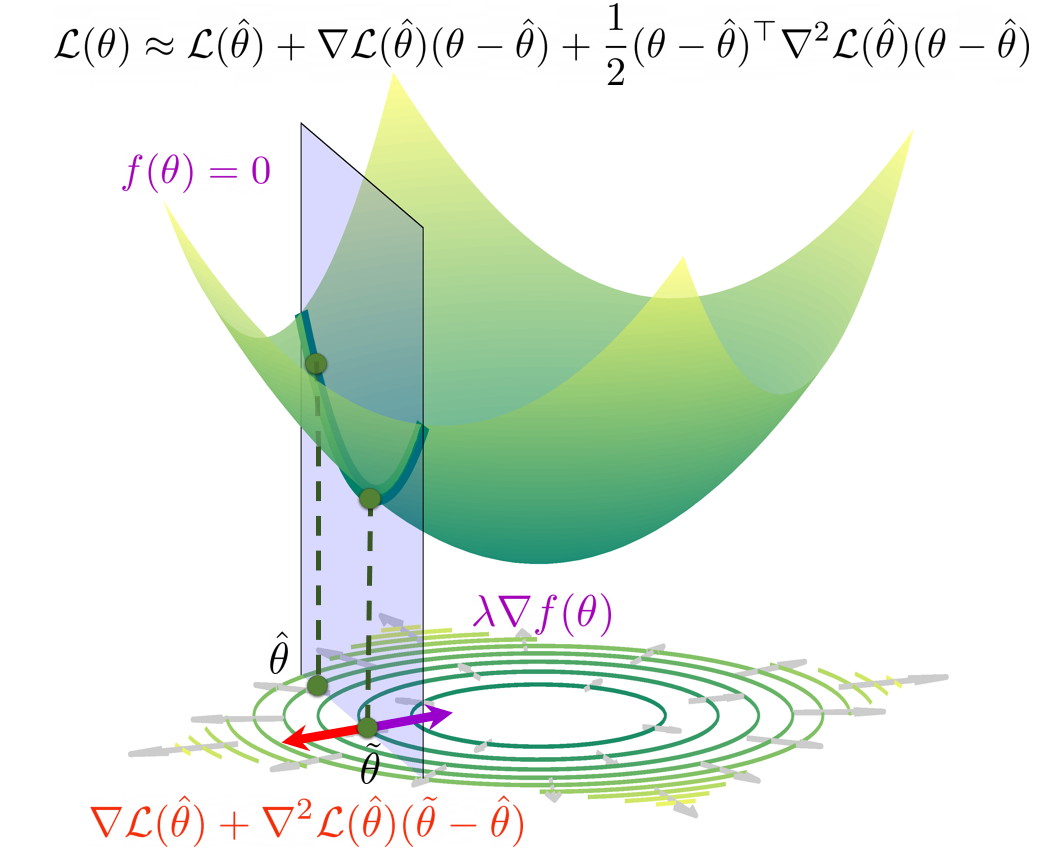} 
 		\end{tabular}
 	\end{center}
 	\caption{Geometric illustration of our Lagrangian debiasing method. The green surface is the approximation of the loss function $\mathcal{L}(\theta)$. The green circles are the contour lines of the surface. The transparent purple plane is the constraint $f(\theta)=0$. The red arrow  represents $\nabla \cL(\hat{{\theta}})+ \nabla^2 \cL(\hat{{\theta}})(\tilde{\theta}-\hat{{\theta}})$, and the purple arrow represents $\lambda \nabla f(\theta)$.} 
 	\label{fig:LDE}
 \end{figure}

We point out that our Lagrangian debiasing method can be applied to general inference problems beyond  the BTL model. In what follows, we first present the debiasing approach for inference under general constraints. Then we provide the debiasing method under the BTL model under the special case that the constraint function $f$ is linear.

 For general inferential problems under some constraint that the parameter belongs to the set $\mathcal{C}=\{{\theta}:f({\theta})=0\}$, suppose we have an initial estimator $\hat{{\theta}} \in \mathcal{C}$, and let the loss function be $\cL(\theta)$. By~\eqref{equ:gen-debias}, if the matrix below is invertible, we define 
\begin{equation}\label{equ:hat sigma}
\hat{\Sigma}:=\Bigg( \begin{array}{cc} \nabla^2 {\cL(\hat{{\theta}})} 
& \nabla {f(\hat{{\theta}})}\\
\nabla {f(\hat{{\theta}})}^\top 
& 0 \end{array} \Bigg).
\end{equation}
We have that the debiased estimator $\hat\theta^d$ satisfies
\begin{equation} \label{eqn:Sinv}
\Bigg(\begin{array}{c}
\hat{\theta}^d-\hat{\theta} \\
\lambda
\end{array}\Bigg) 
= \hat\Sigma^{-1}
\Bigg(\begin{array}{c}
-\nabla \cL(\hat{\theta}) \\
-f(\hat{\theta})
\end{array}\Bigg).
\end{equation}
When the problem is  high-dimensional, the  matrix $\hat{\Sigma}$ 
is not invertible due to the rank-deficiency, and it becomes challenging to solve problem \eqref{equ:gen-debias}. Motivated by \eqref{eqn:Sinv}, {we aim to find an estimator for the inverse of  the population version of $\hat\Sigma$, which is }
$${\Sigma}^*=\Bigg( \begin{array}{cc} \mathbb{E}\left[\nabla^{2} \cL(\theta^{*}) \right] 
& \nabla {f({{\theta}}^*)}\\
\nabla {f({{\theta}}^*)}^\top 
& 0 \end{array} \Bigg).$$
We achieve this by first finding an estimator for the inverse of the population version of $\mathbb{E}\left[\nabla^{2} \cL(\theta^{*}) \right] $, and  then obtain  an estimator for the inverse by block matrix inverse. 
 
Specifically, we estimate the inverse of $\mathbb{E}\left[\nabla^{2} \cL(\theta^{*}) \right]$ using the constrained $\ell_1$-minimization for inverse matrix estimation (CLIME) method \citep{cai2011constrained}. Denote the estimator as $\hat\Omega$. We obtain an estimator for the inverse of  $\hat \Sigma$ by $\hat \Theta = \Big( \begin{smallmatrix} 
  \hat \Theta_{11} & \hat \Theta_{12}\\
  \hat \Theta_{12}^\top & \hat \Theta_{22} \end{smallmatrix} \Big)$ 
  where 
\begin{equation}\nonumber
\hat \Theta_{11} 
  = \hat \Omega - \hat \Omega \nabla {f(\hat{{\theta}})} \big(\nabla {f(\hat{{\theta}})}^\top \hat \Omega \nabla {f(\hat{{\theta}})}\big)^{-1} \nabla {f(\hat{{\theta}})}^\top \hat \Omega,
\end{equation}
  and
\begin{equation}\nonumber
  \hat \Theta_{12} 
  =\hat \Omega \nabla {f(\hat{{\theta}})}
  \big(\nabla {f(\hat{{\theta}})}^\top \hat \Omega \nabla {f(\hat{{\theta}})}\big)^{-1}, \ \text{ and }\ 
  \hat \Theta_{22} = 
  - \big(\nabla {f(\hat{{\theta}})}^\top \hat \Omega \nabla {f(\hat{{\theta}})}\big)^{-1}.
  \end{equation}
  Thus, we obtain $\hat{\theta}^{d}$ by plugging $\hat\Theta$ into \eqref{eqn:Sinv} that

\begin{equation} \label{eqn:general debias}
\Bigg(\begin{array}{c}
\hat{\theta}^d-\hat{\theta} \\
\lambda
\end{array}\Bigg) 
= \hat \Theta
\Bigg(\begin{array}{c}
-\nabla \cL(\hat{\theta}) \\
-f(\hat{\theta})
\end{array}\Bigg)
\end{equation}
and 
\begin{equation} \label{eqn:lde}
\widehat{{\theta}}^{d} = \hat{{\theta}} -\hat \Theta_{11}  \nabla {\cL({\hat{{\theta}}})} .
\end{equation}
  

Before presenting the asymptotic properties  of $\widehat{{\theta}}^{d}$, we first impose some assumptions. We point out that here we purposely do not specify the convergence rates in the following assumptions since our proposed method is a general framework, and as long as the assumptions for Theorem~\ref{thm:gen-asy} are satisfied, the Lagrangian debiased method achieves the asymptotic normality. We also point out that, under our scaling assumptions in the following theorems, the assumptions are indeed satisfied. 

\begin{assumption}[Consistency for initial estimation of parameters]
\label{ass:consistency}
For some rate $r_1$ which depends on the sample size and parameter dimension, we assume 
$\big\|\widehat{\theta}-\theta^{*}\big\|_{\infty}
\lesssim r_1$.
\end{assumption}

\begin{assumption}[Condition on loss function]
	\label{ass:loss}
For some rate $r_2$ and constant $L_1$, if $\theta={\theta}^*+t(\hat{{\theta}}-{\theta}^*)$ for $t\in[0,1]$, it holds that
$$\big\| \nabla \cL({{\theta}}^*)\|_{\infty}
\lesssim r_2,
\|\nabla^2 \cL({{\theta}})-\nabla^2\cL({{\theta}}^*)\|_\infty
\leq L_1 \|{{\theta}}- {\theta}^*\|_\infty.$$
\end{assumption}

\begin{assumption}[Condition on constraint function]
	\label{ass:constraint}
For some constants $c_1$ and $L_2$, if $\theta={\theta}^*+t(\hat{{\theta}}-{\theta}^*)$ for $t\in[0,1]$, it holds that
$$\big\|\nabla {f({{\theta}}^*)} \big\|_\infty \lesssim c_1,\ 
\|\nabla f({{\theta}})-\nabla f({{\theta}}^*)\|_\infty
\leq L_2 \|{{\theta}}- {\theta}^*\|_\infty.$$	
\end{assumption}

\begin{assumption} \label{ass:omega}
For some rates $r_3, r_4, r_5$ and constants $c_2, c_3$, we assume that
$$	
\|I-\hat{\Omega} \nabla^2{\cL(\hat{{\theta}})}\|_\infty
	\lesssim r_3 ,\ 
	\|\hat \Omega-\Omega^*\|_\infty \lesssim r_4,\ \big\| \Omega^* \big\|_\infty \lesssim c_2, \  
	\nabla {f({{\theta}}^*)}^\top  \Omega^* \nabla {f({{\theta}}^*)} \gtrsim c_3,\ \big\|{\Theta}^*\big\|_{\infty} \lesssim r_5.
$$
\end{assumption}

\begin{assumption}[Central limit theorem (CLT) of the score function]
\label{ass:CLT}
For every $i \neq j$, if $({\Theta}_{11}^* {\Sigma}_{11}^* {{\Theta}_{11}^*}^\top)_{jj} \geq C$ and $(\bm{e}_i-\bm{e}_j)^\top({\Theta}_{11}^* {\Sigma}_{11}^* {{\Theta}_{11}^*}^\top)(\bm{e}_i-\bm{e}_j) \geq C$ for some constant $C>0$, it holds that
$$\frac{\sqrt{n}\big[{\Theta}_{11}^* \nabla \cL({{\theta}}^*)\big]_j}
{\sqrt{[{\Theta}_{11}^* {\Sigma}_{11}^* {{\Theta}_{11}^*}^\top]_{jj}}} \overset{d}{\longrightarrow} N(0,1)$$ and 
$$\sqrt{n}
\frac{[{\Theta}^*_{11}
	\nabla \cL({{\theta}}^*)]_i-[{\Theta}^*_{11}
	\nabla \cL({{\theta}}^*)]_j} {\sqrt{(\bm{e}_i-\bm{e}_j)^T({\Theta}_{11}^* {\Sigma}_{11}^* {{\Theta}_{11}^*}^\top)(\bm{e}_i-\bm{e}_j)}}
\overset{d}{\longrightarrow} N(0,1), $$
 where $[{\Theta}_{11}^* \nabla \cL({{\theta}}^*)]_j$ is the $j$-th entry of ${\Theta}_{11}^* \nabla \cL({{\theta}}^*)$, $[{\Theta}_{11}^* {\Sigma}_{11}^* {{\Theta}_{11}^*}^\top]_{jj}$ is the $j$-th diagonal {element} of matrix ${\Theta}_{11}^* {\Sigma}_{11}^* {{\Theta}_{11}^*}^\top$, and ${\Sigma}_{11}^*$, $\Theta_{11}^*$  is the upper left $ n\times n $  block of ${\Sigma}^*$, $\Theta^*$ respectively.	
\end{assumption}

By the above assumptions, the following two corollaries hold, which are crucial for later proofs. Proofs of Corollaries \ref{cor:gen-asy:1} and \ref{cor:gen-asy:2} can be found in Section~\ref{sec:pf:gen-asy:1} and \ref{sec:pf:gen-asy:2}.

\begin{corollary}\label{cor:gen-asy:1}
	Under Assumptions \ref{ass:consistency} -- \ref{ass:omega},
	we have
	 $\big\|\widehat{\Theta}-\Theta^*\big\|_{\infty}\lesssim r_1 + r_4$.
\end{corollary}

\begin{corollary}\label{cor:gen-asy:2}
	Under Assumptions \ref{ass:consistency} -- \ref{ass:omega}, we have $\|I-\widehat{\Theta}\widehat{\Sigma} \|_{\infty}
	\lesssim r_3.$
\end{corollary}

We then present the asymptotic distribution of the Lagrangian debiased estimator.

\begin{theorem}
\label{thm:gen-asy}
Under Assumptions \ref{ass:consistency} -- \ref{ass:CLT},
%
%
%
%
%
%
%
%
%
%
 	if $c_1,c_2,c_3,L_1,L_2=\cO(1)$ and $\sqrt{n}\big(r_1^2r_5  + (r_1+r_4)(r_1^2 + r_2)+ r_1r_3\big)=o(1)$,  
 	we have the following asymptotic distribution for the Lagrangian debiased estimator \eqref{eqn:lde},
 	$$\sqrt{n} \frac{\widehat{{\theta}}_j^{d}-{{\theta}}^*_j} {\sqrt{[{\Theta}_{11}^* {\Sigma}_{11}^* {{\Theta}_{11}^*}^\top]_{jj}}}
 	 \overset{d}{\longrightarrow} N(0,1),$$
 	and 
 	$$\sqrt{n} \frac{(\widehat{{\theta}}_i^{d}-{{\theta}}^*_i)-(\widehat{{\theta}}_j^{d}-{{\theta}}^*_j)} {\sqrt{(\bm{e}_i-\bm{e}_j)^\top {\Theta}_{11}^* {\Sigma}_{11}^* {{\Theta}_{11}^*}^\top(\bm{e}_i-\bm{e}_j)}}
 	 \overset{d}{\longrightarrow} N(0,1),$$
 	where $\bm{e}_k$ is the natural basis with the $k$-th entry be 1 and other entries be 0.
 \end{theorem}
\begin{proof}
See Appendix Section~\ref{sec:pf:gen-asy}  for the detailed proof.
\end{proof}

 \begin{remark}
 	If the constraint function $f$ is linear, it is not difficult to say that $\hat\theta^d$ satisfies the constraint. Meanwhile, under the general constraint, as the problem is nonconvex, $\hat\theta^d$ may violate the constraint. However, even if the constraint is violated, the asymptotic results above still hold.	
 \end{remark}

\subsubsection{Lagrangian Debiasing for BTL model}
We present the debiasing method under the  BTL model. In particular, for ranking problems, letting the constraint function be linear that $\bm{1}^\top{\theta}=0$ as in \cite{chen2013pairwise,negahban2017rank,Jin_Xu_Gu_Farnoud_2020}, { we  define 
\begin{equation}\label{eqn:Hinv}
\Bigg(\begin{array}{cc}
\hat{\Theta}_{11} & \frac{1}{n} \mathbf{1} \\
\frac{1}{n} \mathbf{1}^\top & 0
\end{array}\Bigg) 
= \Bigg(\begin{array}{cc}
\nabla^{2} \cL(\hat{\theta}) & \mathbf{1} \\
\mathbf{1}^\top & 0
\end{array}\Bigg)^{-1} 
 \text{ and }
\Bigg(\begin{array}{cc}
\Theta^*_{11} & \frac{1}{n} \mathbf{1} \\
\frac{1}{n} \mathbf{1}^\top & 0
\end{array}\Bigg) 
= \Bigg(\begin{array}{cc}
\nabla^{2} \cL(\theta^*) & \mathbf{1} \\
\mathbf{1}^\top & 0
\end{array}\Bigg)^{-1}
\end{equation}
Here the invertibility is provied in Remark \ref{rmk:inv} in the Appendix, and the form of inverse is validated in Corollary \ref{lem:eigenvalue_Q}.
} 

The next theorem shows that under mild scaling conditions, the Lagrangian debiasing estimator $\hat{\theta}^d_j$ and the component-wise difference ${\hat{\theta}^d}_i-{\hat{\theta}^d}_j$ for any $i$ and $j$ ($1 \leq i, j \leq n$) are asymptotically normal with mean $\theta^*_j$ and ${{\theta}}^*_i-{{\theta}}^*_j$, respectively.
\begin{theorem}
	\label{thm:asy}
	Considering the BTL model, under constraint parameter set $\mathcal{C}=\{{\theta}:\bm{1}^\top{\theta}=0\}$, if $p \geq \frac{C_{0} \log n}{n}$ for some sufficiently large constant $C_{0}>0$ and $\frac{n\log n}{\sqrt{L}} + \frac{\log n}{\sqrt{p L}} = o(1)$, we have that the Lagrangian debiasing estimator satisfies that, as $n,L\rightarrow\infty$,
	$$
	\sqrt{L}\ \dfrac{\hat{\theta}^d_j-{\theta}_j^*}{\sqrt{[{\Theta_{11}^*}]_{jj}}}\overset{d}{\longrightarrow} N(0,1),
$$
	and
	$$
	\sqrt{L}\frac{{\hat{\theta}^d}_i-{\hat{\theta}^d}_j-({{\theta}}^*_i-{{\theta}}^*_j)}{\sqrt{ (\bm{e}_i-\bm{e}_j)^\top {\Theta_{11}^*} (\bm{e}_i-\bm{e}_j)}} \overset{d}{\longrightarrow} N(0,1). 
	$$ 
\end{theorem}
\begin{proof}
See Appendix~\ref{sec:pf-asy-BTL}  for the detailed proof.
\end{proof}

\begin{remark}
	By Corollary \ref{lem:eigenvalue_Q} in the Appendix, we have that 
	$[{\Theta_{11}^*}]_{jj}  \asymp \frac{1}{np}\text{ for all } 1 \leq j \leq n$. 
	Consequently, for all $j\in[n]$, we have
	$$
	\big|\hat{\theta}^d_j-{\theta}_j^*\big| \lesssim \sqrt{\frac{1}{npL}}
	$$
	with  probability goes to 1. This matches the $\ell_\infty$-norm error achieved by the spectral and regularized MLE methods as analyzed in \citep{chen2019spectral}.
\end{remark}

The asymptotic normality of ${\hat{\theta}^d}_i-{\hat{\theta}^d}_j$ is the fundamental building block for inferring the pairwise preference such as in Example~\ref{eg:intro-1},  where we test if item $i$ is preferred over item $j$. Basically, it is a ``local" test that only involves two items, which can be done with the asymptotic distribution of ${\hat{\theta}^d}_i-{\hat{\theta}^d}_j$. For the more challenging ``global" testing problems, such as Example \ref{eg:intro-2}, where we test if a given item is among the top-$K$ ranked items, we need to {uniformly control} the quantile of maximal statistic, which will be discussed in the next subsection.

\subsection{Hypothesis Testing}
\label{sec:ht}

In this subsection, we propose our general inferential framework for ranking problems. As mentioned in the Introduction, we first  test if a given item $i$ satisfies some property that $$H_{0}:  \gamma^* \notin \cR_i
\text{ \ \ v.s. \ \  } 
H_{a}: \gamma^* \in \cR_i.$$
To facilitate our discussion, we  define the legal pair  and a distance between the null and alternative, which essentially measure the signal strengths in our testing problems. In particular, when we test some property of item $i$, we say a pair $(i,i')$ is a legal pair if the property is true, and if we swap the scores of item $i$ and item $i'$, the property no longer holds. That is, swapping the scores of item $i$ and item $i'$ changes the property of our interest. The distance between the null and alternative is thus defined as the minimal difference of scores among all legal pairs.

\begin{definition}[Legal pair]
	\label{def:legal}
Suppose $\gamma \in \cR_i$. After swapping the scores of item $i$ and item $i'$, we obtain a new rank $\gamma'$. We say that the pair of items $(i,i')$ is legal  if $\gamma' \notin \cR_i$.
\end{definition}

\begin{definition}[Distance between the null and alternative]\label{def:distance}
	We define the distance between the null and alternative as
	$$
	\Delta({\theta},\cR_i)
	=\min_{i': (i,i') \text{ is legal }} |\theta_{i}-\theta_{i'}|.
	$$	
\end{definition}

Then,  we  provide the specific legal pairs and  distances $\Delta({\theta},\cR_i)$ for  Examples \ref{eg:intro-1} and \ref{eg:intro-2}.

\begin{itemize}
	\item Example \ref{eg:intro-1}: (Pairwise preference between item $i$ and item $j$). We aim to test if item $i$ is ranked higher than item $j$, which means $\gamma^*_i <\gamma^*_j$  or  $\theta^*_i>\theta^*_j$. Let the ranking property be 	$$\cR_i = \{\gamma: \gamma_i < \gamma_j\} = \{\gamma(\theta): \theta_i > \theta_j\}.$$ 

	If $\gamma \in \cR_i$ (i.e., $\theta_i > \theta_j$) and we swap scores of item $i$ and item $i'$ where $\theta_{i'} \leq \theta_{j}$, the new rank does not satisfy $\cR_i$. Meanwhile, if we swap  scores of item $i$ and item $i''$ where { $\theta_{i''} > \theta_{j}$}, the new rank still satisfies $\cR_i$. So $(i,i')$ is a legal pair if $\theta_{i'} \leq \theta_{j}$. See
	Figure \ref{fig:pairwise} for  illustration. 
	This observation leads to the distance $$\Delta({\theta},\cR_i)
	=\min_{i': \theta_{i'}\leq\theta_{j}} |\theta_{i}-\theta_{i'}| = |\theta_{i}-\theta_{j}|.$$

	Equivalently, we can test on preference scores instead of ranking, i.e.,
	testing whether item $i$ has a larger score than item $j$,
	$$
	H_0: \theta^*_i \leq \theta^*_j  \text{\ \ v.s.\ \ } H_a: \theta^*_i>\theta^*_j.
	$$

\begin{figure}[h]
	\begin{center}
		\begin{tabular}{c}		
			\includegraphics[width=.7\textwidth,angle=0]{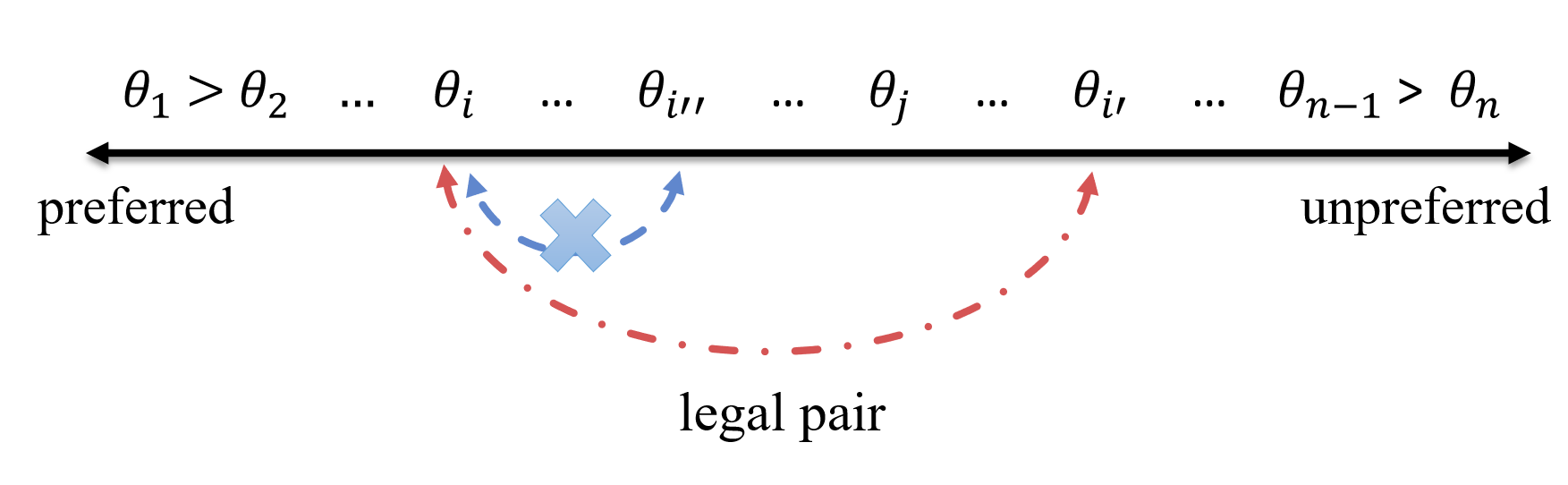} 
		\end{tabular}
	\end{center}
	\caption{Illustration of Example \ref{eg:intro-1}. Assuming $\theta_1 > \theta_2 > \cdots > \theta_n$. Here $\gamma \in \cR_i$ since $\theta_i > \theta_j$. If we swap scores of item $i$ and item $i'$ where $\theta_{i'} \leq \theta_{j}$, the new rank does not satisfy $\cR_i$. Meanwhile, if we swap  scores of item $i$ and item $i''$ where { $\theta_{i''} > \theta_{j}$}, the new rank still satisfies $\cR_i$. Thus, $(i,i')$ is a legal pair if $\theta_{i'} \leq \theta_{j}$. } 
	\label{fig:pairwise}
\end{figure}

	\item Example \ref{eg:intro-2}: (Top-$K$ test). If item $i$'s preference score is larger than $n-K$ items, i.e., $\theta_i > \theta_{(K+1)}$, where $\theta_{(K+1)}$ denotes the $(K+1)$-th largest preference score, or equivalently, $\gamma_i \le K$, we say that item $i$ is ranked among top-$K$ items. We aim to test if item $i$ is ranked among top-$K$ items. Thus, we have that $\cR_i$ is
	$$\cR_i=\{\gamma: \gamma_i \leq K\} = \{\gamma(\theta): \theta_i > \theta_{(K+1)}\}.$$	
	
	If $\gamma \in \cR_i$ (i.e., $\theta_i > \theta_{(K+1)}$), and we swap scores of item $i$ and item $j$ where $\theta_{j} \leq \theta_{(K+1)}$, the new rank does not satisfy $\cR_i$. Meanwhile, if we swap  scores of item $i$ and item $j$ where $\theta_{j} > \theta_{(K+1)}$, we have that the new rank satisfies $\cR_i$. Thus, $(i,j)$  is a leagal pair if $\theta_{j} \leq \theta_{(K+1)}$, and the distance is $$\Delta({\theta},\cR_i)
	=\min_{j:\theta_{j}\leq\theta_{(K+1)}} |\theta_{i}-\theta_{j}|=|\theta_{i}-\theta_{(K+1)}|.$$
	
	Similarly, we can transform the test on ranking into testing on preference score, our test is 
	$$
	H_0: \theta^*_i-\theta^*_{(K+1)}\leq 0 \text{\ \ v.s.\ \ }
	H_a: \theta^*_i-\theta^*_{(K+1)}>0. 
	$$
\end{itemize}


%

Next, we explain our testing procedure with the above two examples.  
Consider Example \ref{eg:intro-1},  where we test if item $i$ is ranked higher than item $j$, or equivalently, we test if item $i$ has a larger score than item $j$,
$$
	H_0: \theta^*_i \leq \theta^*_j  \text{\ \ v.s.\ \ } H_a: \theta^*_i>\theta^*_j.
$$
For this example, by Theorem \ref{thm:asy}, if the assumptions are satisfied, we have 
$$
\sqrt{L}\frac{{\hat{\theta}^d}_i-{\hat{\theta}^d}_j-({{\theta}}^*_i-{{\theta}}^*_j)}{\sqrt{ (\bm{e}_i-\bm{e}_j)^\top {\Theta_{11}^*} (\bm{e}_i-\bm{e}_j)}} \overset{d}{\longrightarrow} N(0,1) .
$$
We  thus reject $H_0$ if
$$ \sqrt{L}\frac{\hat{\theta}^d_i-\hat{\theta}^d_j}{\sqrt{ (\bm{e}_i-\bm{e}_j)^\top {\hat{\Theta}_{11}} (\bm{e}_i-\bm{e}_j)}}>\Phi(1-\alpha),$$ 
where $\Phi(\cdot)$ is the cumulative distribution function of a standard normal random variable. 

Note that this example only involves two items, which is a relatively simple local test. However, for more general testing problems, we need to consider more than two items. For instance, in Example~\ref{eg:intro-2},  we test if item $i$ is ranked among the top-$K$ items that
$$
	H_0: \theta^*_i-\theta^*_{(K+1)}\leq 0 \text{\ \ v.s.\ \ }
	H_a: \theta^*_i-\theta^*_{(K+1)}>0. 
$$
where $\theta^*_{(K+1)}$ is the $(K+1)$-th largest score in terms of order statistic. If the score of item $i$ is larger than the $(K+1)$-th largest score, we have that item $i$ is ranked among top-$K$ items. In this example and more general problems,  we need to study the maximal statistic $\max_{j \neq i} (\hat{\theta}^d_i-\theta^*_i-\hat{\theta}^d_j+\theta^*_j)$. In what follows, we demonstrate that we can estimate the quantiles of this maximal statistic via the Gaussian multiplier bootstrap.



\noindent{\bf Gaussian multiplier bootstrap.}
We start from a general fixed edge set $E \subseteq \mathcal{V} \times \mathcal{V}$. The goal is to control the tail probability of the statistic that 
\begin{equation}\label{eqn:T}
\begin{aligned}
T:=&\max_{(i,j) \in E} \sqrt{{npL}} (\hat{\theta}^d_i-\theta^*_i-\hat{\theta}^d_j+\theta^*_j)\\
=&-\max_{(i, j) \in E} \sqrt{\frac{1}{npL}}\sum_{\ell=1}^L \sum_{k>m}{\mathcal{E}_{mk}}\Big(-y_{mk}^{(\ell)}+\frac{e^{\theta_k^*}}{e^{\theta_k^*}+e^{\theta_m^*}}\Big) n p \big([{\Theta}^*_{11}]_i - [{\Theta}^*_{11}]_j \big)(\bm{e}_k-\bm{e}_m)\\
&+\sqrt{{npL}}(r_{i}-r_{j})\\
:=&\max_{(i, j) \in E} \sqrt{\frac{1}{L}}\sum_{\ell=1}^L x_{ij}^{(\ell)}+\sqrt{{npL}}(r_{i}-r_j),
\end{aligned}
\end{equation}
where $\bm{e}_k$ is the natural basis,  $[{\Theta_{11}^*}]_i$ is the $i$-th row of matrix ${\Theta_{11}^*}$ defined in  \eqref{eqn:Hinv}. The second equality comes from \eqref{eqn:op1} and \eqref{eqn:grad}. $x_{ij}^{(\ell)}$ is defined as
\begin{equation}
\label{eqn:def-x}
 x_{ij}^{(\ell)}:=-\sqrt{{np}} \sum_{k>m}{\mathcal{E}_{mk}}\Big(-y_{mk}^{(\ell)}+\frac{e^{\theta_k^*}}{e^{\theta_k^*}+e^{\theta_m^*}}\Big) \big([{\Theta_{11}^*}]_i-[{\Theta_{11}^*}]_j \big)(\bm{e}_k-\bm{e}_m).
\end{equation} 
which is an independent zero-mean {random variable in $\mathbb{R}$} for $\ell=1, \cdots, L$. {Here $\cE_{mk} = 1$ if $(m,k)\in \cE$ the comparison graph, and $\cE_{mk} = 0$ otherwise.}

To estimate the quantile of $T$, we consider the Gaussian multiplier bootstrap in \cite{Chernozhukov2013Gaussian}. The main idea is to  approximate the distribution of the maximum of a sum of independent random vectors with unknown covariance by the distribution of the maximum of a sum of conditional Gaussian random vectors, which is obtained by multiplying the original vectors with i.i.d. normal random variables. In our case, even though vector ${\theta}^*$
is not  observable, we have some estimators $\hat{{\theta}}$  are available, and we use the estimators to approximate ${\theta}^*$ in the bootstrap.

Hence, we define the following statistic from Gaussian multiple bootstrap
\begin{equation}\label{eqn:W}
W:=\max_{(i, j) \in E} \sqrt{\frac{1}{L}}\sum_{\ell=1}^L\bigg\{- \sqrt{{np}}\sum_{k>m}{\mathcal{E}_{mk}}\Big(-y_{mk}^{(\ell)}+\frac{e^{\hat{\theta}_k}}{e^{\hat{\theta}_k}+e^{\hat{\theta}_m}}\Big)\big([\hat{\Theta}_{11}]_i - [\hat{\Theta}_{11}]_j \big)(\bm{e}_k-\bm{e}_m) \bigg\} z_\ell,
\end{equation}
where $z_\ell$, $\ell=1,\cdots, L$, are i.i.d standard normal random variables.

We then estimate the conditional quantile of $W$ given data $\bm{y}=\big\{y_{mk}^{(\ell)}\big\}^{\ell=1,\cdots,L}_{k>m}$ by 
\begin{equation}\label{eqn:quantile}
c_W(\alpha,E)=\inf \big\{t \in \mathbb{R}:\mathbb{P}(W>t \given\bm{y}) \leq\alpha\big\}.
\end{equation}
The next theorem uniformly controls the tail probability of $T$ by $c_W(\alpha,E)$.  

\begin{theorem}\label{thm:bootstrap}
	Considering the BTL model, for any edge set $E \subseteq \mathcal{V} \times \mathcal{V}$, if $ n^2 p \frac{(\log (n L))^{7}}{L} =o(1)$ and $\frac{n(\log n)^{3/2}}{\sqrt{L}} =o(1)$, we have
	\begin{equation}\nonumber
	\sup_{\alpha \in (0,1)} \big|\mathbb{P}(T > c_W(\alpha,E) )-\alpha\big| \to 0 .
	\end{equation} as $n,L \to \infty$,
\end{theorem}
\begin{proof}
We provide the proof in Appendix  Section \ref{sec:pf:bootstrap}.
\end{proof}

This theorem shows that $c_W(\alpha,E)$ obtained from the Gaussian multiplier bootstrap is a valid quantile estimator for $T=\max_{i, j \in E} \sqrt{{npL}} (\hat{\theta}^d_i-\theta^*_i-\hat{\theta}^d_j+\theta^*_j)$.
In this theorem, the first scaling condition
$n^2 p \frac{(\log (n L))^{7}}{L} =o(1)$ is from the Gaussian approximation for the
maximum of a sum of random vectors, and  the second scaling condition $\frac{n(\log n)^{3/2}}{\sqrt{L}}=o(1)$
is  from approximating $T$ and $W$ by their leading terms. Given this statistic,  we are ready to present the procedure for testing general ranking properties.


\noindent{\bf General testing procedure.}
For general ranking property test {with respect to item $i$}
$${H}_{0}:  \gamma^* \notin \cR_i
\text{ \ \ v.s. \ \  } 
{H}_{a}: \gamma^* \in \cR_i,$$
we perturb the preference score of every item (i.e., $\hat{\theta}^d_i$) up to $\alpha$-quantile of $\max_{j \neq i} (\hat{\theta}^d_i-\theta^*_i-\hat{\theta}^d_j+\theta^*_j)$, and conduct the test. Specifically, let  $\tilde{\Theta}$ be the set of all possible score vectors after perturbation~that
\begin{equation}\label{eqn:set theta tilde}
\tilde{\Theta}=
\Big\{{\theta}:{\theta}_k \in \big[\hat{\theta}^d_k - c_W(\alpha,i)/\sqrt{npL}, \ \hat{\theta}^d_k + c_W(\alpha,i)/\sqrt{npL}\big],
1 \leq k \leq n \Big\}.
\end{equation}
We reject the null hypothesis  if $\gamma{({\theta})} \in \cR_i$ for any ${\theta} \in \tilde{\Theta}$, i.e., the event $\cap_{{\theta} \in \tilde{\Theta}} \{\gamma{({\theta})} \in \cR_i \}$~holds. 

\begin{remark} \label{rmk:topk}
We point out that we can simplify this general procedure for specific problems. For instance, when we test if item $i$ is ranked within the  top-$K$, we only need to consider the extreme point of the perturbation, where for $k=1,...,n$, its $k$-th entry is defined as 
\begin{equation}\nonumber
\theta_k = 
\left\{\begin{array}{ll}
\hat{\theta}^d_{k}  & \quad \hat{\theta}^d_{k} > \hat{\theta}^d_{i}, \\
\hat{\theta}^d_k - c_W(\alpha, i)/\sqrt{npL}  &\quad  k = i,\\
\hat{\theta}^d_{k} + c_W(\alpha, i)/\sqrt{npL}  &\quad  \hat{\theta}^d_{k} < \hat{\theta}^d_{i}.
\end{array}\right.
\end{equation}
In fact, we can further simplify this procedure that we only consider ${\theta}$ where its $k$-th entry is defined as
\begin{equation}\nonumber
\theta_k = 
\left\{\begin{array}{ll}
\hat{\theta}^d_{k}  & \quad
\hat{\theta}^d_{k} > \hat{\theta}^d_{i}, \\
\hat{\theta}^d_k - c_W(\alpha, i)/\sqrt{npL}  &\quad  k = i,\\
\hat{\theta}^d_{k}  & \quad
\hat{\theta}^d_{k} < \hat{\theta}^d_{i}.
\end{array}\right.
\end{equation}
We justify this procedure in Appendix Section \ref{sec:test:topk}.

\end{remark}
 
We conclude this section by the following theorem that we show that the proposed procedure controls the Type I error, and we provide the power analysis.
\begin{theorem}
	\label{thm:type I error}
	Under same assumptions as in Theorem \ref{thm:asy} and \ref{thm:bootstrap}, we have the general testing procedure satisfies that, as $n,L\to\infty$,
	\begin{equation}\nonumber
 \sup_{\gamma^* \notin \cR_i} \mathbb{P}_0 \left(\text{Reject } H_0\right) \leq \alpha, 
	\end{equation}
	and we have
	\begin{equation}\nonumber
\inf_{\gamma^* \in \cR_i, \Delta({\theta}^*,\cR_i)>\delta} \mathbb{P}\left(\text{Reject } H_0\right) \to 1, 
	\end{equation} 
	where $ \delta=C\sqrt{\frac{\log n}{npL}}$ for some constant $C$. 
\end{theorem}	
\begin{proof}
We provide the proof in Appendix Section \ref{sec:pf:ht}.
\end{proof}



%% file: theory.tex

\section{Multiple Testing}
\label{sec:multi test} 
In this section, we extend the proposed procedure to the multiple testing setting. As discussed in the introduction, the multiple testing finds important applications in our ranking inference problems  such as Example~\ref{eg:intro-3}, where we aim to infer the set of all top-$K$ ranked items. In general multiple testing problems, we aim to control the familywise Type I error rate (FWER), or the false discovery rate (FDR), while achieving certain power. Widely used procedures include   Bonferroni correction, Dunn-\v{S}id\'{a}k procedure \citep{vsidak1967rectangular}, Holm Procedure \citep{holm1979simple} for controlling FWER, and Benjamini-Hochberg procedure (BH) \citep{benjamini1995controlling}, Benjamini-Yekutieli procedure \citep{benjamini2001control} for controlling FDR.  In addition, there are also resampling based procedures such as permutation testing and bootstrap method \citep{westfall1993resampling, ge2003resampling}.
However, these methods cannot be directly applied to our problems, as their theoretical properties cannot be easily justified. This is mainly because that the test statistics for the hypotheses are clearly dependent, which makes our multiple testing problems challenging.

In general, we aim to  test the following hypotheses simultaneously,
\begin{equation}
\label{eqn:MT}
H_{0 i}: \text{item } i \text{ does not satisfy property  }  \cR_i 
\text{\  \ v.s.\  \ } 
H_{a i}: \text{item } i  \text{ satisfies  }  \cR_i, \text{\ \ \ for } i\in[n]. 
\end{equation} 

\subsection{Control FWER}\label{sec:FWER}
We first present our procedure for controlling the FWER. Recall that the FWER is  the probability of making at least one Type I error that  
$$
\text{FWER} = \mathbb{P}(\symbol{"0023} \text{ false positives} > 0) .
$$
Specifically, when we aim to control the FWER in multiple testing, we let the maximal statistic be
$$M=\max _{i \in [n]} \max_{j\in [n]} \sqrt{{npL}} (\hat{\theta}^d_i-\theta^*_i-\hat{\theta}^d_j+\theta^*_j),
$$ 
and we estimate its $(1-\alpha)$-th percentile $C_M(\alpha, [n] \times [n])$ by Gaussian multiplier bootstrap and taking the edge set $E$ as $[n] \times [n] $ in \eqref{eqn:T}.
Next, we reject the $H_{0i}$ in \eqref{eqn:MT} if item $i$ satisfies property $\cR_i$ for all possible perturbation for the debiased estimator $\hat\theta^d$ up to $C_M(\alpha, [n] \times [n])/\sqrt{npL}$ entrywise. Equivalently, letting $\tilde{\Theta}$ be the set of  possible latent scores after perturbation that
$$
\tilde{\Theta}=
\Big\{{\theta}:{\theta}_k \in \big[\hat{\theta}^d_k-C_M(\alpha, [n] \times [n])/\sqrt{npL},\ \hat{\theta}^d_k+C_M(\alpha, [n] \times [n])/\sqrt{npL}\big],
1 \leq k \leq n \Big\},
$$
we reject $H_{0i}$ in \eqref{eqn:MT} if for any ${\theta} \in \tilde{\Theta}$, we have $\gamma{({\theta})} \in \cR_i$, i.e., the event $\cap_{{\theta} \in \tilde{\Theta}} \{\gamma{({\theta})} \in \cR_i \}$ holds.

%

To facilitate our power analysis, we define the signal strength for multiple testing problems as
\begin{equation}\label{equ:distance}
\Delta({\theta})
=\min_{i: \gamma(\theta) \in \cR_i}\Delta({\theta},\cR_i)
=\min_{i: \gamma(\theta) \in \cR_i}\min_{i':(i,i') \text{ is legal }} |\theta_{i}-\theta_{i'}|.
\end{equation} 
In what follows, we use two  examples to illustrate some insights of this signal strength.
\begin{itemize}
	\item 
	Consider the problem of selecting all top-$K$ ranked items.  Denote the $i$-th largest order statistic as $\theta_{(i)}$. For any given $1\leq i\leq n$ satisfying $\theta_{i} \geq \theta_{(K)}$, the pairs of items $(i,i')$ such that $\theta_{i'} \leq \theta_{(K+1)}$ are the legal pairs, and the smallest distance is $|\theta_{i}-\theta_{(K+1)}|$. Hence,  we have
	 $$ \Delta({\theta})
	=\min_{\theta_{i} \geq \theta_{(K)}}\min_{\theta_{i'} \leq \theta_{(K+1)}} |\theta_{i}-\theta_{i'}|
	=\min_{\theta_{i} \geq \theta_{(K)}}|\theta_{i}-\theta_{(K+1)}|
	=|\theta_{(K)}-\theta_{(K+1)}|.$$
	This is consistent with our intuition that when we aim to choose the top $K$ items, the gap 
	between the scores of $\theta_{(K)}$ and $\theta_{(K+1)}$ somewhat determines this problem's difficulty.
	
	\item 	
	Consider the problem of selecting all items  ranked higher than item $k$ with score $\theta_{k}$. 	For any item $i$ satisfying $\theta_{i} > \theta_{k}$, the pairs $(i,i')$ where $\theta_{i'} \leq \theta_{k}$ are legal pairs. Hence, we have that the distance is 
	$$ \Delta({\theta})
	=\min_{\theta_i > \theta_k}\min_{\theta_i' \leq \theta_k} |\theta_{i}-\theta_{i'}|
	=\min_{\theta_i > \theta_k} |\theta_{i}-\theta_{k}|
	.$$

\end{itemize}

The following theorem shows that the FWER based on our procedure above is guaranteed to be no greater than $\alpha$ asymptotically and is powerful.

\begin{theorem}[Familywise Type I Error Rate]\label{thm:FWER}
	{Under the same assumptions for Theorems \ref{thm:asy} and \ref{thm:bootstrap}, following the multiple testing procedure  above,} for any $0<\alpha<0.5$, we have 
	\begin{equation}\nonumber
	\mathrm{FWER} =\mathbb{P}(\text {Making at least one Type I error})
	\leq \alpha+o(1).
	\end{equation} 
Further,  if 
$ \Delta({\theta}^*) \gtrsim \sqrt{\frac{\log n}{npL}}$ holds, we have  
	\begin{equation}\nonumber
	\mathbb{P}(\text {Making at least one Type II error}) \to 0.
	\end{equation}
\end{theorem}
\begin{proof}
We provide the proof in Appendix Section \ref{sec:pf:FWER}.
\end{proof}

This theorem shows that our method  controls the FWER asymptotically for the given level. Meanwhile, our procedure is asymptotically powerful if $\Delta({\theta}^*) \gtrsim \sqrt{\frac{\log n}{npL}}$, which matches the lower bound we derive in Section \ref{sec:LB_FWER}.

\subsection{FDR Control}
\label{sec:FDR}
We then consider the problem of controlling the false discovery rate (FDR). FDR is the expected proportion of Type I errors among the all discoveries \citep{benjamini1995controlling}, and our goal is to control the  FDR under some prespecified level $\alpha$ that
\begin{equation}\nonumber
\text{FDR} = \mathbb{E} \Big[\frac{\symbol{"0023} \text{ false positives }}{\symbol{"0023} \text{ discoveries }}\mathbb{I}[\ \symbol{"0023} \text{ discoveries }>0]\Big] \leq \alpha.
\end{equation}

Consider the multiple testing problem of our interest \eqref{eqn:MT}. For each hypothesis of item $i$, we perform our proposed single testing procedure in Section \ref{sec:ht} and get the p-value $p_i$ that
$$p_i =\inf\big\{\alpha_0:\cap_{{\theta} \in \tilde{\Theta}(\alpha_0)} \{\gamma{({\theta})} \in \cR_i \}\big\} \ \ \text{for } 1\leq i \leq n, $$
where $$\tilde{\Theta}(\alpha_0)=
\Big\{{\theta}:{\theta}_k \in \big[\hat{\theta}^d_k-c_W(\alpha_0,i)/\sqrt{npL},\ \hat{\theta}^d_k+c_W(\alpha_0,i)/\sqrt{npL}\big],
1 \leq k \leq n \Big\}.$$
Since the p-values $p_i$'s for different tests have complicated dependency, we consider the  Benjamini-Yekutieli  procedure by \cite{benjamini2001control} to control the FDR, which ranks the hypotheses according to their corresponding p-values and  chooses a cutoff to control the FDR. In  particular, we order the $n$ p-values in the ascending order as $p_{(1)}, \cdots, p_{(n)}$ and  
reject the null hypothesis  for all $ H_{(i)}$, $i = 1, \ldots, r$, where    $$r = \max_k\Big\{k:p_{(k)} \leq \frac{k}{n \cdot N}\cdot\alpha \Big\} \text{ \ \  and \ \  } N = \sum _{k=1}^{n} \frac {1}{k}.$$

The following theorem shows that the FDR based on our procedure achieves the desired FDR level asymptotically if $|\mathcal{H}_{0}|\Big(\frac{1}{L^{c_3}}+\frac{2}{n^5}\Big)=o(1)$, {where $|\mathcal{H}_{0}|$ is the number of true null hypotheses}.
\begin{theorem}[FDR control]\label{thm:FDR}
	{ Suppose that the conditions in Theorems \ref{thm:asy} and \ref{thm:bootstrap} hold.  Following the multiple testing procedure above,} for any $0<\alpha<1$, we have 
	\begin{equation}\nonumber
	\text{FDR} \leq  \frac{|\mathcal{H}_{0}|}{n}\cdot\alpha+C|\mathcal{H}_{0}|\Big(\frac{1}{L^{c_3}}+\frac{2}{n^5}\Big),
	\end{equation} 
for some constant $C>0$, and the constant $c_3$ satisfies the condition in Remark \ref{rmk:c3}. 	
\end{theorem}
\begin{proof}
We provide the proof in Appendix Section  \ref{sec:pf:FDR}.
\end{proof}


{
\begin{remark}
	We point out that, as show in the theorem above and in our simulation studies, our developed B-Y based FDR controlling procedure is  relatively conservative when the number of true nulls $|\mathcal{H}_{0}|$ is small compared to total number of items $n$. The similar conservative result also holds in \cite{eisenach2020high}.
\end{remark}
}

\section{Lower Bound Theory}
\label{sec:LB_FWER}
 
 In this section, we derive a lower bound for the multiple testing problem \eqref{eqn:MT}.  
To the best of our knowledge, even beyond ranking problems, all existing works focus on discussing the lower bound for single hypothesis testing problems.  To facilitate our discussion,  we first define a novel minimax risk of multiple testing problems. In particular, we let the risk be the probability of making at least one Type I  or Type II error that
\begin{equation}\label{eqn:R}
\mathfrak{R} = \inf_{\psi}  \sup_{\theta^* \in \Xi} \P(\symbol{"0023} \text{ false positives}+ \symbol{"0023} \text{ false negatives } \geq 1),
\end{equation} 
 where $\psi$ is any selection procedure giving a vector  $\psi \in \{0,1\}^n$ that $\psi_i=1$ means that we reject $H_{0i}$ in \eqref{eqn:MT}, and $\Xi$ is our parameter space,  which is closed under swapping scores of any two items. 
 If $\mathfrak{R} \geq 1-\epsilon$, where $\epsilon>0$ is a   constant, we say that any procedure fails in the sense that {Type I or Type II} error can not be controlled. Given this setup,
we aim to characterize necessary conditions under which
we can control the minimax risk to the desired level.


\subsection{Lower Bound Theorem}

Recall that we define the legal pair in Definition \ref{def:legal} and the distance $\Delta({\theta})$ in \eqref{equ:distance}. We then  define a critical  subset of all legal pairs, which is crucial in deriving the information-theoretic lower bound and  captures the ``hardness'' of multiple testing for different properties.

\begin{definition}[Divider set]
	\label{def:divider}
	 A set $\mathcal{M}(\theta) \subseteq [n] \times [n]$ is a divider set if $\mathcal{M}(\theta)$ satisfies the following two conditions,
	 \begin{itemize}
	 	\item For any $(i,i') \in \mathcal{M}(\theta)$, it satisfies $ \gamma \in \cR_i$, $(i,i')$ is legal, and $|\theta_i-\theta_{i'}|=\Delta({\theta})$. 
	 	\item For any two pairs $(i_1,i_1'), (i_2,i_2') \in \mathcal{M}(\theta)$, letting $\gamma_1$ be the scores by swapping items $i_1$ and $i'_1$, and $\gamma_2$ be the scores by swapping  items of  $i_2$ and $i'_2$. They must satisfy  $\{i:\gamma_1 \in \cR_i\} \neq \{i:\gamma_2 \in \cR_i\}$. 
	 \end{itemize} 
\end{definition} 
The following theorem derives necessary signal strengths in $\Delta(\theta)$ and  $|\mathcal{M}(\theta)|$ for controlling the risk $\mathfrak{R} \le 1-\epsilon$.

\begin{theorem}[Necessary Signal Strength] 
	\label{thm:LB}
	Considering the BTL model where the pairwise comparison probability $p$ in Erd\"{o}s-R\'{e}nyi graph satisfies {$p \gtrsim \frac{\log n}{n}$}, if  there exists $\theta \in \Xi$ such that
	\begin{equation}\label{equ:LB}
	\Delta({\theta}) \lesssim  
	 \sqrt{\frac{\epsilon \log \big(|\mathcal{M}(\theta)|+1 \big)-\log 2 }{npL}},
	\end{equation} 
we have the minimax risk $\mathfrak{R} = \inf_{\psi}  \sup_{\theta^* \in \Xi} \P(\symbol{"0023} \text{ false positives}+ \symbol{"0023} \text{ false negatives } \geq 1) > 1-\epsilon$.
\end{theorem}
\begin{proof} 
Our proof is based on three steps. In the first step, we reduce the problem of obtaining a lower bound of minimax risk in \eqref{eqn:R} for $\theta^* \in \Xi$  to the problem of deriving a lower bound for $\theta^*$ in a finite set. Next, we construct a finite set of hypotheses. Finally, we derive the minimax risk by Fano-type arguments, and derive the necessary signal strength condition.

\noindent{\bf Step 1.}
Let  $\cR_j$ be a ranking property for a single item $j$ as defined in Definition \ref{def:ranking property}. We have
 \begin{equation}
 \label{eqn:reduce}
 \begin{aligned}
 \mathfrak{R} 
 & = \inf_{\psi} \sup_{\theta^* \in \Xi} \P(\symbol{"0023} \text{ false positives}+ \symbol{"0023} \text{ false negatives } \geq 1) \\
 & = \inf_{\psi} \sup_{\theta^* \in \Xi} \P(\psi \neq \psi^*)
 = \inf_{\psi} \sup_{\theta^* \in \Xi} \P(d(\psi,\psi^*)\geq 1/2),
 \end{aligned}
 \end{equation}
 where $\psi^* \in \{0,1\}^n$ with $\psi^*_j=1$ meaning $\gamma^* \in \cR_j $ and $\psi^*_j=0$ meaning $\gamma^* \notin \cR_j$  for $j=1, \cdots, n$, and $\psi \in \{0,1\}^n$ is a selection procedure that   $\psi_i=1$ means that we reject $H_{0 i}$ in \eqref{eqn:MT}, and the distance $d(\psi,\psi^*) = \|\psi-\psi^*\|_1$.
 By Section~2.2 of \cite{tsybakov2008introduction}, the problem of obtaining a lower bound of minimax risk in \eqref{eqn:reduce} for $\theta^* \in \Xi$ can be reduced to the problem of deriving a lower bound for $\theta^*$ in a finite set.
 %
In particular, suppose that we have a collection $M+1$ hypotheses $\cH_{M} = \{H_i:\theta^* =  \theta^{(i)}, 0 \leq i \leq M\}$. Recall that a test is any measurable function  $\phi:\bm{Y}=\{\bm{Y}^{(\ell)}\}_{\ell=1}^{L} \mapsto \{0,1, \cdots, M\}$.
 Let $\mathbb{P}_{e,M}$ be the minimax probability of error
 \begin{equation}\label{eqn:minimax prob error}
 \mathbb{P}_{e,M} :=  \inf _{\phi} \max _{\theta^* \in \{\theta^{(0)}, \theta^{(1)}, \cdots, \theta^{(M)}\}} \mathbb{P}_{i}(\phi \neq i)
 = \inf _{\phi} \max _{0 \leq i \leq M} \mathbb{P}_{i}(\phi \neq i),
 \end{equation} 
 where $\mathbb{P}_{i}$  denotes the probability measure  under the $i$-th hypothesis $\theta^*=\theta^{(i)}$.
The following lemma in Section 2.2 of \cite{tsybakov2008introduction} shows that the minimax risk $ \mathfrak{R}$ in \eqref{eqn:reduce} is lower bounded by  $ \mathbb{P}_{e,M}$ above, and  we provide the proof in Appendix Section~\ref{sec:reduction} for self-completeness.

 \begin{lemma}\label{lem:reduction}
 	Suppose we have $M+1$ hypotheses $\cH_{M} = \{H_i:\theta^* =  \theta^{(i)}, 0 \leq i \leq M\}$ with deduced tests $\psi^{(i)} \in \{0,1\}^n$ where $\psi^{(i)}_j=1$ means $\gamma(\theta^{(i)}) \in \cR_j$ and $\psi^{(i)}_j = 0$ means $\gamma(\theta^{(i)}) \notin \cR_j$ for $0 \leq i\leq M$,  $1 \leq j \leq n$. If these hypotheses satisfy $d(\psi^{(i_1)},\psi^{(i_2)}) \geq 1$ for any $ i_1 \neq i_2$,
 	we have 
 	$$\mathfrak{R} = \inf_{\psi} \sup_{\theta^* \in \Xi} \P(d(\psi,\psi^*)\geq 1/2) \geq \mathbb{P}_{e,M}.$$	
 \end{lemma}

 Thus, if  $\mathbb{P}_{e, M} > 1-\epsilon$,  we have $\mathfrak{R} > 1-\epsilon$, i.e., for any selection procedure $\psi$, there exists a preference vector ${\theta} \in \Xi$ such that the probability of making at least one Type I or Type II error in the family is uncontrollable. This shows that if we can find a finite set $\{\theta^{(i)}, 0 \leq i \leq M\}$ satisfying the conditions of Lemma \ref{lem:reduction}, we can derive the lower bound. We construct this set in Step 2.
 
\noindent{\bf Step 2.}
 In this step, we  construct a finite set of hypotheses for deriving the lower bound. We choose a base preference score vector $\theta$ first. Without loss of generality, we assume that the preference score~$\theta$ is in descending order, i.e., \begin{equation}\label{eqn:step2}
 \theta_{1} > \theta_{2} > \cdots > \theta_{n}.
 \end{equation} 
 
  
 Recall that the divider set $\mathcal{M}(\theta) \subseteq [n] \times [n]$ defined in Definition \ref{def:divider} is a collection of pairs with cardinality $|\mathcal{M}(\theta)|$. Denote this set by $\mathcal{M}(\theta) = \{(k_i,k'_i),i=1,\cdots, |\mathcal{M}(\theta)|\}$.
 Based on the base preference score $\theta$ and the divider set, we construct a set of hypotheses  $\cH_{\cM} = \{H_0, H_1, \cdots, H_{|\mathcal{M}(\theta)|}\}$ such that 
\begin{equation}\label{eqn:set hypo}
 H_0: \theta^* = \theta^{(0)}, \ \  H_i:\theta^* =  \theta^{(i)}, 1 \leq i \leq |\mathcal{M}(\theta)|,
\end{equation}
 where $\theta^{(0)}=\theta$, and $\theta^{(i)}$ is obtained by swapping scores of the $i$-th pair $(k_i,k'_i) \in \mathcal{M}(\theta)$ in the base score vector $\theta$. For each hypothesis $H_i:\theta^* =  \theta^{(i)}$, we also have the induced rank $\gamma^{(i)}=\gamma(\theta^{(i)})$ and  vector $\psi^{(i)} \in \{0,1\}^n$ with $\psi^{(i)}_j=1$ meaning $\gamma^{(i)} \in \cR_j $ and $\psi^{(i)}_j=0$ meaning $\gamma^{(i)} \notin \cR_j $, $j=1, \cdots, n$.
 
By the construction of $\mathcal{M}(\theta)$ in Definition \ref{def:divider} , for $i_1 \neq i_2$,
  we have $\{j:\gamma^{(i_1)} \in \cR_j\} \neq \{j:\gamma^{(i_2)} \in \cR_j\}$, which gives $\psi^{(i_1)} \neq \psi^{(i_2)}$, and we have $d(\psi^{(i_1)},\psi^{(i_2)}) \geq 1$ for any $i_1 \neq i_2$, which satisfies the condition in Lemma~\ref{lem:reduction}.
 
 
\noindent{\bf Step 3.}
 In this step, we obtain a lower bound on the minimax risk by  Fano-type bounds.  
Here we denote $\mathcal{M}(\theta)$ by $\mathcal{M}$ for simplicity. Let the average probability of error and
 the minimum average probability of error of a test $\phi:\bm{Y}=\{\bm{Y}^{(\ell)}\}_{\ell=1}^{L} \mapsto \{0,1, \cdots, M\}$ be
 $$
 \bar{\P}_{e,\mathcal{M}}(\phi)=\frac{1}{|\mathcal{M}|+1} \sum_{j=0}^{|\mathcal{M}|} \P_{j}(\phi \neq j), \quad 
 \text{ and }\quad
\bar{\P}_{e, \mathcal{M}}=\inf _{\phi} \bar{\P}_{e, \mathcal{M}}(\phi).$$
 One can easily verify that $\P_{e, \mathcal{M}} \geq \bar{\P}_{e, \mathcal{M}}$, where the minimax probability of error $\P_{e, \mathcal{M}}$ is defined in  \eqref{eqn:minimax prob error}. Then we reduce bounding  the minimax probability of error to bounding
 the minimum average probability of error since if $\bar{\P}_{e, \mathcal{M}} > 1-\epsilon$,  we have $\mathfrak{R} > 1-\epsilon$.

 Following the argument in \cite{chen2015spectral}, we also apply the generalized Fano inequality in \cite{verdu1994generalizing}, which gives a lower bound for $\bar{\P}_{\mathrm{e}, \mathcal{M}} $. We summarize the result in the following lemma and provide the proof in Appendix Section~\ref{sec:fano}.
 
 \begin{lemma}
 	\label{lem:fano}
 	Under the BTL model, let $\mathcal{M}=\mathcal{M}(\theta)$ be the divider set defined in Definition~\ref{def:divider}. Given the set of hypotheses $\mathcal{H}_{\mathcal{M}}$ constructed   in Step 2, we have
 	\begin{equation}\label{eqn:l5.4}
 	\begin{aligned}
 	\bar{\P}_{\mathrm{e}, \mathcal{M}} 
 	\geq 
 	 1-\frac{1}{\log (|\mathcal{M}|+1)}\bigg\{\frac{p L}{(|\mathcal{M}|+1)^{2}} \sum_{H,\tilde{H}\in \mathcal{H}_{\mathcal{M}}} \sum_{i < j} \mathrm{KL}\left(\mathbb{P}_{y_{i, j}^{(1)} \mid H} \| \mathbb{P}_{y_{i, j}^{(1)} \mid \tilde{H}}\right)+\log 2\bigg\}.
 	\end{aligned}
 	\end{equation}
 \end{lemma}	 	
 
  Let $\omega = \exp(\theta)$, 
 and define  $$d_{\mathcal{M}}({{\omega}})=\max_{(k_1, k_2) \in \mathcal{M}} \Big(\frac{\omega_{\min(k_1,k_2)}}{\omega_{\min(k_1,k_2)+1}}+\cdots+\frac{\omega_{\max(k_1,k_2)-1}}{\omega_{\max(k_1,k_2)}}-|k_1-k_2|\Big).$$
 {We further have the following lemma to control the right hand side of \eqref{eqn:l5.4}. We provide the proof in Appendix Section~\ref{sec:KL}.
 	\begin{lemma}\label{lem:KL}
 		Under the BTL model, let $\mathcal{M}=\mathcal{M}(\theta)$ be the divider set defined in Definition~\ref{def:divider}. Given score vector $\omega$ ($\omega_1 > \cdots > \omega_n$) and the set of hypotheses $\mathcal{H}_{\mathcal{M}}$ constructed in Step 2, we have
 		\begin{equation}
 		\label{equ:KL}
 		\sum_{H,\tilde{H}\in \mathcal{H}_{\mathcal{M}}} \sum_{i < j} \mathrm{KL}\left(\mathbb{P}_{y_{i, j}^{(1)} \mid H} \| \mathbb{P}_{y_{i, j}^{(1)} \mid \tilde{H}}\right)
 		\leq 
 		4n|\mathcal{M}|^2 \big(d^2_{\mathcal{M}}({{\omega}})+\cO(d^3_{\mathcal{M}}({{\omega}})) \big).
 		\end{equation}
 	\end{lemma}
 }
 Plugging \eqref{equ:KL} into \eqref{eqn:l5.4}, we have  $\bar{\P}_{\mathrm{e}, \mathcal{M}} >1-\epsilon $ if 
 \begin{equation}\label{eqn:eps}
\frac{1}{\log (|\mathcal{M}|+1)}\bigg\{4npL \frac{|\mathcal{M}|^{2}}{(|\mathcal{M}|+1)^{2}} \big(d^2_{\mathcal{M}}({{\omega}})+\cO(d^3_{\mathcal{M}}({{\omega}})) \big)+\log 2\bigg\} < \epsilon.
 \end{equation}

Finally, {we have the following lemma, and the proof is provided in Appendix Section~\ref{sec:necessary condition}.
 	 \begin{lemma} \label{lem:necessary condition}
 		Under the BTL model where the pairwise comparison probability $p$ in Erd\"{o}s-R\'{e}nyi Graph satisfies { $p \gtrsim \frac{\log n}{n}$}, let  $\mathcal{M}=\mathcal{M}(\theta)$ be the divider set defined in Definition \ref{def:divider}, and $\theta$ ($\theta_{1} > \theta_{2} > \cdots > \theta_{n}$) be the score vector in Step 2. Assuming that
 		$$
 		\Delta({\theta}) 
 		\lesssim \sqrt{\frac{\epsilon \log(|\mathcal{M}|+1)-\log 2}{npL}},
 		$$
 		we have	
 		\begin{equation}
 		\label{equ:distance-omega}
 		d_\mathcal{M}({{\omega}}) 
 		\lesssim \sqrt{\frac{\epsilon \log(|\mathcal{M}|+1)-\log 2}{npL}}.
 		\end{equation}
 \end{lemma}
}

We have that \eqref{equ:distance-omega}
implies \eqref{eqn:eps}. This essentially concludes the proof that when \eqref{equ:LB} holds, we have the minimax risk $\mathfrak{R}\ge \bar{\P}_{\mathrm{e}, \mathcal{M}} >1-\epsilon $.
 \end{proof}

\subsection{Applications}\label{sec:LB:example}
We provide some examples of ranking property testings to illustrate the lower bound.

\begin{example}[Top-$K$ items inference]\label{eg:topk}
Consider the problem of selecting all items ranked among top-$K$.  We construct a base preference score ${\theta}$ satisfying $\theta_{1}=\cdots=\theta_{K} > \theta_{K+1}=\cdots=\theta_{n}$.
Here {the target select set} is $\{i: \gamma(\theta) \in \cR_i\}=\{1, \cdots, K\}$.  Intuitively, for this example, if $\theta_{K}$ and $\theta_{K+1}$ are close, this multiple testing problem becomes difficult, and Theorem \ref{thm:LB} justifies this intuition. 

In particular, since swapping scores $\theta_{i}$  ($1 \leq i \leq K$) with $\theta_{j}$ ($K+1 \leq j \leq n$)  changes the top $K$ items, and we have $|\theta_i-\theta_j|=\theta_{K}-\theta_{K+1}$, we have $\big\{(i,j)\big\}_{ 1 \leq i \leq K, K+1 \leq j \leq n }$ are all legal pairs, and the distance is $\Delta({\theta})=\theta_{K}-\theta_{K+1}$. Among all legal pairs, after swapping scores $\theta_{i}$  ($1 \leq i \leq K$) with $\theta_{j}$ ($K+1 \leq j \leq n$), the target selection set becomes $ \big\{[n]/\{i\}\big\} \cup \{j\}$. Thus, all legal pairs are included in divider set $\mathcal{M}(\theta)=\{(i,j)\}_{1 \leq i \leq K, K+1 \leq j \leq n}$, which gives us that $\log(|\mathcal{M}(\theta)|) \asymp \log(n)$.
Thus, Theorem~\ref{thm:LB} implies that in this example, the minimax risk $\mathfrak{R}$ is uncontrollable if $$\Delta({\theta})=\theta_K-\theta_{K+1} \lesssim  \sqrt{ \frac{\epsilon\log(n) }{npL}}.$$
\end{example}

\begin{example}\label{eg:relative}
	Consider the problem of inferring all items  ranked higher than  item $k$ with score~$\theta_{k}$.  We construct a base preference score ${\theta}$ satisfying  $\theta_{1}=\cdots=\theta_{k-1} >\theta_{k}=\cdots=\theta_{n}$. In this example, we have $\{i: \gamma(\theta) \in \cR_i\}=\{1, \cdots, k-1\}$. Swapping score $\theta_{i}$  ($1 \leq i \leq k-1$) with $\theta_{j}$ ($k \leq j \leq n$) changes the target selection set $\{i: \gamma(\theta) \in \cR_i\}=\{1, \cdots, k-1\}$, and  $ |\theta_i-\theta_{j}|=\theta_{k-1}-\theta_{k}$. Thus, $\big\{(i,j)\big\}_{ 1 \leq i \leq k-1, k \leq j \leq n }$ are all legal pairs, and the distance is $\Delta({\theta})=\theta_{k-1}-\theta_{k}$. Among all legal pairs, after swapping score $\theta_{i}$  ($1 \leq i \leq k-1$) with $\theta_{k}$, we have the same target selection set $ \{i: \gamma(\theta) \in \cR_i\} = \emptyset $, so only one of them can be included in set $\mathcal{M}(\theta)$. We then have $\mathcal{M}(\theta)=\big\{(i,j)\big\}_{ 1 \leq i \leq k-1, k+1 \leq j \leq n } \cup \big\{(k-1,k)\big\}$ with $\log(|\mathcal{M}(\theta)|)=\log((k-1)(n-k) + 1) \asymp \log(n)$.
	Theorem~\ref{thm:LB} implies that in this example, the minimax risk $\mathfrak{R}$ is uncontrollable if $\Delta({\theta})=\theta_{k-1}-\theta_{k} \lesssim  
	\sqrt{\frac{\epsilon \log(n) }{npL}}$.
\end{example}

\begin{remark}[Upper Bound]
In the above two examples, we have  that if $\Delta({\theta}) \lesssim  
\sqrt{\frac{\epsilon \log n}{npL}}$, any procedure fails to control the  risk. Meanwhile, Theorem \ref{thm:FWER} shows that if $\Delta({\theta}) \gtrsim  
\sqrt{\frac{\log n}{npL}}$, we can control the FWER and achieve power one asymptotically at the same time, which shows our procedure achieves the minimax optimality.

\end{remark}

%


%% file: simulation.tex

\section{Numerical Experiments}\label{sec:simulation}

In this section, we conduct extensive numerical studies to test the empirical performance of the proposed methods using both synthetic data and two real datasets.

\subsection{Synthetic Data}

Using synthetic data, we first investigate the asymptotic normality of Lagrangian debiased estimators for latent preference scores. Specifically, we generate the latent scores $\theta_i^*$'s independently from a uniform distribution over $[8,10]$,  where we set the number of items $n=100$. We let $L \in \{2,6,20\}$, $p \in \{1.25,2\}\times \frac{\log(n)}{n}$. Following the procedure developed in Section~\ref{sec:debias}, we repeat the generating scheme 2,500 time and present the empirical distribution of the estimator. In particular, Figure \ref{fig:qqplot} displays the Q-Q Plots of $\theta_1$, and we find that the empirical distribution of our estimator is closed to a normal distribution, especially when the number of repeated comparisons  $L$ is large. This justifies our result in Theorem \ref{thm:asy} that our estimator weakly converges to a normal distribution.


%

\begin{figure}[h]
	\begin{center}
		\includegraphics[height=.8\textwidth,angle=270]{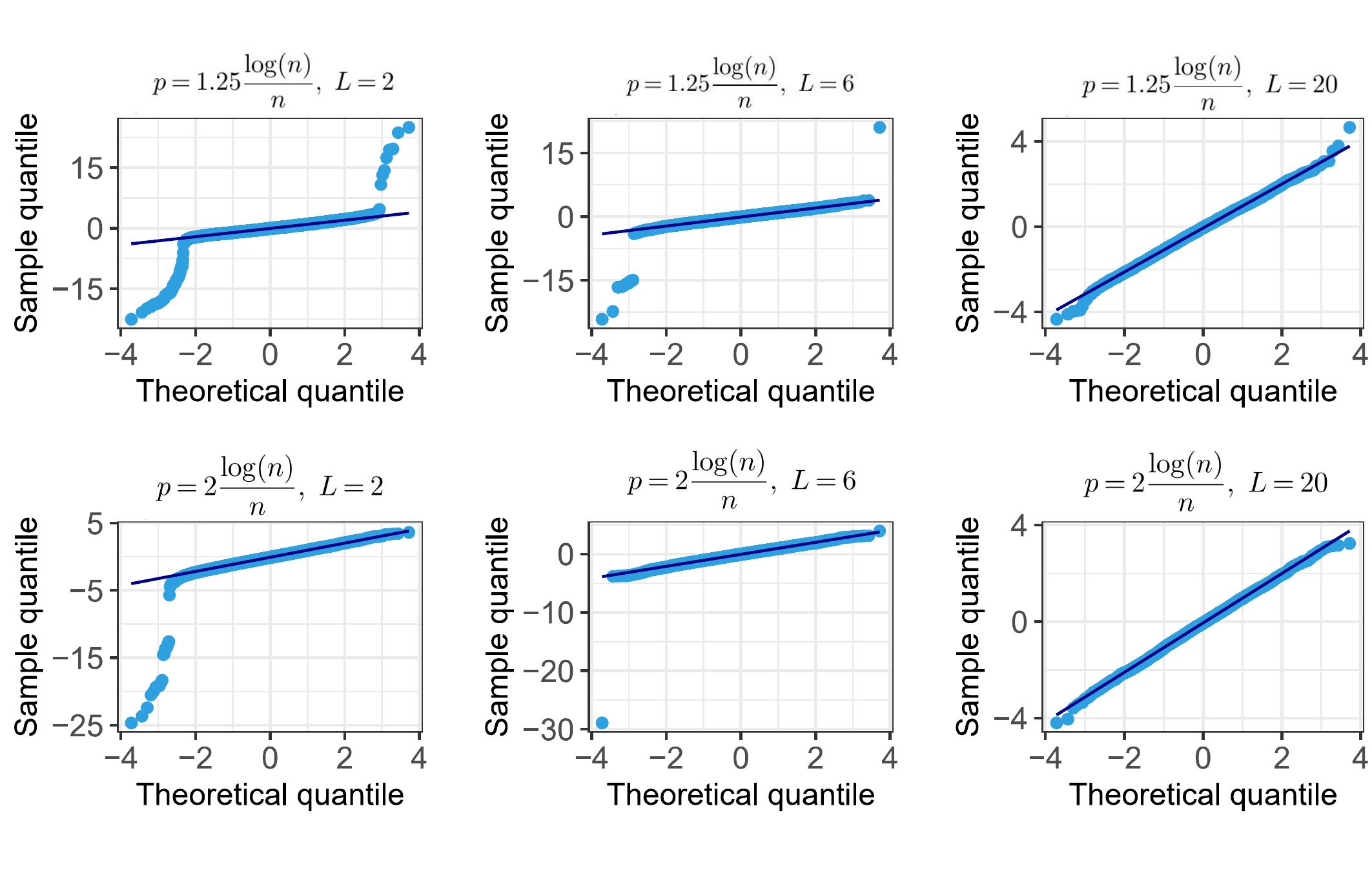}
	\end{center}
	\vskip-10pt
	\caption{Q-Q Plots for Lagrangian debiased estimators, comparing quantiles of standardized Lagrangian debiased estimators with standard normal. We fix $n=100$ and let $L \in \{2,6,20\}$, $p \in \{1.25,2\}\times \frac{\log(n)}{n}$.}
	\label{fig:qqplot}
\end{figure}

%



Next, we examine the performance of the pairwise test and top-$K$ test procedure in Section \ref{sec:ht}.  
For pairwise test, we test $
H_0: \theta^*_1 \leq \theta^*_2  \text{\ \ v.s.\ \ } H_a: \theta^*_1>\theta^*_2$. Figure \ref{fig:test1} (A) displays the Type I error with $p=0.2$ for different total number of items $n$ and  number of repeated comparisons $L$. Here we fix $\theta^*_1=\theta^*_2=10$ and $\theta^*_i=7.5$ for $3 \leq i \leq n$. As  seen from this figure, the empirical Type I error rate is close to the nominal $\alpha=0.05$.  Figure  \ref{fig:test1} (B) shows the empirical power of this test with different $\Delta = |\theta_1^* - \theta_2^*|$. Here we set $\theta^*_1=10, \theta^*_2=10-\Delta$ and $\theta^*_i=7.5$ for $3 \leq i \leq n$, and  we let $n=100$ and $p=0.2$ while change $\Delta$ and $L$. We observe that the empirical power goes to 1 quickly.

\begin{figure}[h]
	\begin{center}
		\begin{tabular}{cc}	
			{\small (A) Type I error} &{\small (B) Power}\\		
			\includegraphics[width=.45\textwidth,angle=0]{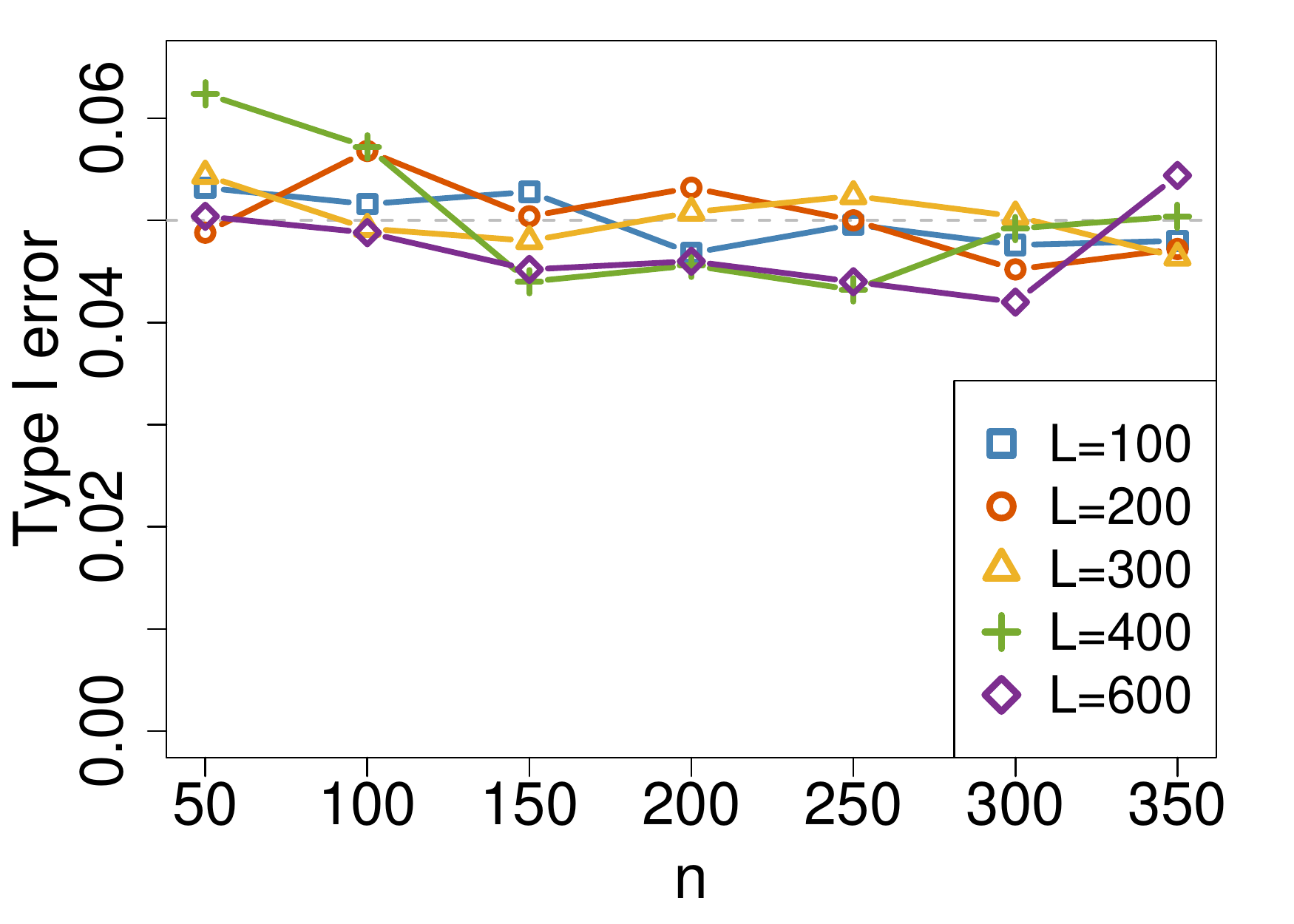} 
			& \includegraphics[width=.45\textwidth,angle=0]{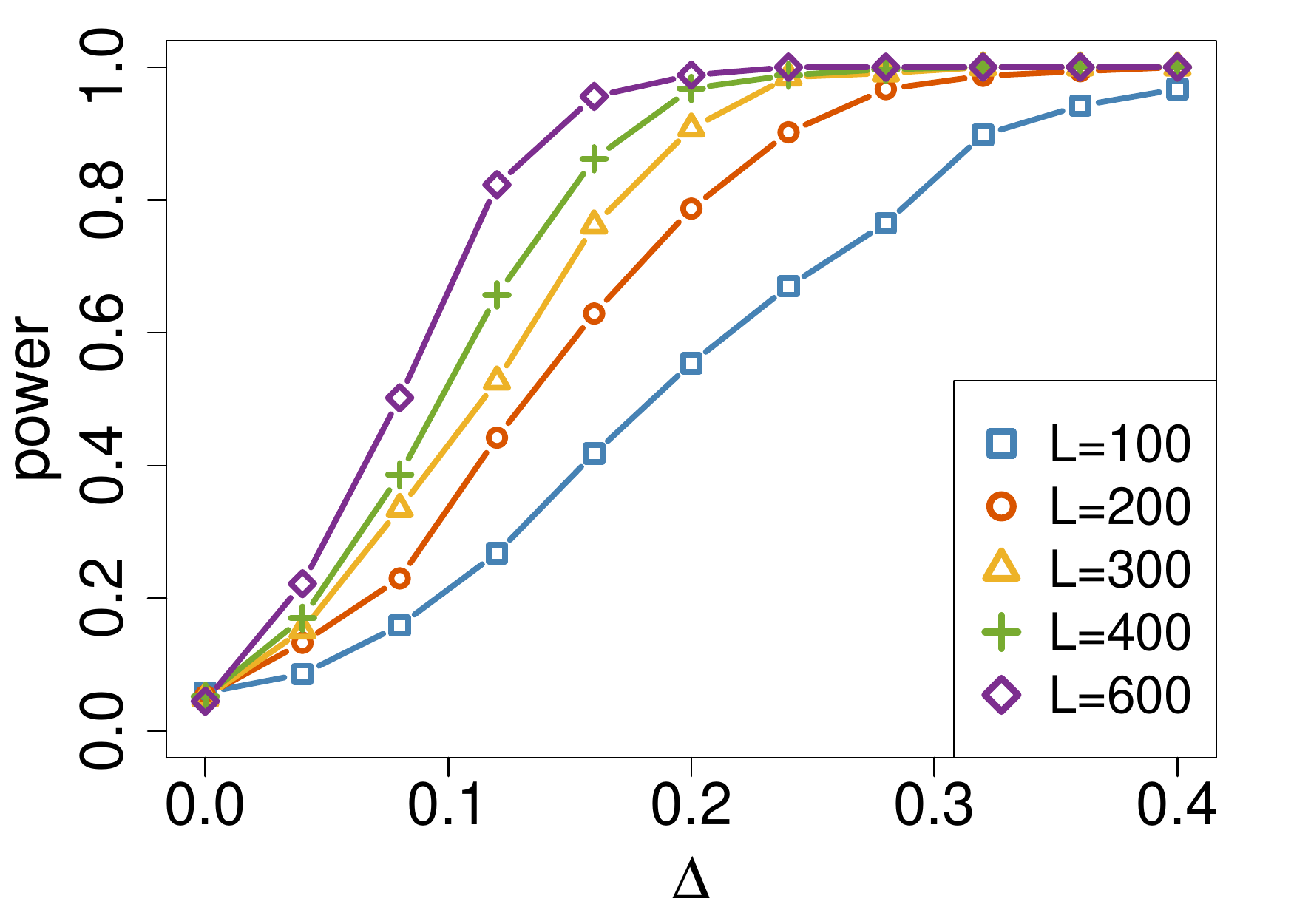} \\
			[-5pt]
		\end{tabular}
	\end{center}
	\caption{The performance of pairwise test $H_0: \theta^*_1 \leq \theta^*_2  \text{\ \ v.s.\ \ } H_a: \theta^*_1>\theta^*_2$. (A) Type I error for our proposed pairwise test procedure with different $n$ and $L$. (B) Power curve as a function of signal strength $\Delta = |\theta^*_1 - \theta^*_2|$ with $n=100$, $p=0.2$ and $L \in \{100,200,300,400,600\}$.} 
	\label{fig:test1}
\end{figure}

For top-$K$ test, we test if item $K+1$ is ranked among the top $K$ items. We fix $K=30$ and $p=0.2$, and we let $\theta^*_i=10$  for $1 \leq i \leq K+1$, and $\theta^*_j=7.5$ for $K+2 \leq j \leq n$. Figure \ref{fig:topk} (A) displays the  empirical Type I error  rate with different $n$ and $L$, and we find that that the empirical Type I error is close to the nominal $\alpha=0.05$. 
Figure  \ref{fig:topk} (B) displays the empirical power of this test with different separation $\Delta$ between top $K$ items and other items that $\Delta = |\theta^*_{(K)} - \theta^*_{(K+1)}|$. We let $n=100$, $L \in \{100,200,300,400,600\}$, and we  set $\theta^*_i=10$  for $1 \leq i \leq K$, $\theta^*_{K+1}=10-\Delta$, and $\theta^*_j=7.5$ for $K+2 \leq j \leq n$. As  seen
from this plot, the empirical power   goes to 1 as $\Delta$ and $L$ increase.

\begin{figure}[]
	\begin{center}
		\begin{tabular}{cc}	
			{\small (A) Type I error} &{\small (B) Power}\\		
			\includegraphics[width=.45\textwidth,angle=0]{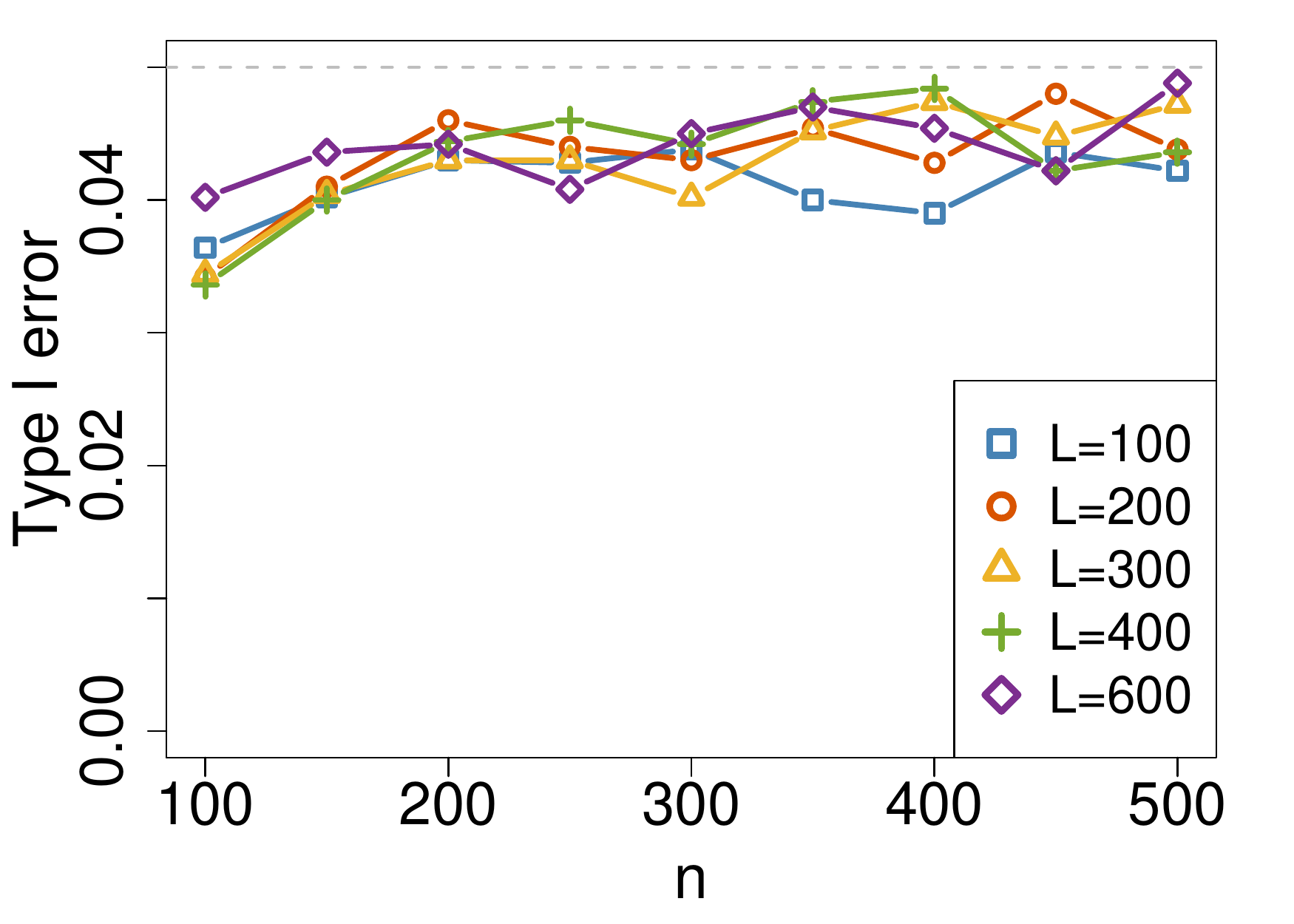} &
			\includegraphics[width=.45\textwidth,angle=0]{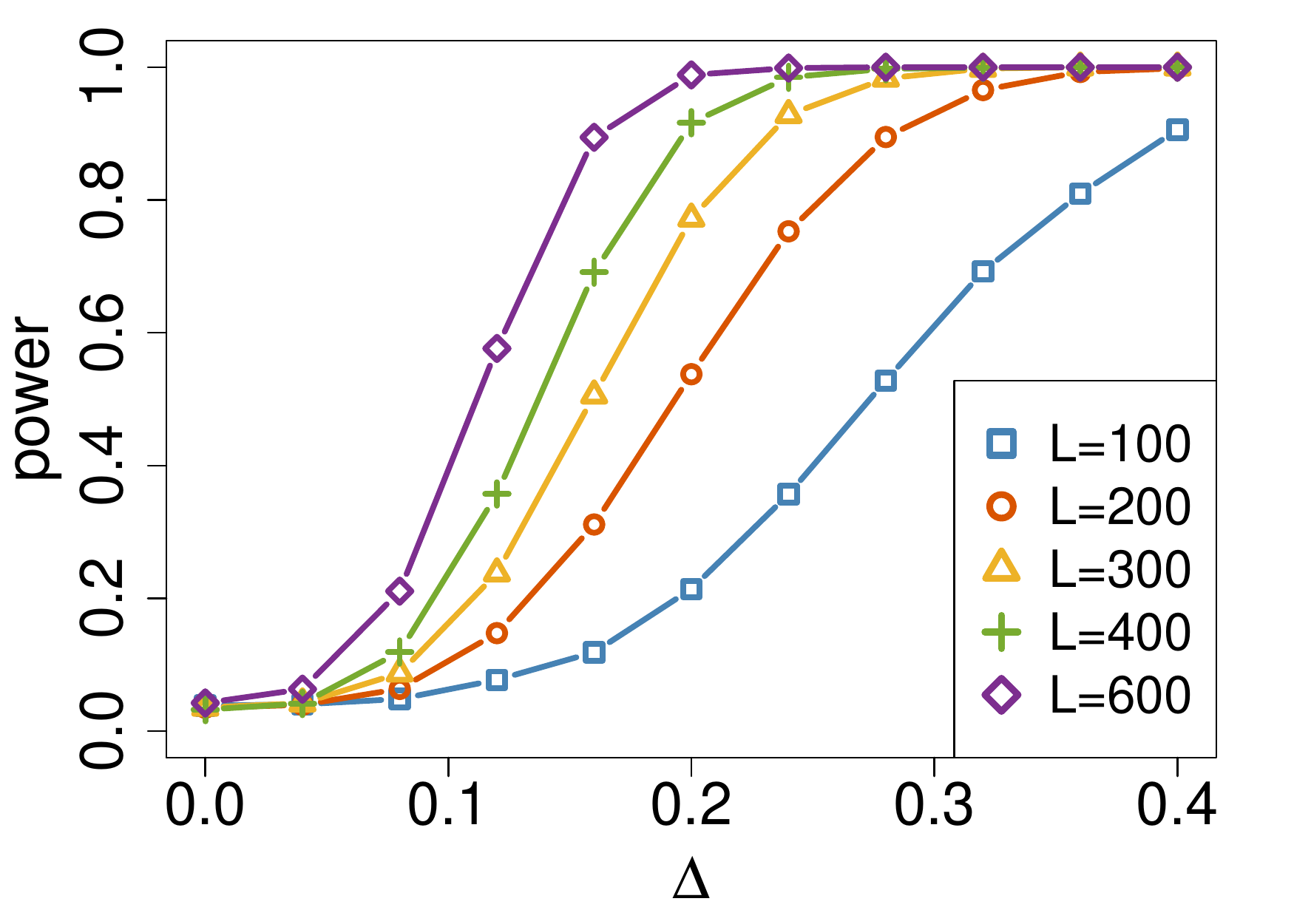} \\
			[-5pt]
		\end{tabular}
	\end{center}
	\caption{Performance of top-30 test using our proposed simplified testing procedure in Remark \ref{rmk:topk}: (A) Averaged Type I error with different $n$ and $L$. (B) Averaged power as a function of signal strength $\Delta = |\theta^*_{(K)}-\theta^*_{(K+1)}|$ with $n=100$, $p=0.2$, and $L \in \{100,200,300,400,600\}$.} 
	\label{fig:topk}
\end{figure}

Finally, we evaluate the empirical performance of our FDR procedure in Section \ref{sec:FDR}  by considering the top-$K$ test. We let $K=30$, $\theta^*_i=10$  for $1 \leq i \leq K+10$, and $\theta^*_j=6.5$ for $K+11 \leq j \leq n$.
Figure  \ref{fig:fdr} (A) displays the empirical FDR based on our procedure with $p=0.2$ and different $n$, $L$.   This figure  illustrates that the FDR is well controlled below the nominal $\alpha=0.05$, consistent with our results in Theorem \ref{thm:FDR}.  Figure  \ref{fig:fdr}~(B) displays the empirical power of the FDR procedure, which is defined as true positive rate, with different separation $\Delta = |\theta^*_{(K)} - \theta^*_{(K+1)}|$.  We set $\theta^*_i=10$ for $1 \leq i \leq K$, $\theta^*_{K+1}=10-\Delta$, and $\theta^*_j=6.5$ for $K+2 \leq j \leq n$, and we let $n=100$, $p=0.2$, and $L \in \{100,200,300,400,600\}$. As shown
in this plot, the empirical power increases and goes to 1 as $\Delta$ and $L$ increase, showing that the proposed FDR procedure is able to identify the top $K$ items with well controlled FDR.

\begin{figure}[H]
	\begin{center}
		\begin{tabular}{cc}		
			{\small (A) FDR} &{\small (B) FDR power}\\	
			\includegraphics[width=.45\textwidth,angle=0]{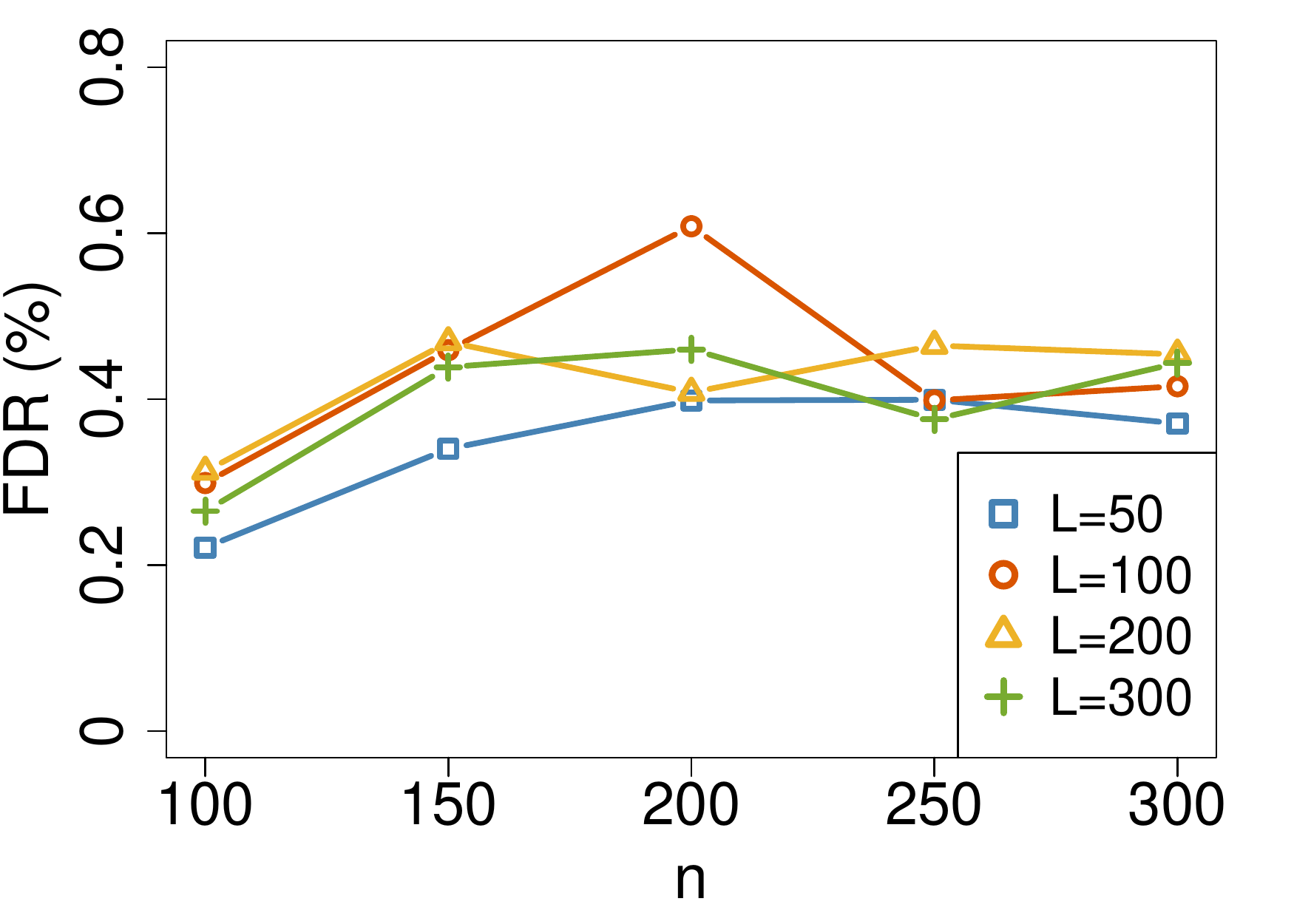} & 
			\includegraphics[width=.45\textwidth,angle=0]{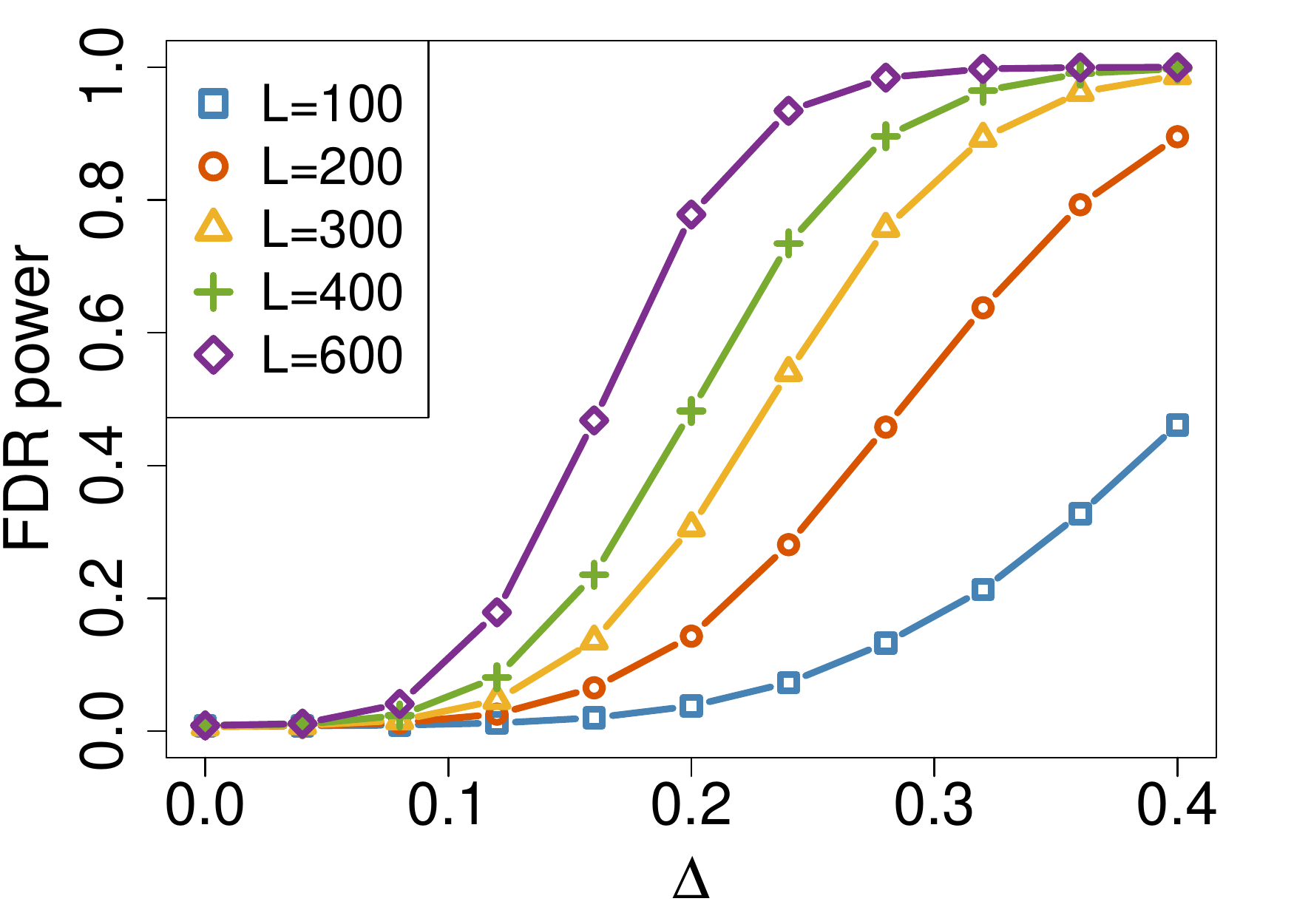} \\
			[-5pt]
		\end{tabular}
	\end{center}
	\caption{Performance of Benjamini-Yekutieli based FDR procedure for selecting top-30 items. (A) Empirical FDR with different $n$ and $L$. (B) Empirical FDR power as a function of signal strength $\Delta = |\theta_{(K)} - \theta_{(K+1)}|$ with $n=100$, $p=0.2$, and $L \in \{100,200,300,400,600\}$.}
	\label{fig:fdr}
\end{figure}

\subsection{Real-World Data} 
In this section, We apply our  method to analyze two real datasets. 

\paragraph{Jester Dataset.}
We first apply our  method to the Jester dataset from \cite{goldberg2001eigentaste}. This dataset  contains ratings of 100 jokes from 73,421 users. The more detailed description and the dataset is available through \url{http://eigentaste.berkeley.edu/dataset/}. In this dataset, 14,116 users rated all 100 jokes, while others only rated some of jokes. We only use samples from users who ranked all 100 jokes for our experiments. 
Since we  need pairwise comparisons for our ranking analysis, we generate
Erd\"{o}s-R\'{e}nyi comparison graph randomly with $p=0.3$ and obtain each pairwise comparison results based on the relative rating of compared pairs by the same user. To be specific, if joke 1 receives a higher rating than joke 2 from a same user, we have joke 1 beats joke 2 in this comparison. \cite{negahban2017rank, kim2017latent} also use similar approaches to break rating results into pairwise comparisons. Further, we randomly choose $L$ samples from the total 14,116 samples.
 

Table \ref{tab:joke} displays the top 10 jokes' IDs and their estimated scores obtained from the spectral method and our debiasing method with $L=1000$.  Furthermore, we evaluate the performance of our FDR-controlling procedure with $K \in \{20,40\}$ and $L \in \{1000,5000\}$. Figure \ref{fig:joke_fdr} presents the p-values from multiple testing and threshold obtained from Benjamini-Yekutieli procedure described in Section~\ref{sec:FDR}. It shows that our FDR procedure gains more power as $L$ increases.

\renewcommand{\arraystretch}{0.75}
\begin{table}[h!] 
	\centering
	\caption{The top 10 jokes on Jester Dataset and their estimated scores obtained from Spectral method and our Lagrangian debiasing method.} 
	\begin{tabular}{ccccc}
		\hline
		& \multicolumn{2}{c}{Spectral method} & \multicolumn{2 }{c}{Debiasing method} \\ \hline		
				Joke ID	 & Score  & Rank  & Score    & Rank    \\ \hline
     89     & 0.841  & 1      & 0.840   &1   	\\
     50     & 0.799  & 2      & 0.801   & 2  	\\
     29     & 0.651  & 3      & 0.645   & 3  	\\
     36     & 0.623  & 4      & 0.628   & 4  	\\
     27     & 0.621  & 5      & 0.620   & 5  	\\
     62     & 0.616  & 6       & 0.616   & 6  	\\
     32    & 0.603  & 7       & 0.599   & 7  	\\
     35    & 0.596  & 8       & 0.596   & 8  	\\
	 54    & 0.527  & 9       & 0.526   & 9  	\\
     69    & 0.515  &10      & 0.516   & 10  	\\ \hline \\
	\end{tabular}
	\label{tab:joke}
\end{table}

 \begin{figure}[h!]
	\begin{center}
		\begin{tabular}{ccc}
			{\small  $(K = 20, L = 1000)$} 
			&{\small   $(K = 40, L = 1000)$} \\
			\includegraphics[width=.4\textwidth,angle=0]{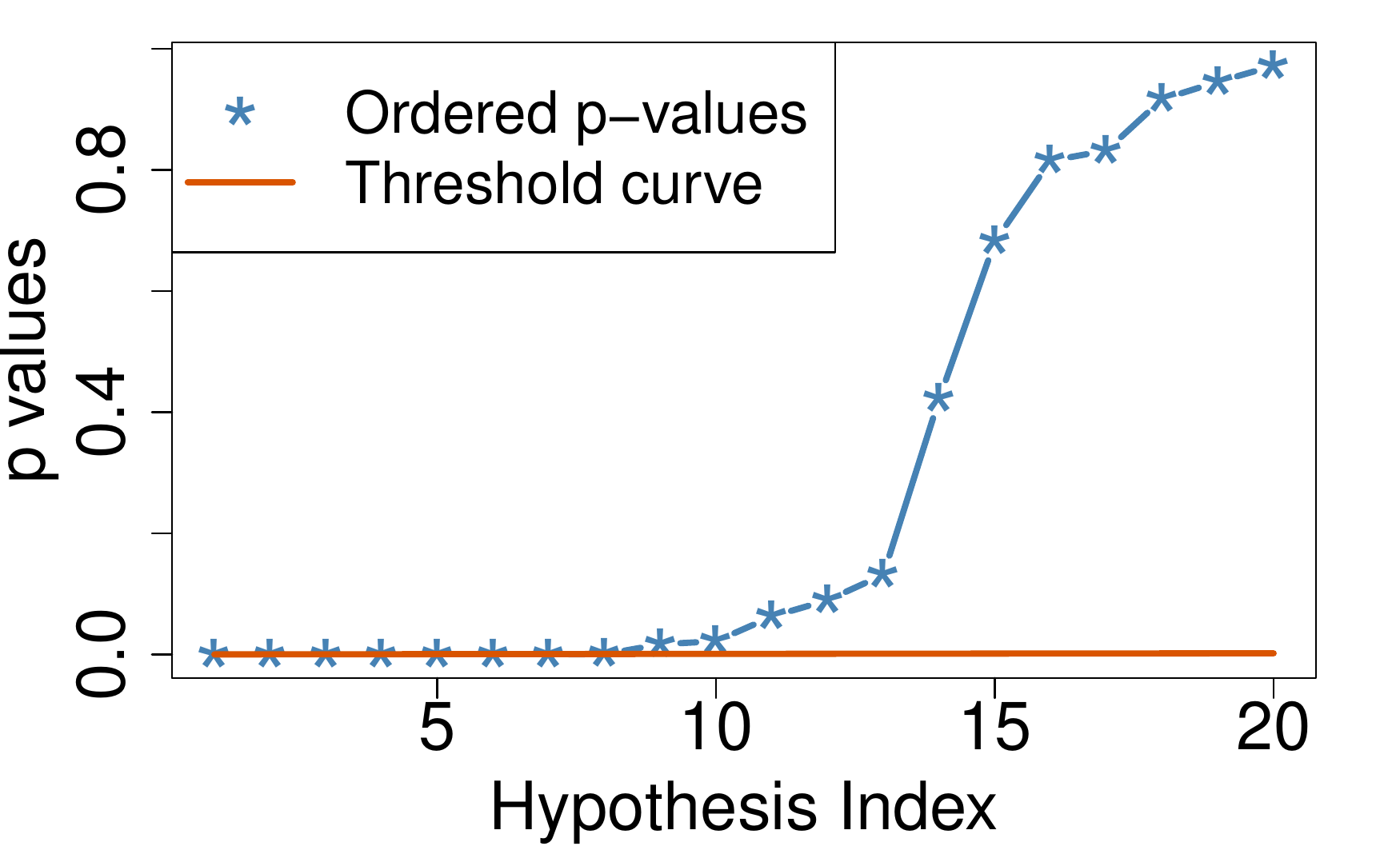} 
			& \includegraphics[width=.4\textwidth,angle=0]{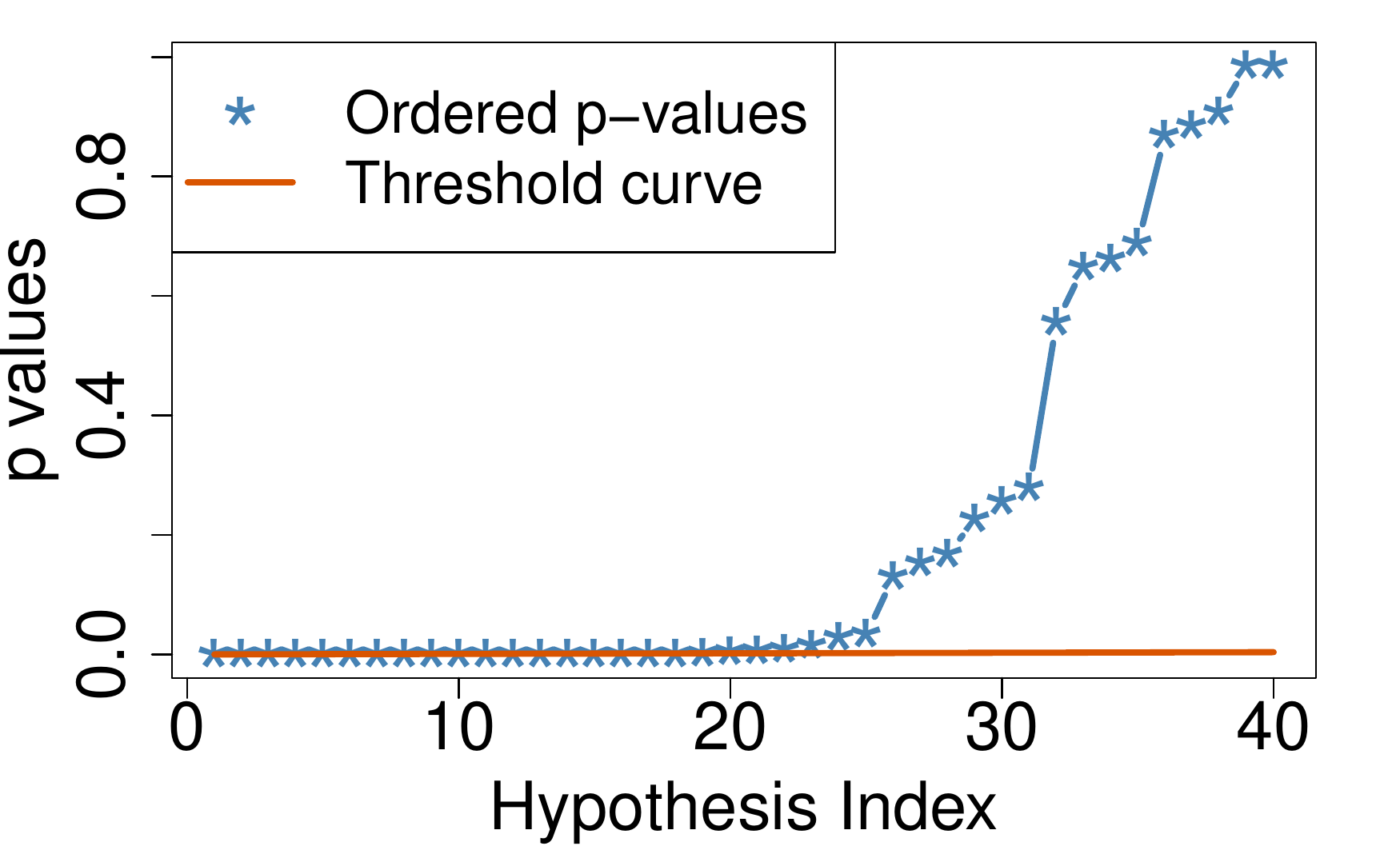} \\
				{\small  $(K = 20, L = 5000)$} 
			&{\small   $(K = 40, L = 5000)$} \\
			\includegraphics[width=.4\textwidth,angle=0]{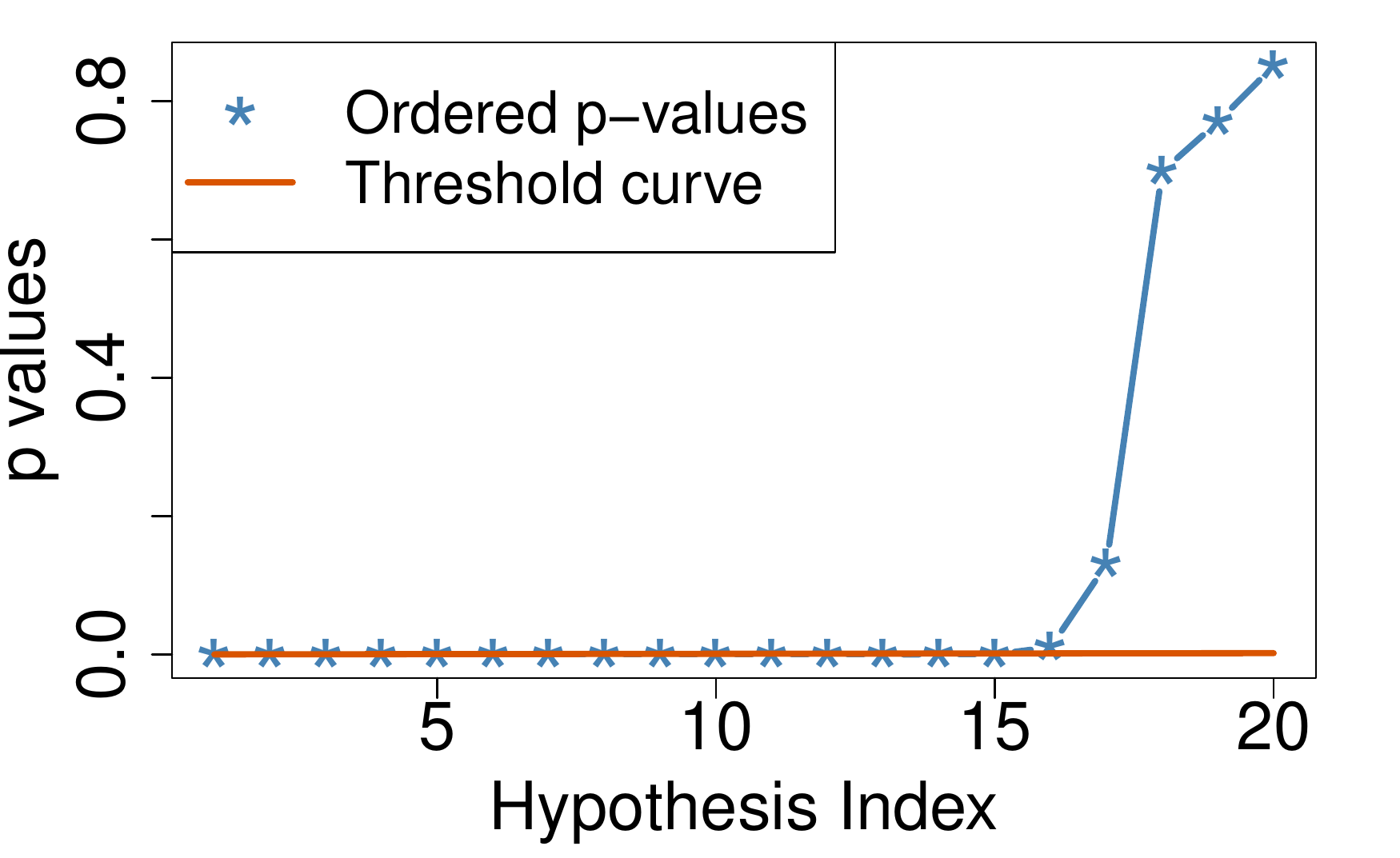} 
			&  \includegraphics[width=.4\textwidth,angle=0]{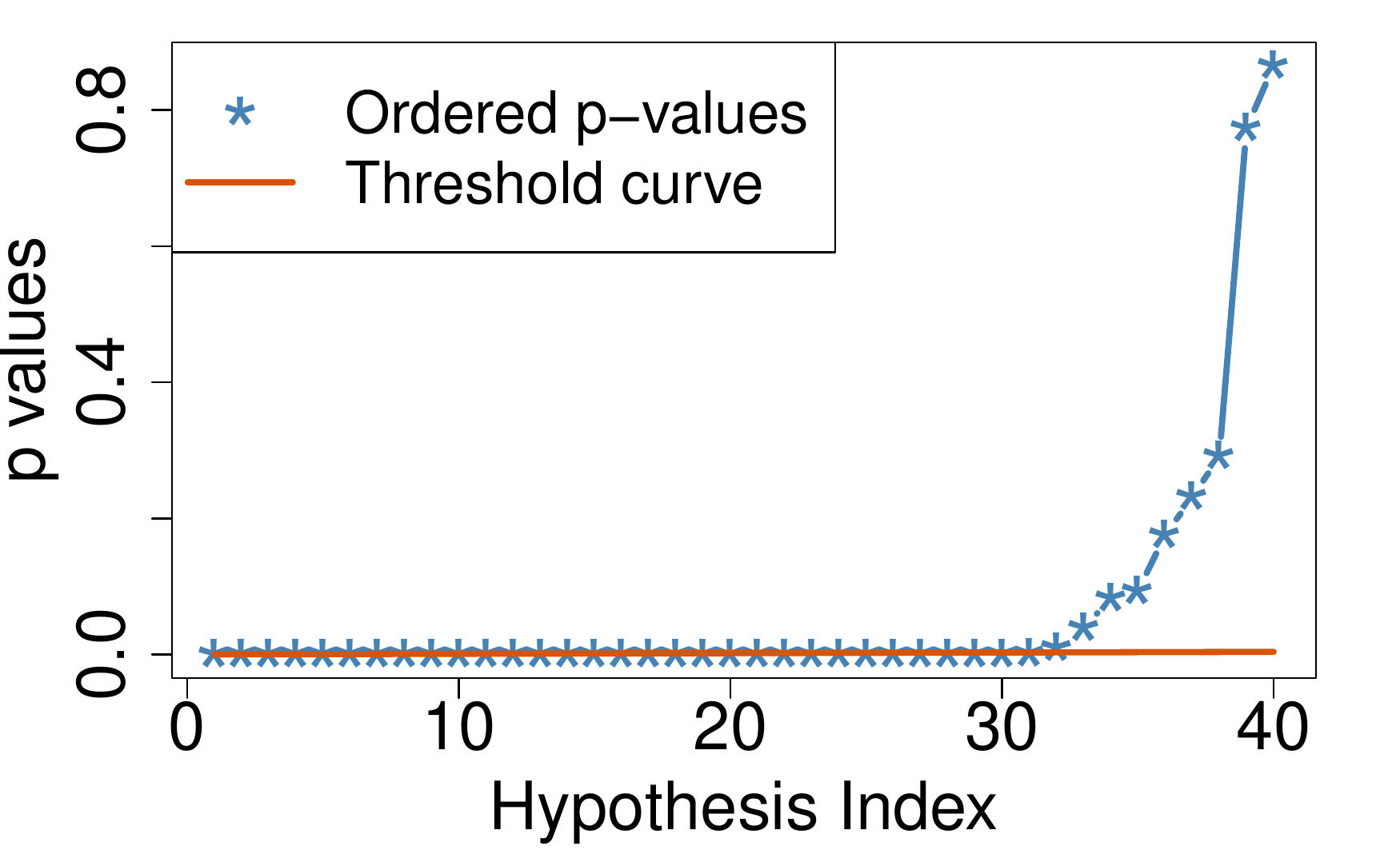} \\
			[-1pt]
		\end{tabular}
	\end{center}
	\caption{Application of our FDR procedure on Jester Dataset to select top $K$ jokes. The four panels display the ordered p-values and adjusted threshold by Benjamini-Yekutieli procedure with  $K \in \{20,40\}$ and $L \in \{1000,5000\}$. The horizontal  red line represents 0.05.} 
	\label{fig:joke_fdr}
\end{figure}

\paragraph{MovieLens Dataset}
We also apply our  method to analyze the MovieLens dataset  \citep{harper2015movielens}. 
Similar to our analysis before, we obtain pairwise comparisons results based on the relative ratings of two movies by the same user. In particular, we analyze $n=218$ movies with largest number of ratings, and randomly sample $L=1000$ comparisons. Table \ref{tab:movie} shows the top 10 movies with highest scores based on Lagrangian debiasing procedure and the corresponding p-values, where we test if they are ranked among the top 10 and top 20 movies. Figure \ref{fig:movie} displays the change of p-values when each movie is tested if it is among top 1, 5, 10, 20 ranked movies. We display the p-values of four movies (The Shawshank Redemption, The Matrix, The Usual Suspects, Schindler's List), which are ranked 1, 2, 5, 10 based on our debiased estimator. 

\begin{table}[h]
	\centering
	\caption{Top 10 movies in MovieLens Dataset based on Lagrangian debiased estimator. The table includes their titles, average rating in  \url{https://movielens.org/}, estimated scores by Lagrangian debiased procedure, and their corresponding p-values in top 10 and top 20 test by our proposed testing procedure in Remark \ref{rmk:topk}.}
	\resizebox{\textwidth}{!}{%
		\begin{tabular}{cccccc}
			\hline
			Rank & Movie Title                                           & Average rating & Debiased Score & P-value in top 10 test & P-value in top 20 test \\ \hline
			1  & The Shawshank Redemption (1994)           & 4.42 & 1.985 & $<1e-6$ & $<1e-6$        \\
			2  & The Matrix (1999)                         & 4.16 & 1.766 & $<1e-6$ & $<1e-6$        \\
			3  & The Godfather (1972)                      & 4.25 & 1.755 & $<1e-6$ & $<1e-6$        \\
			4  & Star Wars: Episode V - The Empire Strikes Back (1980) & 4.12           & 1.684          & $<1e-6$        & $<1e-6$                      \\
			5  & The Usual Suspects (1995)                 & 4.28 & 1.530 & $<1e-6$ & $<1e-6$        \\
			6  & Star Wars: Episode IV - A New Hope (1977) & 4.10 & 1.471 & 0.0010    & $<1e-6$        \\
			7  & The Silence of the Lambs (1991)           & 4.15 & 1.418 & 0.0070    & $<1e-6$        \\
			8  & Seven (a.k.a. Se7en) (1995)               & 4.08 & 1.367 & 0.0401    & $<1e-6$        \\
			9  & Lord of the Rings: The Fellowship of the Ring (2001)  & 4.10   & 1.329   & 0.1133               & $<1e-6$                      \\
			10 & Schindler's List (1993)                   & 4.25 & 1.288 & 0.3431    & 0.0002 \\ \hline \\
		\end{tabular}%
	}
	\label{tab:movie}
\end{table}

\begin{figure}[h!]
	\begin{center}
		\begin{tabular}{cc}			
			\includegraphics[height=.7\textwidth,angle=270]{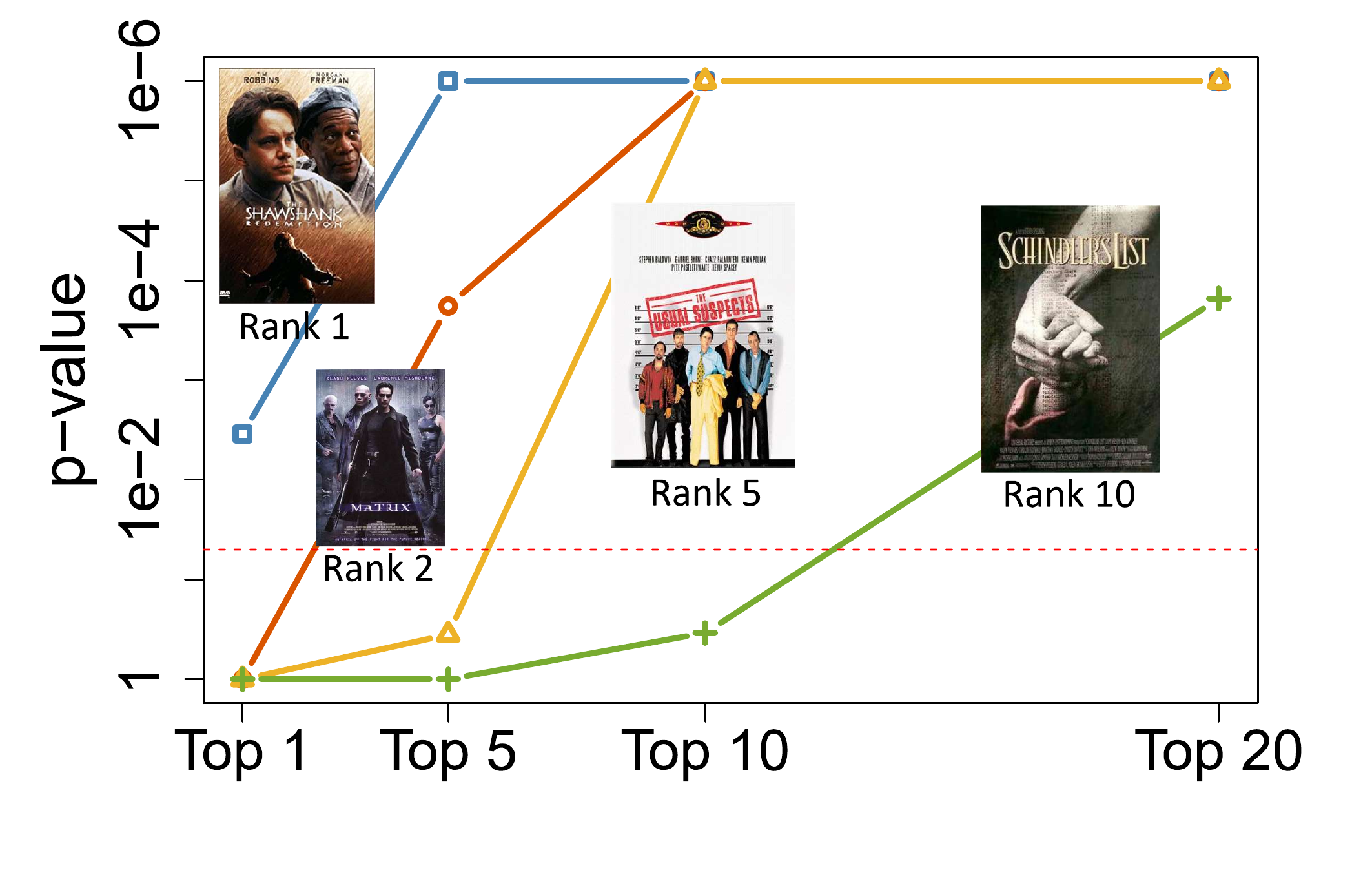} \\
			[-10pt]
		\end{tabular}
	\end{center}
	\caption{Application of top-$K$ test on four movies (The Shawshank Redemption, The Matrix, The Usual Suspects, Schindler's List) in MovieLens Dataset, which are ranked 1, 2, 5, 10 based on Lagrangian debiasing estimator. The figure displays the change of their p-values in top 1, 5, 10, 20 test by our proposed testing procedure in Remark \ref{rmk:topk}. The horizontal dotted red line represents 0.05}
	\label{fig:movie}
\end{figure}

%% file: discussion.tex
\section{Conclusion}\label{sec:discussion}

To conclude, to the best of our knowledge, we propose the first general framework for conducting inference and quantifying uncertainties for ranking problems. Under the BTL model, we first propose a Lagrangian debiasing method to infer the latent score for each item, where we can then test ``local" properties. Next, by leveraging the powerfulness of Gaussian multiplier bootstrap, we can
test more general ``global" properties. Furthermore, we extend the framework to multiple testing problems where we control both the familywise Type I error and the false discovery rate. We prove the optimality of the proposed method by deriving the minimax lower bound. Using both synthetic and real datasets, we demonstrate that our method works well in practice.

There are still numerous promising directions that would be of interest for future investigations. We point out a few possibilities as follows. First, the Gaussian multiplier bootstrap approach is computationally expensive, and it is worth investigating if we can develop a more computationally  efficient approach to make the method more scalable. Second, as we have discussed, the current approach for false discovery rate control is conservative, and we plan to develop a more powerful method to tightly control the false discovery rate. In addition, the ranking of different items may change over  time, and we plan to develop new models to study dynamic ranking systems.

%% file: appendix.tex

%
\newpage
%
%
%
%
\setcounter{section}{0}
\renewcommand{\thesection}{\Alph{section}}
%
%

\renewcommand{\theHchapter}{A\arabic{chapter}}

\renewcommand{\theHsection}{A\arabic{section}}

\section{Proof of Theorems in Section \ref{sec:debias}} \label{sec:pf-asy}

\subsection{Proof of Theorem \ref{thm:gen-asy}}
\label{sec:pf:gen-asy}
\begin{proof}We prove the asymptotic normality of the general Lagrangian debiasing estimator. 
First we decompose the general Lagrangian debiasing estimator as
\begin{equation}\label{eqn:decompose}
\begin{aligned}
{
	\left( \begin{array}{ccc}
	\widehat{{\theta}}^{d}-{{\theta}}^*\\
	\lambda
	\end{array}
	\right )}
=&{\
	\left( \begin{array}{ccc}
	\widehat{{\theta}}^{d}-\hat{{\theta}}\\
	\lambda
	\end{array}
	\right )}
+{
	\left( \begin{array}{ccc}
	\hat{{\theta}}-{{\theta}}^*\\
	0
	\end{array}
	\right )}\\
=&\ 
\hat{\Theta}
{
	\left( \begin{array}{ccc}
	-\nabla {\cL({\hat{{\theta}}})} \\
	-f(\hat{{\theta}})
	\end{array} 
	\right )}
+\widehat{\Theta}\widehat{\Sigma}
{
	\left( \begin{array}{ccc}
	\hat{{\theta}}-{{\theta}}^*\\
	0
	\end{array}
	\right )}
+ (I-\widehat{\Theta}\widehat{\Sigma})
{
	\left( \begin{array}{ccc}
	\hat{{\theta}}-{{\theta}}^*\\
	0
	\end{array}
	\right )}\\
=&\
\hat{\Theta}
{
	\left( \begin{array}{ccc}
	-\nabla {\cL({\hat{{\theta}}})} \\
	-f(\hat{{\theta}})
	\end{array} 
	\right )}
+\widehat{\Theta}
{
	\left( \begin{array}{ccc}
	\nabla^2 \cL(\hat{{\theta}})(\hat{{\theta}}-{\theta}^*)  \\
	\nabla f(\hat{{\theta}})^\top(\hat{{\theta}}-{\theta}^*)
	\end{array} 
	\right )}
+ (I-\widehat{\Theta}\widehat{\Sigma})
{
	\left( \begin{array}{ccc}
	\hat{{\theta}}-{{\theta}}^*\\
	0
	\end{array}
	\right )} \\
=&\ \widehat{\Theta}
{
	\left( \begin{array}{ccc}
	-\nabla \cL(\hat{{\theta}})+ \nabla^2 \cL(\hat{{\theta}})(\hat{{\theta}}-{\theta}^*)  \\
	-f(\hat{{\theta}})
	+ \nabla f(\hat{{\theta}})^\top(\hat{{\theta}}-{\theta}^*)
	\end{array} 
	\right )}
+ 
(I-\widehat{\Theta}\widehat{\Sigma})
{
	\left( \begin{array}{ccc}
	\hat{{\theta}}-{{\theta}}^*\\
	0
	\end{array}
	\right )}\\
=&\ {\Theta}^*
{
	\left( \begin{array}{ccc}
	-\nabla \cL({{\theta}}^*)\\
	-f({{\theta}}^*)
	\end{array} 
	\right )}
+ \underbrace{(\widehat{\Theta}-{\Theta}^*)
	\left( \begin{array}{ccc}
	-\nabla \cL(\hat{{\theta}})+ \nabla^2 \cL(\hat{{\theta}})(\hat{{\theta}}-{\theta}^*)  \\
	-f(\hat{{\theta}})
	+ \nabla f(\hat{{\theta}})^\top(\hat{{\theta}}-{\theta}^*)
	\end{array} 
	\right )}_{J_1}\\
&	+ \underbrace{{\Theta}^*
	\left( \begin{array}{ccc}
	\nabla \cL({{\theta}}^*)-\nabla \cL(\hat{{\theta}})+ \nabla^2 \cL(\hat{{\theta}})(\hat{{\theta}}-{\theta}^*)  \\
	f({{\theta}}^*)-f(\hat{{\theta}})
	+ \nabla f(\hat{{\theta}})^\top(\hat{{\theta}}-{\theta}^*)
	\end{array} 
	\right )}_{J_2}
+ \underbrace{
(I-\widehat{\Theta}\widehat{\Sigma})
	\left( \begin{array}{ccc}
	\hat{{\theta}}-{{\theta}}^*\\
	0
	\end{array}
	\right )}_{J_3},
\end{aligned} 	
\end{equation}
where the second equality comes from \eqref{eqn:general debias}, the third equality comes from  \eqref{equ:hat sigma}.
Then we bound $J_1$, $J_2$ and $J_3$ in the following lemma, and the proof is provided in Section \ref{sec:pf-J1}.

\begin{lemma}\label{lem:J1}
Under the conditions of Theorem \ref{thm:gen-asy},
we have 
\begin{equation}\label{eqn:J1}
\begin{aligned}
\|J_1\|_\infty
\lesssim
(r_1+r_4)(r_1^2 + r_2),
\end{aligned}
\end{equation}	
\begin{equation}\label{eqn:J2}
\begin{aligned}	
\|J_2\|_\infty
\lesssim r_1^2 r_5 ,
\end{aligned}
\end{equation}
\begin{equation}\label{eqn:J3}
\begin{aligned}
\|J_3\|_\infty
\lesssim r_1 r_3,
\end{aligned}
\end{equation}
where $r_1$, $r_2$, $r_3$, $r_4$ and $r_5$ are defined in Assumptions \ref{ass:consistency} -- \ref{ass:omega}.
\end{lemma}

%
%
%

	Plugging  \eqref{eqn:J1}, \eqref{eqn:J2}, \eqref{eqn:J3} into \eqref{eqn:decompose}, we have $$\|J_1+J_2+J_3\|_\infty
	\lesssim r_1^2r_5  + (r_1+r_4)(r_1^2 + r_2)+ r_1r_3. $$ 
	Then, by our scaling assumption $$\sqrt{n}(r_1^2r_5  + (r_1+r_4)(r_1^2 + r_2)+ r_1r_3)=o(1),$$ and Assumption \ref{ass:CLT}, we immediately have
		$$\sqrt{n} \frac{\hat{{\theta}}_j^{d}-{{\theta}}^*_j}
	{\sqrt{[{\Theta}_{11}^* {\Sigma}^*_{11} {\Theta}_{11}^*]_{jj}}}
	=-\sqrt{n}
	\frac{[{\Theta}^*_{11}
	\nabla \cL({{\theta}}^*)]_j}
	{\sqrt{[{\Theta}_{11}^* {\Sigma}^*_{11} {\Theta}_{11}^*]_{jj}}} + \op(1) \overset{d}{\longrightarrow} N(0,1),$$
	and
	$$\sqrt{n} \frac{(\widehat{{\theta}}_i^{d}-{{\theta}}^*_i)-(\widehat{{\theta}}_j^{d}-{{\theta}}^*_j)} {\sqrt{(\bm{e}_i-\bm{e}_j)^\top({\Theta}_{11}^* {\Sigma}^*_{11} {\Theta}_{11}^*)(\bm{e}_i-\bm{e}_j)}}
	= -\sqrt{n}\frac{[{\Theta}^*_{11}
		\nabla \cL({\theta}^*)]_i-[{\Theta}^*_{11}
		\nabla \cL({\theta}^*)]_j} {\sqrt{(\bm{e}_i-\bm{e}_j)^\top({\Theta}_{11}^* {\Sigma}^*_{11} {\Theta}_{11}^*)(\bm{e}_i-\bm{e}_j)}}
	+ {\op(1)} \overset{d}{\longrightarrow} N(0,1),$$ 
which concludes our proof.
\end{proof}

\subsection{Proof of Theorem \ref{thm:asy}}
\label{sec:pf-asy-BTL} 
\begin{proof}
First, by simple algebra, we have that  the gradient and the Hessian of negative log likelihood function $\cL({\theta})$ defined in~\eqref{eqn:likelihood} are
\begin{equation}\label{eqn:grad} \nabla{\cL({\theta})}=\sum_{i>j}{\mathcal{E}_{ji}}(-y_{ji}+\frac{e^{\theta_i}}{e^{\theta_i}+e^{\theta_j}})(\bm{e}_i-\bm{e}_j),
\end{equation}
and
\begin{equation}\label{eqn:hess}
\nabla^2 {\cL({\theta})}= \sum_{i>j}{\mathcal{E}_{ji}}\frac{e^{\theta_i}e^{\theta_j}}{(e^{\theta_i}+e^{\theta_j})^2}(\bm{e}_i-\bm{e}_j)(\bm{e}_i-\bm{e}_j)^\top
\end{equation}
where $\cE_{ji} = 1$ if $(i,j)\in \cE$, $\cE_{ji} = 0$ otherwise, and $\cE$ is  the comparison graph.

Similar to the proof of Theorem \ref{thm:gen-asy},
by simple algebra, $\hat{{\theta}}^d-{\theta}^*$ can be decomposed as
\begin{small}
\begin{equation}
\label{equ:decomposition}
\begin{aligned}
{
	\Bigg( \begin{array}{ccc}
	\hat{{\theta}}^d-{\theta}^*\\
	\lambda
	\end{array}
	\Bigg )}
=& \ \underbrace{\Bigg( {
	\Bigg( \begin{array}{ccc}
	\nabla^2 {\cL(\hat{{\theta})}} & \bm{1}\\
	\bm{1}^\top & 0
	\end{array} 
	\Bigg )}^{-1}-{
	\Bigg( \begin{array}{ccc}
	\nabla^2 {\cL({{\theta}}^*)} & \bm{1}\\
	\bm{1}^\top & 0
	\end{array} 
	\Bigg )}^{-1}\Bigg) 
{
	\Bigg( \begin{array}{ccc}
	-\nabla \cL(\hat{{\theta}})+ \nabla^2 \cL(\hat{{\theta}})(\hat{{\theta}}-{\theta}^*)  \\
	0
	\end{array} 
	\Bigg )}}_{I_1}\\
& + \underbrace{
	\Bigg( \begin{array}{ccc}
	\nabla^2 {\cL({{\theta}^*})} & \bm{1}\\
	\bm{1}^\top & 0
	\end{array} 
	\Bigg )^{-1}
{
	\Bigg( \begin{array}{ccc}
	\nabla \cL({{\theta}}^*)-\nabla \cL(\hat{{\theta}})+ \nabla^2 \cL(\hat{{\theta}})(\hat{{\theta}}-{\theta}^*)  \\
	0
	\end{array} 
	\Bigg )}}_{I_2}\\
& +{
	\Bigg( \begin{array}{ccc}
	\nabla^2 {\cL({{\theta}}^*)} & \bm{1}\\
	\bm{1}^\top & 0
	\end{array} 
	\Bigg )}^{-1}
{
	\Bigg( \begin{array}{ccc}
	-\nabla \cL({{\theta}}^*) \\
	0
	\end{array} 
	\Bigg )}.
\end{aligned}
\end{equation}
\end{small}

Note that we have two parts of randomness here. The first part comes from the random comparison graph~$\cE$, and the other part comes from the binary pairwise comparison outcome $y_{i,j}^{(\ell)}$ in \eqref{eqn:y}. In what follows, to obtain the asymptotic normality of the debiasing estimator, we first derive some bounds for the objective which depend on the randomness of $\cE$, and then derive the conditional asymptotic distribution of $\hat{{\theta}}^d-{\theta}^*$ by conditioning on the comparison graph $\cE$ such that we only have randomness from the binary outcome $y_{i,j}^{(\ell)}$.

In particular, by the randomness of the graph $\cE$, the next lemma derives some inequalities, which are essential for deriving the bounds for $I_1$ and $I_2$. We provide the proof in Section~\ref{sec:pf:cond}.

\begin{lemma}\label{lem:cond}
 Under the conditions of Theorem \ref{thm:asy}, we have that, with probability at least $1- \cO(n^{-5})$,	
 
	\begin{equation}	\label{eqn:gradient}
	\|{\nabla \cL({{\theta}}^*)}\|_\infty
	\lesssim np \sqrt{\frac{\log n}{L}},
	\end{equation}
	
	\begin{equation}\label{eqn:smoothness}
	\|\nabla \cL(\hat{{\theta}})-\nabla \cL({{\theta}}^*)
	- \nabla^2 \cL({\theta}^*)(\hat{{\theta}}-{\theta}^*)\|_\infty	\lesssim \frac{\log n}{L},
	\end{equation}
	
	\begin{equation}	\label{eqn:Hessian_concentartion}
	\| \nabla^2 {\cL(\hat{{\theta})}}-\nabla^2 {\cL({{\theta}}^*)}\|_\infty\lesssim np\sqrt{\frac{\log n}{npL}},
	\end{equation}		
	
	\begin{equation}\label{eqn:eigenvalue_inv_Hessian}
	\begin{aligned}
	\left\|{
		\left( \begin{array}{ccc}
		\nabla^2 {\cL({{{\theta}})}} & \bm{1}\\
		\bm{1}^\top & 0
		\end{array} 
		\right )^{-1}}\right\|_2
	=\frac{1}{\lambda_{n-1}\left(\nabla^{2} \cL({{{\theta}}})\right)}
	\lesssim \frac{1}{np},
	\end{aligned}
	\end{equation}
	
	\begin{equation}\label{eqn:qhat-qstar}
	\left\| 
	{\left( \begin{array}{ccc}
		\nabla^2 {\cL(\hat{{\theta})}} & \bm{1}\\
		\bm{1}^\top & 0
		\end{array} 
		\right )}^{-1}-{
		\left( \begin{array}{ccc}
		\nabla^2 {\cL({{\theta}}^*)} & \bm{1}\\
		\bm{1}^\top & 0
		\end{array} 
		\right )}^{-1} \right\|_2
	\lesssim \frac{1}{np}\sqrt{\frac{\log n}{pL}}. 
	\end{equation}
\end{lemma}

The next lemma derives the upper bounds of the terms $I_1$ and $I_2$ conditioning on the graph $\cE$ based on the above lemma. We provide the proof in Section~\ref{sec:pf:I1}.

\begin{lemma}
	\label{lem:I1}
	Under the BTL model, suppose that the conditions of Theorem \ref{thm:asy} hold, and suppose that conditioning on the random graph $\cE$,
	\eqref{eqn:gradient}, \eqref{eqn:smoothness},  \eqref{eqn:Hessian_concentartion}, \eqref{eqn:eigenvalue_inv_Hessian}, and \eqref{eqn:qhat-qstar} hold. We  have that  $$\|I_1\|_\infty \lesssim \frac{\sqrt{n}\log n}{\sqrt{p}\ L},\quad \text{ and }\quad\|I_2\|_\infty \lesssim \frac{\log n}{\sqrt{n}pL}.$$ 
	\end{lemma}
	
	Meanwhile, by the fact that
	\begin{equation}\nonumber
	{
		\left(  \begin{array}{ccc}
		{\Theta}^*_{11} & \frac{1}{n}\bm{1}\\
		\frac{1}{n}\bm{1}^\top & 0
		\end{array} 
		\right) }=
	{	\left(  \begin{array}{ccc}
		\nabla^2 {\cL({{\theta}^*)}} & \bm{1}\\
		\bm{1}^\top & 0
		\end{array} 
		\right) ^{-1}},
	\end{equation}
	we immediately have
	\begin{equation}\label{eqn:op1}
	\hat{{\theta}}^d-{\theta}^*= - {\Theta}^*_{11} 
	\nabla {\cL({{\theta}^*)}}  + r,
	\end{equation}
	where $\|r\|_\infty \leq \|I_1\|_\infty + \|I_2\|_\infty \lesssim\frac{\sqrt{n}\log n}{\sqrt{p}\ L}+\frac{\log n}{\sqrt{n}pL}$. By our assumption that  $\frac{n\log n}{\sqrt{L}} + \frac{\log n}{\sqrt{p L}} = o(1)$, we have $\sqrt{npL}\|r\|_\infty =o(1)$.

	To obtain the conditional asymptotic distribution of $\hat{{\theta}}^d-{\theta}^*$, we have the following lemma about the asymptotic distribution of ${\Theta}^*_{11} 
	\nabla {\cL({{\theta}^*)}}$ conditioning on graph $\cE$, and  the proof is provided in Section~\ref{sec:pf:CLT}. 
	
	\begin{lemma}[Central Limit Theorem]
		\label{lem:CLT}
	Under BTL model, assume the conditions of Theorem \ref{thm:asy} are satisfied. Suppose that the comparison graph $\mathcal{E}$  satisfies conditions \eqref{eqn:gradient},  \eqref{eqn:smoothness},  \eqref{eqn:Hessian_concentartion}, \eqref{eqn:eigenvalue_inv_Hessian}, \eqref{eqn:qhat-qstar}. We have 	
	\begin{equation}\nonumber 
	\sqrt{L} \ \frac{[{\Theta_{11}^*}]_j \nabla \cL (\boldsymbol{\theta}^{*} )}
	{\sqrt{[{\Theta}^*_{11}]_{jj}}}     \Biggiven \mathcal{E}
	\xrightarrow{d} N(0,1), 
	\end{equation}
	and 
	\begin{equation}\nonumber
	\sqrt{L}\frac{([{\Theta}^*_{11}]_i-[{\Theta}^*_{11}]_j) \nabla \cL (\boldsymbol{\theta}^{*} )}{\sqrt{(\bm{e}_i-\bm{e}_j)^\top {\Theta}^*_{11} (\bm{e}_i-\bm{e}_j)}}  \   \Biggiven \mathcal{E}
	\xrightarrow{d}  N(0,1),
	\end{equation}	
	where $[{\Theta_{11}^*}]_j$ is the $j$-th row of matrix ${\Theta_{11}^*}$, and $[{\Theta_{11}^*}]_{jj}$ is the $j$-th diagonal element of ${\Theta_{11}^*}$.
\end{lemma}

	Combining \eqref{eqn:op1} with the above central limit theorem, 
	 conditioning on comparison graph  $\mathcal{E}$, we have the following asymptotic distributions that
$$\sqrt{L}\ \dfrac{\hat{\theta}^d_j-{\theta}_j^*}{\sqrt{[{\Theta}^*_{11}]_{jj}}}\Biggiven \mathcal{E} \xrightarrow{d} N(0,1), \quad\text{and}\quad
\sqrt{L}\frac{{\hat{\theta}^d}_h-{\hat{\theta}^d}_j-({{\theta}}^*_h-{{\theta}}^*_j)}{\sqrt{ (\bm{e}_h-\bm{e}_j)^\top ({\Theta}^*_{11}) (\bm{e}_h-\bm{e}_j)}} \Biggiven \mathcal{E} \xrightarrow{d} N(0,1).
$$
Meanwhile, by Lemma~\ref{lem:cond}, we have that the conditions  \eqref{eqn:gradient},  \eqref{eqn:smoothness},  \eqref{eqn:Hessian_concentartion}, \eqref{eqn:eigenvalue_inv_Hessian}, \eqref{eqn:qhat-qstar} are satisfied with probability at least $1-\cO(n^{-5})$, and our claim holds as desired.
\end{proof}

\section{Proofs of Corollaries in Section~\ref{sec:311}}

\subsection{Proof of Corollary \ref{cor:gen-asy:1}} \label{sec:pf:gen-asy:1}
\begin{proof}
To prove $ \big\|\widehat{\Theta}-\Theta^*\big\|_{\infty}\lesssim r_1+r_4$,
recall that $\hat \Theta = \Big( \begin{smallmatrix} 
\hat \Theta_{11} & \hat \Theta_{12}\\
\hat \Theta_{12}^\top & \hat \Theta_{22} \end{smallmatrix} \Big)$, 
where 
\begin{equation}\nonumber
\hat \Theta_{11} 
= \hat \Omega - \hat \Omega \nabla {f(\hat{{\theta}})} \big(\nabla {f(\hat{{\theta}})}^\top \hat \Omega \nabla {f(\hat{{\theta}})}\big)^{-1} \nabla {f(\hat{{\theta}})}^\top \hat \Omega,
\end{equation}
\begin{equation}\nonumber
\hat \Theta_{12} 
=\hat \Omega \nabla {f(\hat{{\theta}})}
\big(\nabla {f(\hat{{\theta}})}^\top \hat \Omega \nabla {f(\hat{{\theta}})}\big)^{-1}, \ \text{ and }\ 
\hat \Theta_{22} = 
- \big(\nabla {f(\hat{{\theta}})}^\top \hat \Omega \nabla {f(\hat{{\theta}})}\big)^{-1}.
\end{equation}
Also recall that $r_1, r_2, r_3, r_4, r_5 = o(1)$, and $L_1, L_2, c_1, c_2, c_3$ are constants defined in Assumptions \ref{ass:consistency} -- \ref{ass:omega}.
Before going further, we first have the following two inequalities what are used in bounding $\widehat{\Theta}-\Theta^*$,
	\begin{equation}\nonumber
	\begin{aligned} 
	\big\|\nabla {f(\hat{{\theta}})}^\top \hat \Omega
	-\nabla {f({{\theta}}^*)}^\top  \Omega^* \big\|_\infty 
	= &\ 
	\big\|\nabla {f(\hat{{\theta}})}^\top (\hat \Omega-\Omega^*)
	+(\nabla {f(\hat{{\theta}})}^\top-\nabla {f({{\theta}}^*)}^\top )
	\Omega^* 
	\big\|_\infty  
	\\
	\lesssim&\ (c_1+L_2 r_1) r_4+L_2 r_1 c_2
	\lesssim r_1+r_4,
	\end{aligned} 	
	\end{equation}
	and
	\begin{equation}\nonumber
	\begin{aligned}
	&\big|\nabla {f(\hat{{\theta}})}^\top \hat \Omega \nabla {f(\hat{{\theta}})}
	-
	\nabla {f({{\theta}}^*)}^\top  \Omega^* \nabla {f({{\theta}}^*)}\big| \\
	&\quad =  
	\big|\big(\nabla {f(\hat{{\theta}})}^\top \hat \Omega
	-\nabla {f({{\theta}}^*)}^\top  \Omega^* \big)\nabla {f(\hat{{\theta}})}
	+\nabla {f({{\theta}}^*)}^\top  \Omega^* (\nabla f(\hat{{\theta}})-\nabla f({{\theta}}^*))
	\big| 
	\\
	& \quad \leq 
	\big\|\nabla {f(\hat{{\theta}})}^\top \hat \Omega
	-\nabla {f({{\theta}}^*)}^\top  \Omega^* \big\|_\infty
	\big\|
	\nabla {f(\hat{{\theta}})}
	\big\|_\infty
	+
	\big\| \nabla {f({{\theta}}^*)} \big\|_\infty
	\big\| \Omega^* \big\|_\infty
	\big\|\nabla 	 f(\hat{{\theta}})-\nabla f({{\theta}}^*)
	\big\|_\infty
	\\
	&\quad \lesssim (r_1 + r_4)(L_2 r_1 + c_1)+c_1 c_2 L_2 r_1
	\lesssim r_1 + r_4.
	\end{aligned} 	
	\end{equation}
	Then, we bound each block of $\widehat{\Theta}-\Theta^*$ separately. We first have 
		\begin{equation}\label{eqn:22}
		\begin{aligned}
		\big|\widehat{\Theta}_{22}-\Theta^*_{22}\big|
		=&\ \big|\big(\nabla {f(\hat{{\theta}})}^\top \hat \Omega \nabla {f(\hat{{\theta}})}\big)^{-1}
		-
		\big(\nabla {f({{\theta}}^*)}^\top  \Omega^* \nabla {f({{\theta}}^*)}\big)^{-1}\big| \\
		\leq&\ 	\big|\big(\nabla {f(\hat{{\theta}})}^\top \hat \Omega \nabla {f(\hat{{\theta}})}\big)^{-1}
		\big|
		\big|\nabla {f(\hat{{\theta}})}^\top \hat \Omega \nabla {f(\hat{{\theta}})} \\
		& \ -
		\nabla {f({{\theta}}^*)}^\top  \Omega^* \nabla {f({{\theta}}^*)}\big|	
		\big|
		\big(\nabla {f({{\theta}}^*)}^\top  \Omega^* \nabla {f({{\theta}}^*)}\big)^{-1}\big|\\
		\lesssim&\  \frac{r_1 + r_4}{(c_3-(r_1 + r_4))c_3}
		\lesssim r_1 + r_4,
		\end{aligned} 	
		\end{equation}
\noindent	where the second inequality holds since $$\big|\nabla {f(\hat{{\theta}})}^\top \hat \Omega \nabla {f(\hat{{\theta}})}
	\big| 
	\geq 
	 \big|\nabla {f({{\theta}}^*)}^\top  \Omega^* \nabla {f({{\theta}}^*)}\big|
	 -
	 \big|\nabla {f(\hat{{\theta}})}^\top \hat \Omega \nabla {f(\hat{{\theta}})}
	 -
	 \nabla {f({{\theta}}^*)}^\top  \Omega^* \nabla {f({{\theta}}^*)}\big|
	\gtrsim c_3-(r_1+r_4)
	,$$ and the last inequality holds since $r_1 ,r_4 = o(1)$ and  $c_3 =\cO(1)$ by assumptions. Next, we have
		\begin{equation}\label{eqn:12}
		\begin{aligned}
		\big\|\widehat{\Theta}_{12}-\Theta^*_{12}\big\|_\infty
		=&\ \big\|\hat \Omega \nabla {f(\hat{{\theta}})} \big(\nabla {f(\hat{{\theta}})}^\top \hat \Omega \nabla {f(\hat{{\theta}})}\big)^{-1}
		-
		\Omega^* \nabla {f({{\theta}}^*)}
		\big(\nabla {f({{\theta}}^*)}^\top  \Omega^* \nabla {f({{\theta}}^*)}\big)^{-1}\|_\infty \\
		\leq &\ 	
		\big\|\hat \Omega \nabla {f(\hat{{\theta}})}  \|_\infty 
		\big|\big(\nabla {f(\hat{{\theta}})}^\top \hat \Omega \nabla {f(\hat{{\theta}})}\big)^{-1}
		-
		\big(\nabla {f({{\theta}}^*)}^\top  \Omega^* \nabla {f({{\theta}}^*)}\big)^{-1}\big| \\
		& +
		\big\|\hat \Omega \nabla {f(\hat{{\theta}})}
		- \Omega^* \nabla {f({{\theta}}^*)} \|_\infty	
		\big|
		\big(\nabla {f({{\theta}}^*)}^\top  \Omega^* \nabla {f({{\theta}}^*)}\big)^{-1}\big|\\
		\lesssim&\  (r_1 + r_4 +c_2c_1) (r_1 + r_4)
		+\frac{(r_1 + r_4)}{c_3}
		\lesssim r_1 + r_4,
		\end{aligned} 	
		\end{equation}
	where the last ineuality holds since $	
	\big\|\hat \Omega \nabla {f(\hat{{\theta}})}  \|_\infty 
	\leq \big\|\nabla {f({{\theta}}^*)}^\top  \Omega^* \big\|_\infty
	+ \big\|\nabla {f(\hat{{\theta}})}^\top \hat \Omega
	-\nabla {f({{\theta}}^*)}^\top  \Omega^* \big\|_\infty 
	\lesssim c_1c_2 + r_4 + r_1.
	$ 
	Similarly, we have
		\begin{equation}\label{eqn:11}
		\begin{aligned}
\big\|\widehat{\Theta}_{11}-\Theta^*_{11}\big\|_\infty 
		=&\ \big\|\hat \Omega - \hat \Omega \nabla {f(\hat{{\theta}})} \big(\nabla {f(\hat{{\theta}})}^\top \hat \Omega \nabla {f(\hat{{\theta}})}\big)^{-1} \nabla {f(\hat{{\theta}})}^\top \hat \Omega
		-
		\Omega^* \\
		& +
		\Omega^* \nabla {f({{\theta}}^*)} 
		\big(\nabla {f({{\theta}}^*)}^\top  \Omega^* \nabla {f({{\theta}}^*)}\big)^{-1}
		\nabla {f({{\theta}}^*)}^\top \Omega^*  \|_\infty \\
		\leq&\ 	\big\|\hat \Omega-\Omega^*  \|_\infty 
		+ \big\|\widehat{\Theta}_{12}-\Theta^*_{12}\big\|_\infty
		\big\| \nabla {f(\hat{{\theta}})}^\top \hat \Omega \big\|_\infty \\
		& + 
		\big\| \Omega^* \nabla {f({{\theta}}^*)} 
		\big(\nabla {f({{\theta}}^*)}^\top  \Omega^* \nabla {f({{\theta}}^*)}\big)^{-1} \big\|_\infty
		\big\| \nabla {f(\hat{{\theta}})}^\top \hat \Omega
		-\nabla {f({{\theta}}^*)}^\top  \Omega^* \big\|_\infty
		\\
		\lesssim &\  r_4
		+
		(r_1 + r_4) 
		(r_1 + r_4 +  c_1 c_2)
		+ (r_1 + r_4)(r_1 + r_4)
		\lesssim r_1 + r_4.
		\end{aligned} 	
		\end{equation}
	Combining \eqref{eqn:22}, \eqref{eqn:12}, \eqref{eqn:11} together, we have that $ \big\|\widehat{\Theta}-\Theta^*\big\|_{\infty}\lesssim r_1+r_4$, which completes the proof.
\end{proof}

\subsection{Proof of Corollary \ref{cor:gen-asy:2}}
\label{sec:pf:gen-asy:2}
\begin{proof}
To prove $\|I-\widehat{\Theta}\widehat{\Sigma} \|_{\infty}
	\lesssim r_3$, recall that $\hat{\Sigma}=\Big( \begin{smallmatrix} \nabla^2 {\cL(\hat{{\theta}})} 
	& \nabla {f(\hat{{\theta}})}\\
	\nabla {f(\hat{{\theta}})}^\top 
	& 0 \end{smallmatrix} \Big)$,   $r_1, r_2, r_3, r_4, r_5 = o(1)$, and $L_1, L_2, c_1, c_2, c_3$ are constants defined in Assumptions \ref{ass:consistency} -- \ref{ass:omega}. 
We bound the four blocks in $I-\widehat{\Theta}\widehat{\Sigma}$ separately.	First, we have
		\begin{equation}\nonumber
		\begin{aligned}
		\big\|(I-\widehat{\Theta}\widehat{\Sigma})_{11} \big\|_{\infty}
		=&\ \Big\| \big(I-\hat \Omega \nabla {f(\hat{{\theta}})} \big(\nabla {f(\hat{{\theta}})}^\top \hat \Omega \nabla {f(\hat{{\theta}})}\big)^{-1} \nabla {f(\hat{{\theta}})}^\top 
		\big)
		\big(\hat \Omega \nabla^2 {\cL(\hat{{\theta}})} -I
		\big) \Big\|_\infty \\
		\leq&\ \Big\|I-\hat \Omega \nabla {f(\hat{{\theta}})} \big(\nabla {f(\hat{{\theta}})}^\top \hat \Omega \nabla {f(\hat{{\theta}})}\big)^{-1} \nabla {f(\hat{{\theta}})}^\top 
		\Big\|_\infty
		\big\|\hat \Omega \nabla^2 {\cL(\hat{{\theta}})} -I
		\big\|_\infty,
		\end{aligned}
		\end{equation}
where 	
	\begin{equation}\nonumber
	\begin{aligned}
	& \Big\|I-\hat \Omega \nabla {f(\hat{{\theta}})} \big(\nabla {f(\hat{{\theta}})}^\top \hat \Omega \nabla {f(\hat{{\theta}})}\big)^{-1} \nabla {f(\hat{{\theta}})}^\top 
	\Big\|_\infty \\
	& \quad \leq
	\Big\|I-\Omega^* \nabla {f({{\theta}}^*)} 
	\big(\nabla {f({{\theta}}^*)}^\top  \Omega^* \nabla {f({{\theta}}^*)}\big)^{-1}
	\nabla {f({{\theta}}^*)}^\top
	\Big\|_\infty \\
	&\qquad+
	\Big\|\Omega^* \nabla {f({{\theta}}^*)} 
	\big(\nabla {f({{\theta}}^*)}^\top  \Omega^* \nabla {f({{\theta}}^*)}\big)^{-1}
	\nabla {f({{\theta}}^*)}^\top
	-\hat \Omega \nabla {f(\hat{{\theta}})} \big(\nabla {f(\hat{{\theta}})}^\top \hat \Omega \nabla {f(\hat{{\theta}})}\big)^{-1} \nabla {f(\hat{{\theta}})}^\top 
	\Big\|_\infty 
 \\
	& \quad\leq
	\Big\|I-\Omega^* \nabla {f({{\theta}}^*)} 
	\big(\nabla {f({{\theta}}^*)}^\top  \Omega^* \nabla {f({{\theta}}^*)}\big)^{-1}
	\nabla {f({{\theta}}^*)}^\top
	\Big\|_\infty
	+
	\big\|\widehat{\Theta}_{12}-\Theta^*_{12}\big\|_\infty
	\big\|
	\nabla {f(\hat{{\theta}})}
	\big\|_\infty \\
	&\qquad+
	\Big\|\Omega^* \nabla {f({{\theta}}^*)} 
	\big(\nabla {f({{\theta}}^*)}^\top  \Omega^* \nabla {f({{\theta}}^*)}\big)^{-1}
	\Big\|_\infty 
	\big\|
	\nabla {f(\hat{{\theta}})}
	- \nabla {f({{\theta}^*})}
	\big\|_\infty
	 \\
	&\quad\lesssim
	(1+c_2 c_1^2 /c_3)
	+
	(r_1+r_4)(c_1+L_2r_1)
	+
	(c_2 c_1 /c_3)(L_2r_1)
	\lesssim 1,
	\end{aligned}
	\end{equation}
where we used the fact that $$\widehat{\Theta}_{12}-\Theta^*_{12}
= \hat \Omega \nabla {f(\hat{{\theta}})} \big(\nabla {f(\hat{{\theta}})}^\top \hat \Omega \nabla {f(\hat{{\theta}})}\big)^{-1}
-
\Omega^* \nabla {f({{\theta}}^*)}
\big(\nabla {f({{\theta}}^*)}^\top  \Omega^* \nabla {f({{\theta}}^*)}\big)^{-1}.$$
Thus,
	\begin{align*}\nonumber
	\big\|(I-\widehat{\Theta}\widehat{\Sigma})_{11} \big\|_{\infty}
	\lesssim r_3.
	\end{align*}
Similarly, we have 
	\begin{equation}\nonumber
	\begin{aligned}
	\|(I-\widehat{\Theta}\widehat{\Sigma})_{21} \|_{\infty}
	=&\ \Big\|\big(\nabla {f(\hat{{\theta}})}^\top \hat \Omega \nabla {f(\hat{{\theta}})}\big)^{-1} \nabla {f(\hat{{\theta}})}^\top
	\big(\hat \Omega \nabla^2 {\cL(\hat{{\theta}})}^\top -I\big)
	\Big\|_\infty \\
	\leq & \ \Big\|\big(\nabla {f(\hat{{\theta}})}^\top \hat \Omega \nabla {f(\hat{{\theta}})}\big)^{-1} \nabla {f(\hat{{\theta}})}^\top \Big\|_\infty
	\Big\| \hat \Omega \nabla^2 {\cL(\hat{{\theta}})}^\top -I
	\Big\|_\infty \\
	\leq &\ 
	\Big\{
	\Big\|\big(\nabla {f({{\theta}}^*)}^\top  \Omega^* \nabla {f({{\theta}}^*)}\big)^{-1}
	\nabla {f(\hat{{\theta}})}^\top \Big\|_\infty
	+
	\big|\widehat{\Theta}_{22}-\Theta^*_{22}\big|
	\big\|
	\nabla {f(\hat{{\theta}})}
	\big\|_\infty
	\Big\}
	\big\|\hat \Omega \nabla^2 {\cL(\hat{{\theta}})}^\top -I
	\big\|_\infty \\
	\lesssim & \
	\Big\{(c_1+L_2 r_1)/c_3+(r_1+r_4)(c_1+L_2r_1)\Big\}r_3
	\lesssim r_3.
	\end{aligned}
	\end{equation}
Combining the above two inequalities with the fact that 	
$(I-\widehat{\Theta}\widehat{\Sigma})_{22} =0$ 
by simple algebra, we conclude that $\|I-\widehat{\Theta}\widehat{\Sigma} \|_{\infty}
\lesssim r_3.$
\end{proof}

\section{Proofs of Lemmas in Section~\ref{sec:pf-asy}}\label{sec:pf:lem}


\subsection{Proofs of Auxiliary Lemmas in Section \ref{sec:pf:gen-asy}}
\label{sec:pf:lem:gen-asy}


%
%
\subsubsection{Proof of Lemma \ref{lem:J1}}
\label{sec:pf-J1}

\begin{proof}
	We bound $J_1$, $J_2$, and $J_3$ separately. Recall that $r_1, r_2, r_3, r_4, r_5 = o(1)$, and $L_1, L_2, c_1, c_2, c_3$ are constants defined in Assumptions \ref{ass:consistency} -- \ref{ass:omega}. 
To bound $J_2$, by Assumptions \ref{ass:consistency} and \ref{ass:loss}, we~have	
\begin{equation}\nonumber
\begin{aligned}
& 	\|\nabla \cL({{\theta}}^*)-\nabla \cL(\hat{{\theta}})+ \nabla^2 \cL(\hat{{\theta}})(\hat{{\theta}}-{\theta}^*)\|_\infty \\
& \quad \leq 	\|\nabla \cL(\hat{{\theta}})-\nabla \cL({{\theta}}^*)
- \nabla^2 \cL({{\theta}}^*)(\hat{{\theta}}-{\theta}^*)\|_\infty
+\|\nabla^2\cL(\hat{{\theta}})-\nabla^2\cL({{\theta}}^*)\|_\infty
\|\hat{{\theta}}-{\theta}^*\|_\infty\\
&\quad  \lesssim  \Big\|\int_{0}^{1} \Big(\nabla^2 \cL\big({\theta}^*+t(\hat{{\theta}}-{\theta}^*)\big)-\nabla^2 \cL({{\theta}}^*)\Big) \mathrm{d} t \  (\hat{{\theta}}-{\theta}^*) \Big\|_\infty
+
\|\nabla^2\cL(\hat{{\theta}})-\nabla^2\cL({{\theta}}^*)\|_\infty \|\hat{{\theta}}-{\theta}^*\|_\infty\\
& \quad \lesssim 			  		
L_1 \|\hat{{\theta}}-{\theta}^*\|_\infty^2
\lesssim L_1 r_1^2.
\end{aligned}
\end{equation}
Similarly, one can derive the same rate for 
$ \|f({{\theta}}^*)-f(\hat{{\theta}})
+ \nabla f(\hat{{\theta}})^\top(\hat{{\theta}}-{\theta}^*)\|_\infty \lesssim L_2 r_1^2$, and we thus have a bound for $J_2$ 	that
\begin{equation}\nonumber
\begin{aligned}	
\|J_2\|_\infty=\|{\Theta}^*\|_\infty
\Bigg\|{
	\left( \begin{array}{ccc}
	\nabla \cL({{\theta}}^*)-\nabla \cL(\hat{{\theta}})+ \nabla^2 \cL(\hat{{\theta}})(\hat{{\theta}}-{\theta}^*)  \\
	f({{\theta}}^*)-f(\hat{{\theta}})
	+ \nabla f(\hat{{\theta}})^\top(\hat{{\theta}}-{\theta}^*)
	\end{array} 
	\right )} \Bigg\|_\infty
\lesssim  r_1^2 r_5.
\end{aligned}
\end{equation}
Then we bound $J_1$. We first have, by Assumptions \ref{ass:consistency}
 and \ref{ass:loss}, 
\begin{equation}\nonumber
\begin{aligned}
\|-\nabla \cL(\hat{{\theta}})+ \nabla^2 \cL(\hat{{\theta}})(\hat{{\theta}}-{\theta}^*) \|_\infty
\lesssim  \|\nabla \cL({{\theta}}^*)-\nabla \cL(\hat{{\theta}})+ \nabla^2 \cL(\hat{{\theta}})(\hat{{\theta}}-{\theta}^*) \|_\infty+\|\nabla \cL({{\theta}}^*)\|_\infty
\lesssim L_1 r_1^2 + r_2.
\end{aligned}
\end{equation}
Similarly, we have, by Assumptions \ref{ass:consistency} and \ref{ass:constraint},
$$\|-f(\hat{{\theta}})
+ \nabla f(\hat{{\theta}})^\top(\hat{{\theta}}-{\theta}^*)\|_\infty=\|f({{\theta}}^*)-f(\hat{{\theta}})
+ \nabla f(\hat{{\theta}})^\top(\hat{{\theta}}-{\theta}^*)\|_\infty
\lesssim L_2 r_1^2.
$$
Then, we have that
\begin{equation}\nonumber
\begin{aligned}
\|J_1\|_\infty
\lesssim\|\widehat{\Theta}-{\Theta}^*\|_\infty 
\Bigg\|{
	\left( \begin{array}{ccc}
	-\nabla \cL(\hat{{\theta}})+ \nabla^2 \cL(\hat{{\theta}})(\hat{{\theta}}-{\theta}^*)  \\
	-f(\hat{{\theta}})
	+ \nabla f(\hat{{\theta}})^\top(\hat{{\theta}}-{\theta}^*)
	\end{array} 
	\right )} \Bigg\|_\infty
\lesssim (r_1 + r_4)(r_1^2 + r_2),
\end{aligned}
\end{equation}
where we use the fact that  $\big\|\widehat{\Theta}-\Theta^*\big\|_{\infty}\lesssim r_1 + r_4$ from Corollary \ref{cor:gen-asy:1}. 
Finally, we  bound $J_3$ by  Assumption \ref{ass:consistency} and Corollary \ref{cor:gen-asy:2} that
\begin{equation}\nonumber
\begin{aligned}
\|J_3\|_\infty
=\|I-\widehat{\Theta}\widehat{\Sigma}\|_\infty
\|\hat{{\theta}}-{{\theta}}^*\|_\infty
\lesssim r_1r_3.
\end{aligned}
\end{equation}
Combining the three bounds, we conclude the proof.
\end{proof}
\subsection{Proofs of Auxiliary Lemmas in Section \ref{sec:pf-asy-BTL}}
\label{sec:auciliary lem - 2}

\subsubsection{Proof of Lemma \ref{lem:cond}} \label{sec:pf:cond}
\begin{proof}
		In Lemmas \ref{lem:gradient}, \ref{lem:smoothness},  \ref{lem:Hessian_concentartion}, \ref{lem:eigenvalue_inv_Hessian}, \ref{lem:qhat-qstar}, we show that conditions \eqref{eqn:gradient}, \eqref{eqn:smoothness},  \eqref{eqn:Hessian_concentartion}, \eqref{eqn:eigenvalue_inv_Hessian}, \eqref{eqn:qhat-qstar} are satisfied with probability at least $1- \cO(n^{-5})$. The proofs of these five lemmas can be found in Section~\ref{sec:auciliary lem - 3}. 
\end{proof}
\begin{lemma}[Gradient concentration]
	\label{lem:gradient}
	Under the conditions of Theorem \ref{thm:asy}, we have
	\begin{equation}\nonumber
	\|{\nabla \cL({{\theta}}^*)}\|_\infty
	\lesssim np \sqrt{\frac{\log n}{L}}
	\end{equation} 
	with probability at least $1-\cO(n^{-6}).$	
\end{lemma}
\begin{proof}
See Section~\ref{sec:pf:gradient} for the detailed proof.
\end{proof}

\begin{lemma}[Local smoothness condition] \label{lem:smoothness}
	Under the conditions of Theorem \ref{thm:asy}, we have
	\begin{equation}\nonumber
	\|\nabla \cL(\hat{{\theta}})-\nabla \cL({{\theta}}^*)
	- \nabla^2 \cL({\theta}^*)(\hat{{\theta}}-{\theta}^*)\|_\infty	\lesssim \frac{\log n}{L}
	\end{equation}
	with probability at least $1-\cO(n^{-5})$.
\end{lemma}
\begin{proof}
See Section~\ref{sec:pf:smoothness} for the detailed proof.
\end{proof}

\begin{lemma}[Hessian concentration]\label{lem:Hessian_concentartion}
	Under the conditions of Theorem \ref{thm:asy}, we have
	\begin{equation}	\nonumber	\| \nabla^2 {\cL(\hat{{\theta})}}-\nabla^2 {\cL({{\theta}}^*)}\|_\infty\lesssim np\sqrt{\frac{\log n}{npL}}
	\end{equation}			
	with probability at least $1-\cO(n^{-5})$.
\end{lemma}
\begin{proof}
See Section~\ref{sec:pf:Hessian_concentartion} for the detailed proof.
\end{proof}

\begin{lemma}\label{lem:eigenvalue_inv_Hessian}
	Under the conditions of Theorem \ref{thm:asy}, we have
	\begin{equation}\nonumber
	\begin{aligned}
	\left\|{
		\left( \begin{array}{ccc}
		\nabla^2 {\cL({{{\theta}})}} & \bm{1}\\
		\bm{1}^\top & 0
		\end{array} 
		\right )^{-1}}\right\|_2
	=\frac{1}{\lambda_{n-1}\left(\nabla^{2} \cL({{{\theta}}})\right)}
	\lesssim \frac{1}{np}
	\end{aligned}
	\end{equation}
	with probability exceeding $1-\cO(n^{-10})$, where ${\theta}$ can be either $\hat{{\theta}}$ or ${\theta}^*$.
\end{lemma}
\begin{proof}
See Section~\ref{sec:pf:eigenvalue_inv_Hessian} for the detailed proof.
\end{proof}

\begin{lemma}\label{lem:qhat-qstar}
	Under the conditions of Theorem \ref{thm:asy}, we have	\begin{equation}\nonumber
	\left\| 
	{\left( \begin{array}{ccc}
		\nabla^2 {\cL(\hat{{\theta})}} & \bm{1}\\
		\bm{1}^\top & 0
		\end{array} 
		\right )}^{-1}-{
		\left( \begin{array}{ccc}
		\nabla^2 {\cL({{\theta}}^*)} & \bm{1}\\
		\bm{1}^\top & 0
		\end{array} 
		\right )}^{-1} \right\|_2
	\lesssim \frac{1}{np}\sqrt{\frac{\log n}{pL}} 
	\end{equation}
	with probability at least $1-\cO(n^{-5})$.
\end{lemma}
\begin{proof}
See Section~\ref{sec:pf:qhat-qstar} for the detailed proof.
\end{proof}

\subsubsection{Proof of Lemma \ref{lem:I1}}\label{sec:pf:I1}
\begin{proof}
We  bound $I_1$ and $I_2$ in \eqref{equ:decomposition} based on \eqref{eqn:gradient}, \eqref{eqn:smoothness},  \eqref{eqn:Hessian_concentartion}, \eqref{eqn:eigenvalue_inv_Hessian}, and \eqref{eqn:qhat-qstar}.
First, we have that 
\begin{small}
	\begin{equation}\label{eqn:bound I1}
	\begin{aligned}
	\|I_1\|_\infty
	\leq &\	\left\| 
	{\left( \begin{array}{ccc}
		\nabla^2 {\cL(\hat{{\theta})}} & \bm{1}\\
		\bm{1}^\top & 0
		\end{array} 
		\right )}^{-1}-{
		\left( \begin{array}{ccc}
		\nabla^2 {\cL({{\theta}}^*)} & \bm{1}\\
		\bm{1}^\top & 0
		\end{array} 
		\right )}^{-1} \right\|_\infty
	\cdot \left\|
	{-\nabla \cL(\hat{{\theta}})+ \nabla^2 \cL({\theta}^*)(\hat{{\theta}}-{\theta}^*)}
	\right\|_\infty \\
	\lesssim&\ \sqrt{n+1} \left\| 
	{\left( \begin{array}{ccc}
		\nabla^2 {\cL(\hat{{\theta})}} & \bm{1}\\
		\bm{1}^\top & 0
		\end{array} 
		\right )}^{-1}-{
		\left( \begin{array}{ccc}
		\nabla^2 {\cL({{\theta}}^*)} & \bm{1}\\
		\bm{1}^\top & 0
		\end{array} 
		\right )}^{-1} \right\|_2  \left\|
	{-\nabla \cL(\hat{{\theta}})+ \nabla^2 \cL({\theta}^*)(\hat{{\theta}}-{\theta}^*)}
	\right\|_\infty \\
	\lesssim&\ \sqrt{n+1} \cdot  \frac{1}{np}\sqrt{\frac{\log n}{pL}} \cdot \left( np\sqrt{\frac{\log n}{L}}+\frac{\log n}{L}\right)
	\lesssim \frac{\sqrt{n}\log n}{\sqrt{p}\ L},
	\end{aligned}
	\end{equation}\end{small}
where the third inequality follows from \eqref{eqn:qhat-qstar} and the fact
$\left\|
{-\nabla \cL(\hat{{\theta}})+ \nabla^2 \cL({\theta}^*)(\hat{{\theta}}-{\theta}^*)}
\right\|_\infty
\leq 
\big\|
{\nabla \cL({{\theta}}^*)}
\big\|_\infty
+
\big\|\nabla \cL(\hat{{\theta}})-\nabla \cL({\theta}^*)- \nabla^2 \cL({\theta}^*)(\hat{{\theta}}-{\theta}^*) \big\|_\infty
\lesssim np\sqrt{\frac{\log n}{L}}+\frac{\log n}{L}$ by \eqref{eqn:gradient} and \eqref{eqn:smoothness}, and the last inequality  follows from the assumption that $p \gtrsim \frac{\log n}{n}$.

Next, we have that for the term $I_2$, by  \eqref{eqn:smoothness},  \eqref{eqn:Hessian_concentartion} and  \eqref{eqn:eigenvalue_inv_Hessian}, 	
\begin{equation} \label{eqn:c7}
\begin{aligned}
\|I_2\|_\infty
&=\Bigg\| {
	\left( \begin{array}{ccc}
	\nabla^2 {\cL({{\theta}^*})} & \bm{1}\\
	\bm{1}^\top & 0
	\end{array} 
	\right )}^{-1}
{
	\left( \begin{array}{ccc}
	\nabla \cL({{\theta}}^*)-\nabla \cL(\hat{{\theta}})+ \nabla^2 \cL(\hat{{\theta}})(\hat{{\theta}}-{\theta}^*)  \\
	0
	\end{array} 
	\right )}
\Bigg\|_\infty\\
&\leq\sqrt{n+1}\left\| {
	\left( \begin{array}{ccc}
	\nabla^2 {\cL({{\theta}}^*)} & \bm{1}\\
	\bm{1}^\top & 0
	\end{array} 
	\right )}^{-1}\right\|_2
\big\|\nabla \cL({{\theta}}^*)-\nabla \cL(\hat{{\theta}})+ \nabla^2 \cL(\hat{{\theta}})(\hat{{\theta}}-{\theta}^*)\big\|_\infty\\
&\lesssim \sqrt{n+1} \times \frac{1}{np} \times \frac{\log n}{L}
\lesssim \frac{\log n}{\sqrt{n}pL},
\end{aligned}
\end{equation}
where the second inequality follows from
$\big\|\nabla \cL({{\theta}}^*)-\nabla \cL(\hat{{\theta}})+ \nabla^2 \cL(\hat{{\theta}})(\hat{{\theta}}-{\theta}^*)\big\|_\infty
\leq 
\big\|\nabla \cL({{\theta}}^*)-\nabla \cL(\hat{{\theta}})+ \nabla^2 \cL({{\theta}^*})(\hat{{\theta}}-{\theta}^*)\big\|_\infty
+ \big\|(\nabla^2 \cL(\hat{{\theta}}) - \nabla^2 \cL({{\theta}^*}))(\hat{{\theta}}-{\theta}^*)\big\|_\infty
\lesssim \frac{\log n}{L}+np\sqrt{\frac{\log n}{npL}} \times \sqrt{\frac{\log n}{npL}} \lesssim \frac{\log n}{L}$ by \eqref{eqn:smoothness},  \eqref{eqn:Hessian_concentartion} and Lemma \ref{lem:consistency}. Combining \eqref{eqn:bound I1} and \eqref{eqn:c7}, we complete the proof.
\end{proof}

\subsubsection{Proof of Lemma \ref{lem:CLT}}\label{sec:pf:CLT}
%
\begin{proof}
	To derive the conditional asymptotic normality, first, by the closed-form of gradient in \eqref{eqn:grad}, we decompose $[{\Theta}^*_{11}]_j \nabla \cL ({\theta}^{*})$ as a summation of independent random variables that 
	\begin{equation}\nonumber
	\begin{aligned}
	\left[\Theta^*_{11}\right]_j\nabla \cL ({\theta}^{*})
	&= \frac{1}{L} [{\Theta}^*_{11}]_j \sum_{\ell=1}^{L} \sum_{k>m, (k,m) \in\mathcal{E}}
	(-y_{m,k}^{(\ell)}+\frac{e^{\theta_k^*}}{e^{\theta_k^*}+e^{\theta_m^*}})(\bm{e}_k-\bm{e}_m) \\
	&= \frac{1}{L} \sum_{\ell=1}^{L} \sum_{k>m, (k,m) \in\mathcal{E}}
	x_{m,k}^{(\ell)},
	\end{aligned}
	\end{equation}
	where $x_{m,k}^{(\ell)} := [{\Theta}^*_{11}]_j
	\big(-y_{m,k}^{(\ell)}+\frac{e^{\theta_k^*}}{e^{\theta_k^*}+e^{\theta_m^*}}\big)(\bm{e}_k-\bm{e}_m)$ that are independent for $m,k,\ell$ given $\mathcal{E}$. 
	We have $\mathbb{E} [x_{m,k}^{(\ell)}]= 0$ and 
	$\mathrm{Var}(x_{m,k}^{(\ell)})= \big([{\Theta}^*_{11}]_j(\bm{e}_k-\bm{e}_m) \big)^2 \frac{e^{\theta_k^*}e^{\theta_m^*}}{(e^{\theta_k^*}+e^{\theta_m^*})^2} $ ,
	which further gives
	\begin{equation}\nonumber
	\begin{aligned}
	\sum_{k>m, (k,m) \in\mathcal{E}}\mathrm{Var}(x_{m,k}^{(\ell)})
	& = \sum_{k>m, (k,m) \in\mathcal{E}}[{\Theta}^*_{11}]_j \frac{e^{\theta_k^*}e^{\theta_m^*}}{(e^{\theta_k^*}+e^{\theta_m^*})^2}(\bm{e}_k-\bm{e}_m)(\bm{e}_k-\bm{e}_m)^\top [\Theta^*_{11}]_j^\top \\
	& = [{\Theta}^*_{11}]_j \nabla^2 {\cL({\theta}^*)} [\Theta^*_{11}]_j^\top
	= [\Theta^*_{11}]_{jj},
	\end{aligned}
	\end{equation} 
	where the second equality follows from the closed-form of Hessian in \eqref{eqn:hess}. In what follows, we prove the last equality.Recall that $
	{
		\left(  \begin{array}{ccc}
		{\Theta}^*_{11} & \frac{1}{n}\bm{1}\\
		\frac{1}{n}\bm{1}^\top & 0
		\end{array} 
		\right) }=
	{	\left(  \begin{array}{ccc}
		\nabla^2 {\cL({{\theta}^*)}} & \bm{1}\\
		\bm{1}^\top & 0
		\end{array} 
		\right) ^{-1}}$.  
	By the proof of Corollary \ref{lem:eigenvalue_Q}, we obtain that ${\Theta}^*_{11} \bm{1}=0$ since 0 is one of the eigenvalues, and $\bm{1}$ is the corresponding eigenvector. Then, we have
	\begin{equation}\nonumber
	\begin{aligned}
	\left(\begin{array}{cc}
	\Theta^*_{11} & \frac{1}{n} \mathbf{1} \\
	\frac{1}{n} \mathbf{1}^\top & 0
	\end{array}\right) 
	= &
	\left(\begin{array}{cc}
	\Theta^*_{11} & \frac{1}{n} \mathbf{1} \\
	\frac{1}{n} \mathbf{1}^\top & 0
	\end{array}\right) 
	\left(\begin{array}{cc}
	\nabla^{2} \cL(\theta^*) & \mathbf{1} \\
	\mathbf{1}^\top & 0
	\end{array}\right)
	\left(\begin{array}{cc}
	\Theta^*_{11} & \frac{1}{n} \bm{1} \\
	\frac{1}{n} \mathbf{1}^\top & 0
	\end{array}\right) \\
	= &
	\left(\begin{array}{cc}
	\Theta^*_{11} \nabla^{2} \cL(\theta^*) \Theta^*_{11} + \frac{1}{n} \bm{1} \bm{1}^\top \Theta^*_{11} + \frac{1}{n} \Theta^*_{11} \bm{1} \bm{1}^\top & * \\
	* & *
	\end{array}\right)
	= 
	\left(\begin{array}{cc}
	\Theta^*_{11} \nabla^{2} \cL(\theta^*) \Theta^*_{11} & * \\
	* & *
	\end{array}\right),
	\end{aligned}
	\end{equation}
	and we immediately have \begin{equation}
	\label{eqn:product}
	{\Theta}^*_{11}  \nabla^2 \cL({\theta}^*)  {\Theta}^*_{11} 
	= {\Theta}^*_{11}.
	\end{equation}
	Also, we have 
	\begin{small}
		\begin{equation}\nonumber
	\begin{aligned}
	\sum_{k>m, (k,m) \in\mathcal{E}}
	\E |x_{m,k}^{(\ell)}|^3
	= & \sum_{k>m, (k,m) \in\mathcal{E}}  \frac{e^{\theta_k^*}e^{\theta_m^*}}{(e^{\theta_k^*}+e^{\theta_m^*})^2} \Big((\frac{e^{\theta_k^*}}{e^{\theta_k^*}+e^{\theta_m^*}})^2 + (\frac{e^{\theta_m^*}}{e^{\theta_k^*}+e^{\theta_m^*}})^2 \Big) \big|[{\Theta}^*_{11}]_j (\bm{e}_k-\bm{e}_m)\big|^3 \\
	= & \sum_{k>m, (k,m) \in\mathcal{E}}  \frac{e^{\theta_k^*}e^{\theta_m^*}}{(e^{\theta_k^*}+e^{\theta_m^*})^2} \big([{\Theta}^*_{11}]_j (\bm{e}_k-\bm{e}_m)\big)^2 \Big(\frac{(e^{\theta_k^*})^2+(e^{\theta_m^*})^2}{(e^{\theta_k^*}+e^{\theta_m^*})^2}  \big|[{\Theta}^*_{11}]_j (\bm{e}_k-\bm{e}_m)\big| \Big) \\
	\leq &\ [{\Theta}^*_{11}]_{jj} \max_{j,k,m} \Big(\frac{(e^{\theta_k^*})^2+(e^{\theta_m^*})^2}{(e^{\theta_k^*}+e^{\theta_m^*})^2}  \big|[{\Theta}^*_{11}]_j (\bm{e}_k-\bm{e}_m)\big| \Big)\\
	\lesssim &\ [{\Theta}^*_{11}]_{jj} \frac{1}{np},
	\end{aligned}
	\end{equation} 
	\end{small}
	where the third equality also follows from \eqref{eqn:product}, and the last inequality comes from Remark \ref{rmk:1}.
	
	Finally, by Berry-Esseen Theorem in Lemma \ref{lem:BE}, we have
\begin{small}
		\begin{equation}\nonumber 
	\begin{aligned}
	\sup_x \bigg|\P \Big(\sqrt{L} \ \frac{[{\Theta_{11}^*}]_j \nabla \cL ({\theta}^{*} )}
	{\sqrt{[{\Theta}^*_{11}]_{jj}}} \leq x  \Biggiven \mathcal{E}\Big) - \Phi(x) \bigg| 
	&\lesssim \Big(\sum_{\ell=1}^{L} \sum_{k>m, (k,m) \in\mathcal{E}}
	\mathrm{Var}(x_{m,k}^{(\ell)})\Big)^{-2/3} \Big(\sum_{\ell=1}^{L} \sum_{k>m, (k,m) \in\mathcal{E}}
	\E |x_{m,k}^{(\ell)}|^3\Big)\\
	&\lesssim 
	\dfrac{L [{\Theta}^*_{11}]_{jj} \frac{1}{np}}
	{(L[\Theta^*_{11}]_{jj})^{3/2}}
	= 
	\dfrac{\frac{1}{np}}
	{\sqrt{L [\Theta^*_{11}]_{jj}}} 
	\lesssim 
	\dfrac{1}
	{\sqrt{npL}},
	\end{aligned}
	\end{equation}
\end{small}
	where the last inequality follows from Corollary \ref{lem:eigenvalue_Q}.  Thus, we have that \begin{equation}\nonumber 
\sqrt{L} \ \frac{[{\Theta_{11}^*}]_j \nabla \cL ({\theta}^{*} )}
{\sqrt{[{\Theta}^*_{11}]_{jj}}}     \Biggiven \mathcal{E}
\xrightarrow{d} N(0,1). 
\end{equation}
	By similar arguments, we also have 
	\begin{equation}\nonumber
	\sqrt{L}\frac{([{\Theta}^*_{11}]_i-[{\Theta}^*_{11}]_j) \nabla \cL ({\theta}^{*} )}{\sqrt{(\bm{e}_i-\bm{e}_j)^\top {\Theta}^*_{11} (\bm{e}_i-\bm{e}_j)}}  \   \Biggiven \mathcal{E}
	\xrightarrow{d} N(0,1),
	\end{equation}
which concludes the proof.	
\end{proof}

\section{Proof of Theorem \ref{thm:bootstrap}}\label{sec:pf:bootstrap}
\begin{proof}
We aim to prove that $c_W(\alpha,E)$ in \eqref{eqn:quantile} obtained from the Gaussian multiplier bootstrap is a valid quantile estimator of $T=\max_{i, j \in E} \sqrt{{npL}} (\hat{\theta}^d_i-\theta^*_i-\hat{\theta}^d_j+\theta^*_j)$ in \eqref{eqn:T} where $E \subseteq \mathcal{V} \times \mathcal{V}$ is a general fixed edge set.

As we mentioned before, we have two parts of randomness here. The first part comes from the random comparison graph~$\cE$, and the other part comes from the binary pairwise comparison outcome $y_{mk}^{(\ell)}$ in \eqref{eqn:y}. In what follows, we first assume that the comparison graph $\mathcal{E}$ is fixed so that we only have randomness from $y_{mk}^{(\ell)}$.
Following  \cite{Chernozhukov2013Gaussian}, we define $$T_0 := \max_{(i, j) \in E} \sqrt{\frac{1}{L}}\sum_{\ell=1}^L \bigg\{-\sqrt{{np}}\sum_{k>m}{\mathcal{E}_{mk}}(-y_{mk}^{(\ell)}+\frac{e^{\theta_k^*}}{e^{\theta_k^*}+e^{\theta_m^*}})([{\Theta}^*_{11}]_i-[{\Theta}^*_{11}]_j)(\bm{e}_k-\bm{e}_m)\bigg\},$$   
	$$W_0:=\max_{(i, j) \in E} \sqrt{\frac{1}{L}}\sum_{\ell=1}^L \bigg\{-\sqrt{{np}}\sum_{k>m}{\mathcal{E}_{mk}}(-y_{mk}^{(\ell)}+\frac{e^{\theta_k^*}}{e^{\theta_k^*}+e^{\theta_m^*}})([{\Theta}^*_{11}]_i-[{\Theta}^*_{11}]_j)(\bm{e}_k-\bm{e}_m)\bigg\}  z_{\ell},$$ such that $T$ and $W$ in \eqref{eqn:T} and \eqref{eqn:W} can be approximated by $T_0$ and $W_0$ respectively where $\cE_{mk} = 1$ if $(m,k)\in \cE$, $\cE_{mk} = 0$ otherwise. 

	Recall that $x_{mk}^{(\ell)},\ k>m, (k,m) \in\mathcal{E},  \ell=1,\cdots, L$ are defined in  \eqref{eqn:def-x}, which are independent given $\cE$. The key of our proof is to check the following conditions in Corollary 3.1 of \cite{Chernozhukov2013Gaussian} that
{	\begin{enumerate}
		
		\item 
		$\mathbb{P}\left(\left|T-T_{0}\right|>\zeta_{1}\right)<\zeta_{2} $, $
		\mathbb{P}\left(\mathbb{P}\left(\left|W-W_{0}\right|>\zeta_{1}\given \bm{y}\right)>\zeta_{2}\right)<\zeta_{2}$ with $ \zeta_{1} \sqrt{\log n}+\zeta_{2} =o(1)$,
		
		\item $c \leqslant \frac{1}{L}\sum_{\ell=1}^{L}{\E}\big[(x_{ij}^{(\ell)})^{2}\big]$ for some constant $c>0$,
		\item $|x_{ij}^{(\ell)}|\lesssim B$  with $B^{2}(\log (n L))^{7} / L =o(1)$ where $B$ is not necessarily a constant,
	\end{enumerate} where $\bm{y}=\{y_{mk}^{(\ell)},\ k>m, (k,m) \in\mathcal{E},  \ell=1,\cdots, L\}$. } 
	
In what follows, we first prove the main result assuming the three above conditions hold, and then show the three conditions hold. In particular, given the comparison graph $\mathcal{E}$, if  the following conditions hold,
\begin{equation}\label{eqn:grad_conc1}
\|\nabla \cL(\hat{{\theta}})-\nabla \cL({{\theta}}^*)\|_{\infty} \lesssim np \sqrt{\frac{\log n}{npL}},
\end{equation}

\begin{equation}\label{eqn:Theta l2}
\|\Theta_{11}\|_2 \lesssim \frac{1}{np},
\end{equation}

	\begin{equation}\label{eqn:3} 
(\bm{e}_h-\bm{e}_k)^\top\Theta_{11}(\bm{e}_h-\bm{e}_k) 
 \gtrsim \frac{1}{np},
\end{equation} 

\begin{equation}\label{eqn:4} 
	\big\|\hat{\Theta}_{11}-{\Theta}^*_{11}\big\|_2\lesssim \dfrac{1}{np}\sqrt{\frac{\log n}{pL}},
\end{equation}

\begin{equation}\label{eqn:degree}
\frac{1}{2}np\leq d_{\min}\leq d_{\max} \leq \frac{3}{2}np,
\end{equation} 
 where $d_{\min}=\min_{1\leq i \leq n} d_i$
and $d_{\max}=\max_{1\leq i \leq n} d_i$,
 we have
\begin{equation}\nonumber
\sup_{\alpha \in (0,1)} |\mathbb{P}(T > c_W(\alpha)\given \mathcal{E} )-\alpha| = o(1)
\end{equation} 
by Corollary 3.1 in \cite{Chernozhukov2013Gaussian}. 

Meanwhile, by Remark~\ref{rmk:grad_conc}, Corollary \ref{lem:eigenvalue_Q} , Remark \ref{rmk:1}, Remark \ref{rmk:eqn:4}, Lemma \ref{lem:degree_concentration}, we have that the conditions  \eqref{eqn:grad_conc1}, \eqref{eqn:Theta l2}, \eqref{eqn:3}, \eqref{eqn:4} and \eqref{eqn:degree} are satisfied with probability at least $1-\cO(n^{-5})$, so we have that as $n,L \to \infty$, our main claim holds that
\begin{equation}\nonumber
\sup_{\alpha \in (0,1)} \big|\mathbb{P}(T > c_W(\alpha))-\alpha\big| \to 0.
\end{equation} 
	
%
%
It  remains to check the above three conditions.	For the third condition, we have
	\begin{equation}\nonumber
	\begin{aligned}
	|x_{ij}^{(\ell)}|=&\Big|\sqrt{{np}} \sum_{k>m}{\mathcal{E}_{mk}}(-y_{mk}^{(\ell)}+\frac{e^{\theta_k^*}}{e^{\theta_k^*}+e^{\theta_m^*}})([{\Theta}^*_{11}]_i-[{\Theta}^*_{11}]_j)(\bm{e}_k-\bm{e}_m)\Big|\\
	\leq & \sqrt{np}\big(\|[{\Theta}^*_{11}]_i\|_1+\|[{\Theta}^*_{11}]_j\|_1\big)\Big\| \sum_{k>m}{\mathcal{E}_{mk}}(-y_{mk}^{(\ell)}+\frac{e^{\theta_k^*}}{e^{\theta_k^*}+e^{\theta_m^*}})(\bm{e}_k-\bm{e}_m)\Big\|_\infty\\
	\lesssim& \sqrt{np} \times \sqrt{n}\frac{1}{np} \times np 
	\lesssim n\sqrt{p},
	\end{aligned}
	\end{equation}
	where the first inequality follows from H\"older's inequality, and the second inequality follows from the fact that $\max_{1 \leq j \leq n} \|[{\Theta}^*_{11}]_j\|_1\leq \|{\Theta}^*_{11}\|_\infty\leq \sqrt{n}\|{\Theta}^*_{11}\|_2 \lesssim \sqrt{n}\frac{1}{np}$ by  \eqref{eqn:Theta l2} and degree concentration by \eqref{eqn:degree}. Therefore, we have $B=n\sqrt{ p  }$, which leads to the condition $n^2 p \frac{(\log (n L))^{7}}{L}  =o(1)$.

	For the second condition, we have 
	\begin{equation}\nonumber
	\begin{aligned}
	\frac{1}{L}\sum_{\ell=1}^{L}{\E}\left[(x_{ij}^{(\ell)})^{2}\right]
	&={\E}\left[(x_{ij}^{(\ell)})^{2}\right]
	=\mathrm{Var}(x_{ij}^{(\ell)})\\
	&=np \sum_{k>m}{\mathcal{E}_{mk}}\frac{e^{\theta_k^*}e^{\theta_m^*}}{(e^{\theta_k^*}+e^{\theta_m^*})^2}([{\Theta}^*_{11}]_i-[{\Theta}^*_{11}]_j)\\
	&\quad \cdot(\bm{e}_k-\bm{e}_m)(\bm{e}_k-\bm{e}_m)^\top {([{\Theta}^*_{11}]_i-[{\Theta}^*_{11}]_j)}^\top\\
	&=np ([{\Theta}^*_{11}]_i-[{\Theta}^*_{11}]_j)\nabla^2 {\cL({\theta}^*)}{([{\Theta}^*_{11}]_i-[{\Theta}^*_{11}]_j)}^\top\\
	&=np ([{\Theta}^*_{11}]_{ii}+[{\Theta}^*_{11}]_{jj}-[{\Theta}^*_{11}]_{ij}-[{\Theta}^*_{11}]_{ji})\\
	&=np(\bm{e}_i-\bm{e}_j)^\top {\Theta}^*_{11} (\bm{e}_i-\bm{e}_j),
	\end{aligned}
	\end{equation}
	where the fourth equality holds by ${\Theta}^*_{11} \nabla^2 {\cL({\theta}^*)} {{\Theta}^*_{11}}^\top={\Theta}^*_{11}$ in \eqref{eqn:product}. 
From \eqref{eqn:3}, we have that the second condition holds that
	\begin{equation}\nonumber
	\frac{1}{L}\sum_{\ell=1}^{L}{\E}\left[(x_{ij}^{(\ell)})^{2}\right]=np(\bm{e}_i-\bm{e}_j)^\top {\Theta}^*_{11} (\bm{e}_i-\bm{e}_j) \gtrsim np\frac{1}{np} = 1.
	\end{equation} 
	For the first condition, we have 
		\begin{equation}\nonumber
		\begin{aligned}
		|W-W_0|
		\leq& \max_{(i, j) \in E} \Bigg| \sqrt{\frac{1}{L}}\sum_{\ell=1}^L \sqrt{{np}}\sum_{k>m}{\mathcal{E}_{mk}} \bigg\{\Big(y_{mk}^{(\ell)}-\frac{e^{\hat{\theta}_k}}{e^{\hat{\theta}_k}+e^{\hat{\theta}_m}}\Big)\big([\hat{\Theta}_{11}]_i-[\hat{\Theta}_{11}]_j\big)  \\ &-\Big(y_{mk}^{(\ell)}-\frac{e^{\theta_k^*}}{e^{\theta_k^*}+e^{\theta_m^*}}\Big)\big([{\Theta}^*_{11}]_i-[{\Theta}^*_{11}]_j\big)\bigg\}(\bm{e}_k-\bm{e}_m)  z_l \Bigg|.
		\end{aligned}
		\end{equation}
Conditioning on the data $\bm{y}$, the  above equation is  suprema of a Gaussian process. So we first bound the following conditional variance that
	\begin{small}
	$$
\begin{aligned}
		&\max_{(i,j) \in E}{\frac{1}{L}}\sum_{\ell=1}^L \Bigg[  \sqrt{{np}}\sum_{k>m}{\mathcal{E}_{mk}} \bigg\{\Big(y_{mk}^{(\ell)}-\frac{e^{\hat{\theta}_k}}{e^{\hat{\theta}_k}+e^{\hat{\theta}_m}}\Big)\big([\hat{\Theta}_{11}]_i-[\hat{\Theta}_{11}]_j\big) \\
		&\qquad\quad-\Big(y_{mk}^{(\ell)}-\frac{e^{\theta_k^*}}{e^{\theta_k^*}+e^{\theta_m^*}}\Big)\big([{\Theta}^*_{11}]_i-[{\Theta}^*_{11}]_j\big)\bigg\}(\bm{e}_k-\bm{e}_m)\Bigg]^2\\
		\lesssim & \underbrace{\max_{(i, j) \in E}{\frac{1}{L}}\sum_{\ell=1}^L \Bigg[ \sqrt{{np}}\sum_{k>m}{\mathcal{E}_{mk}} \bigg(\frac{e^{\hat{\theta}_k}}{e^{\hat{\theta}_k}+e^{\hat{\theta}_m}}-\frac{e^{\theta_k^*}}{e^{\theta_k^*}+e^{\theta_m^*}}\bigg)\big([{\Theta}^*_{11}]_i-({\Theta_{11}^*})_j\big)(\bm{e}_k-\bm{e}_m) \Bigg]^2}_{A_1}\\
		& + \underbrace{\max_{(i, j) \in E}{\frac{1}{L}}\sum_{\ell=1}^L \Bigg[ \sqrt{{np}}\sum_{k>m}{\mathcal{E}_{mk}} \bigg(-y_{mk}^{(\ell)}+\frac{e^{\hat{\theta}_k}}{e^{\hat{\theta}_k}+e^{\hat{\theta}_m}}\bigg)\big([\hat{\Theta}_{11}]_i-[\hat{\Theta}_{11}]_j-[{\Theta}^*_{11}]_i+[{\Theta}^*_{11}]_j\big)(\bm{e}_k-\bm{e}_m) \Bigg]^2}_{A_2}. \notag
		\end{aligned}
		$$
		\end{small}
	Then we control $A_1$ and $A_2$ separately. 
	For $A_1$, with probability at least $1-\cO (n^{-5})$, we have
	$$
	\begin{aligned}
	A_1=& \max_{(i, j) \in E}{\frac{1}{L}}\sum_{\ell=1}^L  \Bigg[ \sqrt{{np}}\sum_{k>m}{\mathcal{E}_{mk}} \bigg(\frac{e^{\hat{\theta}_k}}{e^{\hat{\theta}_k}+e^{\hat{\theta}_m}}-\frac{e^{\theta_k^*}}{e^{\theta_k^*}+e^{\theta_m^*}}\bigg)([{\Theta}^*_{11}]_i-({\Theta_{11}^*})_j)(\bm{e}_k-\bm{e}_m) \Bigg]^2\\
	=&\max_{(i, j) \in E} {np}\Big[\big([{\Theta}^*_{11}]_i-[{\Theta}^*_{11}]_j\big) \big(\nabla \cL(\hat{{\theta}})-\nabla \cL({{\theta}}^*)\big)\Big]^2\\
	\leq&\max_{(i, j) \in E} {np}\Big[\big\|[{\Theta}^*_{11}]_i-[{\Theta}^*_{11}]_j\big\|_1 \big\|\nabla \cL(\hat{{\theta}})-\nabla \cL({{\theta}}^*)\big\|_\infty\Big]^2\\
	\lesssim&\ np(\sqrt{n}\big\|{\Theta}^*_{11}\big\|_2)^2\big\|\nabla \cL(\hat{{\theta}})-\nabla \cL({{\theta}}^*)\big\|_\infty^2\\
	\lesssim &\ np\times n \Big(\frac{1}{np}\Big)^2 \Big(np \sqrt{\frac{\log n}{npL}} \Big)^2
	\lesssim \ \frac{n\log n}{L},
	\end{aligned}
	$$
	where the second inequality follows from $\max_{(i, j) \in E}\|[{\Theta}^*_{11}]_i-[{\Theta}^*_{11}]_j\|_1^2\leq \max_{(i, j) \in E} (\|[{\Theta}^*_{11}]_i\|_1+\|[{\Theta}^*_{11}]_j\|_1)^2\leq  \max_{1\leq j \leq n}(2\|[{\Theta}^*_{11}]_j\|_1)^2\lesssim \|{\Theta}^*_{11}\|_\infty^2\lesssim (\sqrt{n}\|{\Theta}^*_{11}\|_2)^2$ by \eqref{eqn:Theta l2}. The second inequality follows from \eqref{eqn:grad_conc1}.

	For $A_2$, we have that
	$$
	\begin{aligned}
	A_2=& \max_{(i, j) \in E}{\frac{1}{L}}\sum_{\ell=1}^L \Bigg[ \sqrt{{np}}\sum_{k>m}{\mathcal{E}_{mk}} \bigg(-y_{mk}^{(\ell)}+\frac{e^{\hat{\theta}_k}}{e^{\hat{\theta}_k}+e^{\hat{\theta}_m}}\bigg)\big([\hat{\Theta}_{11}]_i-[\hat{\Theta}_{11}]_j-[{\Theta}^*_{11}]_i+[{\Theta}^*_{11}]_j\big)(\bm{e}_k-\bm{e}_m) \Bigg]^2\\
	\leq&\ np  \Bigg(\max_{(i, j) \in E} \big\|[\hat{\Theta}_{11}]_i-[\hat{\Theta}_{11}]_j-[{\Theta}^*_{11}]_i+[{\Theta}^*_{11}]_j\big\|_1 \max_{\ell} \bigg\|\sum_{k>m}{\mathcal{E}_{mk}} \Big(-y_{mk}^{(\ell)}+\frac{e^{\hat{\theta}_k}}{e^{\hat{\theta}_k}+e^{\hat{\theta}_m}}\Big)(\bm{e}_k-\bm{e}_m) \bigg\|_\infty\Bigg)^2\\
	\lesssim&\ np\times n \bigg(\dfrac{1}{np}\sqrt{\frac{\log n}{pL}}\bigg)^2\times (np)^2
	\lesssim\frac{n^2 \log n}{L},
	\end{aligned}
	$$
	where the second inequality follows from $\max_{(i, j) \in E}\big\|[\hat{\Theta}_{11}]_i-[\hat{\Theta}_{11}]_j-[{\Theta}^*_{11}]_i+[{\Theta}^*_{11}]_j\big\|_1^2\leq \max_{(i, j) \in E} \big(\big\|[\hat{\Theta}_{11}]_i-[{\Theta}^*_{11}]_i\big\|_1+\big\|[\hat{\Theta}_{11}]_j-[{\Theta}^*_{11}]_j\big\|_1\big)^2\leq  \max_{1\leq j \leq n}\big(2\big\|[\hat{\Theta}_{11}]_j-[{\Theta}^*_{11}]_j\big\|_1\big)^2\lesssim \big\|\hat{\Theta}_{11})-{\Theta}^*_{11}\big\|_\infty^2\lesssim n\big\|\hat{\Theta}_{11}-{\Theta}^*_{11}\big\|_2^2\lesssim n\Big(\dfrac{1}{np}\sqrt{\frac{\log n}{pL}}\Big)^2$ by \eqref{eqn:4} and degree concentration by \eqref{eqn:degree}.
	
	Therefore, let event $\mathfrak{E}$ be \begin{small}
		$$
		\begin{aligned}
		& \max_{(i, j) \in E}{\frac{1}{L}}\sum_{\ell=1}^L \Bigg[  \sqrt{{np}}\sum_{k>m}{\mathcal{E}_{mk}} \bigg\{(y_{mk}^{(\ell)}-\frac{e^{\hat{\theta}_k}}{e^{\hat{\theta}_k}+e^{\hat{\theta}_m}})([\hat{\Theta}_{11}]_i-[\hat{\Theta}_{11}]_j) \\
		&\qquad\qquad -(y_{mk}^{(\ell)}-\frac{e^{\theta_k^*}}{e^{\theta_k^*}+e^{\theta_m^*}})([{\Theta}^*_{11}]_i-[{\Theta}^*_{11}]_j)\bigg\}(\bm{e}_k-\bm{e}_m)\Bigg]^2 \lesssim  n^2\frac{\log n}{L}
		\end{aligned}
		$$
	\end{small} with $\mathbb{P}(\mathfrak{E}^C)\leq \frac{1}{n^5}$. Then by maximal inequality, under the event $\mathfrak{E}$, we have
		$$
		\begin{aligned}
		&\mathbb{E}\bigg[\max_{(i, j) \in E}
		\sqrt{\frac{1}{L}}\sum_{\ell=1}^L \sqrt{{np}}\sum_{k>m}{\mathcal{E}_{mk}} \Bigg\{(y_{mk}^{(\ell)}-\frac{e^{\hat{\theta}_k}}{e^{\hat{\theta}_k}+e^{\hat{\theta}_m}})([\hat{\Theta}_{11}]_i-[\hat{\Theta}_{11}]_j) \\
		&\qquad -(y_{mk}^{(\ell)}-\frac{e^{\theta_k^*}}{e^{\theta_k^*}+e^{\theta_m^*}})([{\Theta}^*_{11}]_i-[{\Theta}^*_{11}]_j)\Bigg\}(\bm{e}_k-\bm{e}_m) z_l 
		\ \Big| \ 
		  \bm{y}\bigg]\lesssim n\log n /\sqrt{L}. 
		\end{aligned}
		$$Then by Borell's inequality, we have 
		$$
		\begin{aligned}
		\mathbb{P}\Bigg(\max_{(i, j) \in E} &
		\sqrt{\frac{1}{L}}\sum_{\ell=1}^L \sqrt{{np}}\sum_{k>m}{\mathcal{E}_{mk}} \bigg\{\Big(y_{mk}^{(\ell)}-\frac{e^{\hat{\theta}_k}}{e^{\hat{\theta}_k}+e^{\hat{\theta}_m}}\Big)[\hat{\Theta}_{11}]_j  -\Big(y_{mk}^{(\ell)}-\frac{e^{\theta_k^*}}{e^{\theta_k^*}+e^{\theta_m^*}}\Big)[{\Theta}^*_{11}]_j\bigg\}\\
&\cdot		(\bm{e}_k-\bm{e}_m)  z_l
		\gtrsim \dfrac{n\log n}{\sqrt{L}} \ \Big| \   \bm{y}\Bigg)\leq \frac{1}{n^5},
		\end{aligned}
		$$
	which implies
	$$\mathbb{P}\left(\mathbb{P}\left(\left|W-W_{0}\right|\gtrsim n\log n /\sqrt{L}\ \Big| \by \right)>1 / n^5\right)<1 / n^5.$$  
	Meanwhile, with probability at least $1-\cO (n^{-5})$, we have $|T-T_0|\leq\max_{(i, j) \in E}\sqrt{{npL}}|r_{i}-r_j|\leq2\sqrt{{npL}}\|r\|_\infty\lesssim \frac{n\log n}{\sqrt{L}}$ since $\|r\|_\infty\lesssim\frac{\sqrt{n}\log n}{\sqrt{p}\ L}+\frac{\log n}{\sqrt{n}pL}$ and $p\gtrsim\sqrt{\frac{\log n}{npL}}$. Thus, we have  that $$\mathbb{P}\left(\left|T-T_{0}\right|>\zeta_{1}\right)<\zeta_{2}, $$
	and $$\mathbb{P}\left(\mathbb{P}\left(\left|W-W_{0}\right|>\zeta_1 \given \bm{y} \right)>\zeta_2\right)<\zeta_2,$$
	where \begin{equation}\nonumber
	\zeta_{1}=C\frac{n\log n}{\sqrt{L}} \text{ and } \zeta_{2}=1 / n^5.
	\end{equation} Thus, the first condition $\zeta_1\sqrt{\log n}+\zeta_2 \lesssim \frac{n\log n}{\sqrt{L}}\sqrt{\log n}+n^{-5}=o(1)$ holds, which concludes the proof.
	
\end{proof}
\begin{remark}
	\label{rmk:c3}
In our Theorem \ref{thm:bootstrap}, we only need $T$ to satisfy $$\sup_{\alpha \in (0,1)} \big|\mathbb{P}(T > c_W(\alpha,E) )-\alpha\big|=o(1) .$$
However, if we need a stronger result that $$\sup _{\alpha \in(0,1)}\left|\mathrm{P}\left(T > c_{W}(\alpha)\right) -\alpha\right| \lesssim L^{-c_3} + n^{-5},$$
which will be needed in FDR controlling procedure in Section \ref{sec:FWER},
$n, L, p$ need to satisfy stronger condition by the following argument.

 By Theorem 3.2 in \cite{Chernozhukov2013Gaussian}, $c_3$ is a constant satisfying 
\begin{equation}\label{eqn:cond-c3}
\sup _{\alpha \in(0,1)}\left|\mathrm{P}\left(T > c_{W}(\alpha)\right) -\alpha\right| \leq \rho_{\ominus} + \rho \lesssim L^{-c_3} + \frac{1}{n^5},
\end{equation}
where $\rho_{\ominus}$ is bounded by
$$
\rho_{\ominus} \leq 2(\rho+\pi(\vartheta)+\mathrm{P}(\Delta>\vartheta))+C_{3} \zeta_{1} \sqrt{1 \vee \log \left(n / \zeta_{1}\right)}+5 \zeta_{2}
$$
for every $\vartheta>0$, and $\rho$ is bounded by \eqref{eqn:rho}. 
{We aim to find a sufficient condition for \eqref{eqn:cond-c3}.}

We have 
$$\pi(\vartheta) = C_{2} \vartheta^{1 / 3}(1 \vee \log (n / \vartheta))^{2 / 3}, $$ from Lemma 3.2 in \cite{Chernozhukov2013Gaussian},
and taking $\vartheta=(\mathrm{E}[\Delta])^{1 / 2} / \log n$ from proof of Corollary 3.1 in \cite{Chernozhukov2013Gaussian}. 
We also have $$\PP(\Delta>\vartheta) \leqslant \EE[\Delta] / \vartheta, $$ 
where $\EE[\Delta]$ is further bounded by $$\EE[\Delta] \lesssim \sqrt{\frac{B^{2} \log n}{L}} \bigvee \frac{B^{2}(\log (n L))^{2}(\log n)}{L} =  \sqrt{\frac{n^2 p \log n}{L}} \bigvee \frac{n^2 p(\log (n L))^{2}(\log n)}{L}:=B_{\EE[\Delta]},$$ from Lemma C.1 in \cite{Chernozhukov2013Gaussian}.

Combining with $\zeta_{1}=C\frac{n\log n}{\sqrt{L}}$ and $\zeta_{2}=1 / n^5$, we have 
%
$$\rho_{\ominus} \lesssim \rho +\sqrt{B_{\EE[\Delta]}}\log n + \frac{(B_{\EE[\Delta]})^{1/6}}{(\log n)^{1/3}}\Big\{\Big(\log n + \log (\log n) - \frac{1}{2}\log B_{\EE[\Delta]}\Big)^{2/3} \bigvee 1\Big\} + n \log n\sqrt{\frac{\log n}{L}} + \frac{1}{n^5}.$$

Meanwhile, by Theorem 2.2 in \cite{Chernozhukov2013Gaussian},  for every $\gamma \in(0,1)$,
we have 
\begin{equation}\label{eqn:rho}
\rho \lesssim L^{-1 / 8}\left(M_{3}^{3 / 4} \vee M_{4}^{1 / 2}\right)(\log (n L / \gamma))^{7 / 8}+L^{-1 / 2}(\log (n L / \gamma))^{3 / 2} u(\gamma)+\gamma,
\end{equation}
 and taking $\gamma=L^{-c'}$ from the proof of corollary 2.1 in \cite{Chernozhukov2013Gaussian}. We further have
$$u(\gamma) \lesssim \max \left\{D h^{-1}(L / \gamma), B' \sqrt{\log (n L / \gamma)}\right\}$$ from Lemma 2.2 in \cite{Chernozhukov2013Gaussian}, and 
$$B' \vee M_{3}^{3} \vee M_{4}^{2} \lesssim B, \ D \lesssim B \log n,\  h(v)=e^{v}-1 $$ from proof of Corollary 2.1. 

Thus, for every $c'>0$,
$$\rho \color{black}\lesssim \frac{n^{1/4}p^{1/8}(\log n+(1+c')\log L)^{7/8}}{L^{1/8}}
+ \frac{np^{1/2}\log n(1+c')\log L(\log n+(1+c')\log L)^{3/2}} {L^{1/2}}
+L^{-c'}.$$ 

{ Combining the above inequalities together, if we impose a stronger condition that there exists $c'>0$ and $c_3>0$ such that 
\begin{equation}\nonumber
\begin{aligned}
&\sqrt{B_{\EE[\Delta]}}\log n + \frac{(B_{\EE[\Delta]})^{1/6}}{(\log n)^{1/3}}\Big\{\Big(\log n + \log (\log n) - \frac{1}{2}\log B_{\EE[\Delta]}\Big)^{2/3} \vee 1\Big\} + n \log n\sqrt{\frac{\log n}{L}} + \frac{1}{n^5}\\
& +  \frac{n^{1/4}p^{1/8}(\log n+(1+c')\log L)^{7/8}}{L^{1/8}}
+ \frac{np^{1/2}\log n(1+c')\log L(\log n+(1+c')\log L)^{3/2}} {L^{1/2}}
+L^{-c'} \\
&\quad \lesssim L^{-c_3} + n^{-5}, 
\end{aligned}
\end{equation}	
then we have the stronger conclusion that
$$\sup _{\alpha \in(0,1)}\left|\mathrm{P}\left(T > c_{W}(\alpha)\right) -\alpha\right| \lesssim L^{-c_3} + n^{-5}.$$ }
\end{remark}

\section{Proofs in Section \ref{sec:ht}}
\label{sec:pf:ht}
\subsection{Proof of Theorem \ref{thm:type I error}}
\begin{proof}
We show that the Type I error can be well-controlled at desired level, and the power is asymptotically one with the required signal strength. 
First, for the Type I error, we let $H(\gamma,i)=\{j \in [n]: \gamma_j < \gamma_i \} = \{j \in [n]: \theta_j > \theta_i \}$ and $L(\gamma,i)=\{j \in [n]: \gamma_j > \gamma_i \} = \{j \in [n]: \theta_j < \theta_i \}$.
Also, recall that we let the collection of all possible score vector after perturbation  be  $\tilde{\Theta}$ as defined in~\eqref{eqn:set theta tilde}. 
In particular, we let  $\tilde{{\theta}}$  be the score vector after some perturbation that
\begin{equation}\nonumber
\tilde{\theta}_k=
\left\{\begin{array}{ll}
\hat{\theta}^d_{k}+c_W(\alpha,i)/\sqrt{npL}, & k \in H^* \\
\hat{\theta}^d_i, & k=i \\
\hat{\theta}^d_{k}-c_W(\alpha,i)/\sqrt{npL}, & k \in L^*
\end{array}\right.
\end{equation}
where we let $H^*=H(\gamma^*,i), L^*=L(\gamma^*,i)$ for ease of notation.
Thus, we have
\begin{equation}\nonumber
\sup_{\gamma^* \notin \cR_i} \P_0 (\text{reject } H_0) 
=\sup_{\gamma^* \notin \cR_i} \P_0 \big(\bigcap_{{{\theta}} \in \tilde{{\Theta}}}\{\gamma({{\theta}}) \in \cR_i\} \big)
\leq\sup_{\gamma^* \notin \cR_i} \P_0 \big(\gamma(\tilde{{\theta}}) \in \cR_i \big).
\end{equation}
%
%
%
We first bound the probability of event 
$\{\gamma(\tilde{{\theta}}) \in \cR_i | \gamma^* \notin \cR_i\}$,
which is the event that one of the following two events holds
$$E_1:\  \exists j \in H(\gamma^*,i) \text{ such that }
\theta_j^* > \theta_i^*, \tilde{\theta}_j < \tilde{\theta}_i, $$ 
or
$$E_2:\  \exists j \in L(\gamma^*,i) \text{ such that }
\theta_j^* < \theta_i^*, \tilde{\theta}_j > \tilde{\theta}_i. $$
Otherwise, we have that for all $ j \in H(\gamma^*,i), 
\theta_j^* > \theta_i^*, \tilde{\theta}_j > \tilde{\theta}_i$ and for all $ j \in L(\gamma^*,i),
\theta_j^* < \theta_i^*, \tilde{\theta}_j < \tilde{\theta}_i$, which implies $\gamma(\tilde{{\theta}})$ and $\gamma^* $ are in the same equivalent class, and cannot hold since $\gamma(\tilde{{\theta}}) \in \cR_i$ and
$\gamma^* \notin \cR_i$. Thus, under the event $\{\gamma(\tilde{{\theta}}) \in \cR_i | \gamma^* \notin \cR_i\}$, $E_1$ or $E_2$  holds.

Then, we consider the two events $E_1$ and $E_2$ separately. Suppose that $E_1$ holds. We have
$$\min_{j \in H(\gamma^*,i)}\theta_j^* > \theta_i^*
,\ \min_{j \in H(\gamma^*,i)}\tilde{\theta}_j < \tilde{\theta}_i, $$
i.e.,
$$\min_{j \in H(\gamma^*,i)}\theta_j^* > \theta_i^*
,\ \min_{j \in H(\gamma^*,i)}\hat{\theta}_j^d +c_W(\alpha,i)/\sqrt{npL} < \hat{\theta}_i^d. $$
By the fact that 
$\min_j(a_j + b_j) \geq \min_j a_j + \min_j b_j$, we have
\begin{equation}\nonumber
\begin{aligned}
c_W(\alpha,i)/\sqrt{npL}
 \leq 
\hat{\theta}^d_{i}-\min _{j \in H(\gamma^*,i)}\hat{\theta}^d_{j}
&\leq
\hat{\theta}^d_{i}-\min _{j \in H(\gamma^*,i)} \theta_{j}^{*} 
-\min _{j \in H(\gamma^*,i)}\left(\hat{\theta}^d_{j}-\theta_{j}^{*}\right) \\
&=
\theta_{i}^{*}-\min _{j \in H(\gamma^*,i)} \theta_{j}^{*} 
+\max _{j \in H(\gamma^*,i)}\left(\hat{\theta}^d_{i}-\theta_{i}^{*}-\hat{\theta}^d_{j}+\theta_{j}^{*}\right),
\end{aligned}
\end{equation}
which gives $$\max _{j \in H(\gamma^*,i)}\big(\hat{\theta}^d_{i}-\theta_{i}^{*}-\hat{\theta}^d_{j}+\theta_{j}^{*}\big)>c_W(\alpha,i)/\sqrt{npL}.$$
Thus,  we have $$
\{E_1\}
\subseteq
\Big\{\max _{j \in H(\gamma^*,i)}\big(\hat{\theta}^d_{i}-\theta_{i}^{*}-\hat{\theta}^d_{j}+\theta_{j}^{*}\big)>c_W(\alpha,i)/\sqrt{npL}\Big\}.$$
For the event $E_2$, by similar arguments, we have 
$$\{E_2\}\subseteq 
\Big\{\max _{j \in L(\gamma^*,i)}\big(\hat{\theta}^d_{i}-\theta_{i}^{*}-\hat{\theta}^d_{j}+\theta_{j}^{*}\big)>c_W(\alpha,i)/\sqrt{npL}\Big\}. 
$$ 
Combining them together, we have
\begin{equation}
\label{eqn:subseteq}
\{\gamma(\tilde{{\theta}}) \in \cR_i | \gamma^* \notin \cR_i\}
\subseteq
\{E_1\}
\cup
\{E_2\}
\subseteq
\Big\{\max _{j \neq i } \big(\hat{\theta}^d_{i}-\theta_{i}^{*}-\hat{\theta}^d_{j}+\theta_{j}^{*}\big)>c_W(\alpha,i)/\sqrt{npL}\Big\}.
\end{equation}
Thus, we have
\begin{equation}\nonumber
\begin{aligned}
\sup_{\gamma^* \notin \cR_i} \mathbb{P}_0 \left(\text{reject } H_0\right) 
\leq&\ 
\sup_{\gamma^* \notin \cR_i} \P_0 \big(\{\gamma(\tilde{{\theta}}) \in \cR_i\} \big)\\
\leq &\
\P_0
\Big(\max _{j \neq i}\big(\hat{\theta}^d_{i}-\theta_{i}^{*}-\hat{\theta}^d_{j}+\theta_{j}^{*}\big)>c_W(\alpha,i)/\sqrt{npL}\Big) \\
\leq &\  \alpha +o(1),
\end{aligned}
\end{equation}
where the last inequality holds by Theorem \ref{thm:bootstrap}, and the Type I error is controlled as desired.

Next, we prove our testing procedure is powerful that the power goes to one asymptotically.  In particular,
to control Type II error, we first have
$$\sup_{\gamma^* \in \cR_i} \P (\text{Do not reject } H_0 )
=\sup_{\gamma^* \in \cR_i} \P \big(\bigcup_{{{\theta}} \in \tilde{{\Theta}}}\{\gamma({{\theta}}) \notin \cR_i\} \big).
$$
Then we bound the probability of the event $\{\gamma({{\theta}}) \notin \cR_i | \gamma^* \in \cR_i\}$.
We have  
\begin{equation}\nonumber
\begin{aligned}
\{\gamma({{\theta}}) \notin \cR_i | \gamma^* \in \cR_i\}
&\subseteq
\big\{|\tilde{\theta}_i-{\theta}^*_i|+\max_{j \neq i}|\tilde{\theta}_j-{\theta}^*_j| \geq \Delta({\theta}^*,\cR_i)\big\}\\
&\subseteq
\big\{\|\hat{{\theta}}^d-{\theta}^*\|_\infty+c_W(\alpha,i)/\sqrt{npL} \geq \Delta({\theta}^*,\cR_i)/2\big\},
\end{aligned}
\end{equation}
where the first subset follows from the definition of $\Delta({\theta}^*,\cR_i)$ that if $|\tilde{\theta}_i-{\theta}^*_i|+\max_{j \neq i}|\tilde{\theta}_j-{\theta}^*_j| < \Delta({\theta}^*,\cR_i)$ and $\{\gamma^* \in \cR_i\}$, we have $\gamma({{\theta}}) \in \cR_i$), and  the second subset holds by
$\|\tilde{{\theta}}-{\theta}^*\|_\infty \leq \|\hat{{\theta}}^d-{\theta}^*\|_\infty+2 c_W(\alpha,i)/\sqrt{npL}
\leq 2 \|\hat{{\theta}}^d-{\theta}^*\|_\infty+2 c_W(\alpha,i)/\sqrt{npL}$.
Thus, we have
\begin{equation}\nonumber
\begin{aligned}
\bigcup_{{{\theta}} \in \tilde{{\Theta}}}\big\{\gamma({{\theta}}) \notin \cR_i | \gamma^* \in \cR_i\big\}
\subseteq
\big\{\|\hat{{\theta}}^d-{\theta}^*\|_\infty+c_W(\alpha,i)/\sqrt{npL} \geq \Delta({\theta}^*,\cR_i)/2\big\}.
\end{aligned}
\end{equation}
By Theorem \ref{thm:bootstrap}, for any fixed $\alpha$ and sufficiently large $L, n$, we have $$\mathbb{P}(\sqrt{{npL}} (\max_{j \neq i} (\hat{\theta}^d_i-\theta^*_i-\hat{\theta}^d_j+\theta^*_j))\geq  c_W(\alpha,i))>\alpha/2,$$ 
and by Theorem \ref{thm:asy}, we have $$\P(\sqrt{{npL}}\max_{j \neq i} |\hat{\theta}^d_i-\theta^*_i-\hat{\theta}^d_j+\theta^*_j| \leq C\sqrt{\log n}) > 1-\cO(n^{-5}),$$ 
which immediately leads to
\begin{equation}\nonumber
c_W(\alpha,i)\leq C_0 \sqrt{  \log n}
\end{equation}
for  some sufficiently large constant $C_0$ with probability at least $1-\cO(n^{-5})$.
By Lemma~\ref{lem:consistency}, we have $$\|\hat{{\theta}}^d-{\theta}^*\|_\infty
\lesssim \sqrt{\frac{\log n}{npL}}.$$
Combining the above two inequalities together, with  probability at least $1-\cO(n^{-5})$, we have 
$$\|\hat{{\theta}}^d-{\theta}^*\|_\infty+c_W(\alpha,i)/\sqrt{npL}
\lesssim \sqrt{\frac{\log n}{npL}}.$$
Then, we let
\begin{equation}\label{eqn:delta}
\Delta({\theta}^*,\cR_i)>\delta=C\sqrt{\frac{\log n}{npL}}
\end{equation} for  some sufficiently large constant $C$, and it follows that 
\begin{equation}\nonumber
 \P\Big(\|\hat{{\theta}}^d-{\theta}^*\|_\infty+c_W(\alpha,i)/\sqrt{npL} \geq \Delta({\theta}^*,\cR_i)/2\Big)
\leq \cO(n^{-5}).
\end{equation}
Finally, since the above inequality holds for any $\gamma^* \in \cR_i$,  we have
\begin{equation}\nonumber
\inf _{\gamma^* \in \cR_i, \Delta({\theta}^*,\cR_i)>\delta} \mathbb{P}\left(\text{reject } H_0\right) > 1-\cO(n^{-5}),
\end{equation}
which completes the proof.
\end{proof}

\subsection{Specific testing procedure for top $K$}\label{sec:test:topk}

For the top-$K$ inference example, we test if the $i$-th item is ranked among top $K$. Our simplified testing procedure only considers $\tilde{\theta}$, whose  $k$-th entry is
\begin{equation}\nonumber
\tilde{\theta}_k=
\left\{\begin{array}{ll}
\hat{\theta}^d_{k},  & 
\hat{\theta}^d_{k}>\hat{\theta}^d_{i}, \\
\hat{\theta}^d_i-c_W(\alpha,i)/\sqrt{npL},  &  k=i,\\
\hat{\theta}^d_{k},  & 
\hat{\theta}^d_{k}<\hat{\theta}^d_{i}.
\end{array}\right.
\end{equation}
In this section, we provide the validity and power analysis results for this simplified testing procedure. The next theorem shows that the Type I error can be well-controlled and our procedure is powerful.

\begin{theorem}
Suppose that conditions in Theorem \ref{thm:bootstrap} are satisfied. We have the proposed  procedure in Remark \ref{rmk:topk} satisfies that
\begin{equation}\nonumber
 \sup_{\theta^*_i \leq \theta^*_{(K+1)}} \mathbb{P}_0 \left(\text{reject } H_0\right) \leq \alpha\ \ \ \ \text{ as } {n,L \to \infty}.
\end{equation}
Furthermore, we have
\begin{equation}\nonumber
\inf_{ \theta^*_{i}-\theta^*_{(K+1)}>\delta} \mathbb{P}\left(\text{reject } H_0\right) \to 1 
\end{equation} 
where $ \delta=C\sqrt{\frac{\log n}{npL}}$ for some constant $C$. 
\end{theorem}	

\begin{proof}
	First we show that the Type I error can be controlled at the desired level. For ease of presentation, we  denote by $\T_0^*=\{j:\theta^*_j \leq \theta^*_{(K+1)}\}$ and $\hat{\T}_0=\{j:\hat{\theta}_j^d \leq \hat{\theta}_{(K+1)}^d\}$ that $\T_0^*$ and $\hat{\T}_0$ are the indices of the smallest $n-K$ scores in $\theta^*$ and  $\hat{\theta}^d$, respectively. Thus, we have $\theta^*_{(K+1)}=\max_{j \in {\T}^*_0}\theta^*_j$, $\hat{\theta}^d_{(K+1)}=\max_{j \in \hat{\T}_0}\hat{\theta}^d_j$.
	
		By the fact that $\max_j(a_j + b_j) \leq \max_j a_j + \max_j b_j$, we have 
	\begin{equation}\label{equ:tri-inequ}
	\theta^*_i-\max_{j \in \hat{\T}_0}\theta^*_j 
		\geq 
		\hat{\theta}^d_i-\max_{j \in \hat{\T}_0}\hat{\theta}^d_j- (\hat{\theta}^d_i-\theta^*_i+\max_{j \in \hat{\T}_0}(\theta^*_j-\hat{\theta}^d_j))
		=
		\hat{\theta}^d_i-\max_{j \in \hat{\T}_0}\hat{\theta}^d_j- \max_{j \in \hat{\T}_0}(\hat{\theta}^d_i-\theta^*_i+\theta^*_j-\hat{\theta}^d_j).
	\end{equation}
	Also, by the definition of $\T_0^*$ and $\hat{\T}_0$, we have 
	$\max_{j \in {\T}^*_0}\theta^*_j 
	\leq 
	\max_{j \in \hat{\T}_0}{\theta}^*_j 
	$.
	Combining this with \eqref{equ:tri-inequ}, we have
	$$\theta^*_i-\max_{j \in {\T}^*_0}\theta^*_j 
	\geq 
	\hat{\theta}^d_i-\max_{j \in \hat{\T}_0}\hat{\theta}^d_j- \max_{j \in \hat{\T}_0}(\hat{\theta}^d_i-\theta^*_i+\theta^*_j-\hat{\theta}^d_j).$$
Thus, we have that the Type I error satisfies
	\begin{equation}\nonumber
	\begin{aligned}
	\sup_{\theta^*_i \leq \theta^*_{(K+1)}}
	\mathbb{P}\big(\text{reject } H_0\big)
	=&\sup_{\theta^*_i \leq \theta^*_{(K+1)}}
	\mathbb{P}\big(\gamma(\tilde{\theta}) \in \cR_i \big)\\
	=& \sup_{\theta^*_i \leq \theta^*_{(K+1)}} \mathbb{P}\big(\hat{\theta}^d_i-c_W(\alpha,i)/\sqrt{{npL}} > \hat{\theta}^d_{(K+1)}\big)\\
	=& \sup_{\theta^*_i-\max_{j \in {\T}^*_0}\theta^*_j \leq 0} \mathbb{P}\big(\hat{\theta}^d_i-\max_{j \in \hat{\T}_0}\hat{\theta}^d_j > c_W(\alpha,i)/\sqrt{{npL}} \big)\\
	\leq & \
	 \mathbb{P}\Big(\sqrt{{npL}} \max_{j \in \hat{\T}_0}(\hat{\theta}^d_i-\theta^*_i-(\hat{\theta}^d_j-\theta^*_j)) >c_W(\alpha,i) \Big)\\ 
	 \leq & \
	 \mathbb{P}\Big(\sqrt{{npL}} \max_{j \neq i} (\hat{\theta}^d_i-\theta^*_i-(\hat{\theta}^d_j-\theta^*_j)) >c_W(\alpha,i) \Big)\\ 
	\leq&\ \alpha+o(1),
	\end{aligned}
	\end{equation}
where the last inequality holds by Theorem \ref{thm:bootstrap}.

Next, we analyze the power.
%
Similar to the  arguments above, since
	$$\max_{j \in {\T}^*_0}\hat{\theta}^d_j
	\leq 
	\max_{j \in {\T}^*_0}\theta^*_j
	+ \max_{j \in {\T}^*_0}(\hat{\theta}^d_j-\theta^*_j),$$
	we have 
	\begin{equation}
	\label{equ:tri-inequ2}
	\hat{\theta}^d_i - \max_{j \in {\T}^*_0}\hat{\theta}^d_j
	\geq 
	\hat{\theta}^d_i - \max_{j \in {\T}^*_0}\theta^*_j
	- \max_{j \in {\T}^*_0}(\hat{\theta}^d_j -\theta^*_j)
	= 
	\theta^*_i - \max_{j \in {\T}^*_0}\theta^*_j
	- \max_{j \in {\T}^*_0}(\hat{\theta}^d_j - \theta^*_j - (\hat{\theta}^d_i - \theta^*_i)).
	\end{equation}
	Also, by the definition of $\T_0^*$ and $\hat{\T}_0$, we have 
	$\max_{j \in {\T}^*_0}\hat{\theta}^d_j 
	\geq 
	\max_{j \in \hat{\T}_0}\hat{\theta}^d_j 
	$.
	Combining this with  \eqref{equ:tri-inequ2}, we have
	$$\hat{\theta}^d_i- \max_{j \in \hat{\T}_0}\hat{\theta}^d_j
	\geq 
	\theta^*_i- \max_{j \in {\T}^*_0}\theta^*_j
	- \max_{j \in {\T}^*_0}(\hat{\theta}^d_j-\theta^*_j-(\hat{\theta}^d_i-\theta^*_i)).
	$$
	Thus,  we have
	\begin{equation}\nonumber
	\begin{aligned}
	&\mathbb{P}\left(\text{reject } H_0\right)
	= \mathbb{P}\big(\hat{\theta}^d_i-c_W(\alpha,i)/\sqrt{{npL}} > \hat{\theta}^d_{(K+1)}\big)
	= \mathbb{P}\Big(\sqrt{{npL}}(\hat{\theta}^d_i- \max_{j \in \hat{\T}_0}\hat{\theta}^d_j)>c_W(\alpha,i)\Big)\\
	&\qquad \ge 
	\mathbb{P}\Big( \sqrt{{npL}} (\theta^*_i- \max_{j \in {\T}^*_0}\theta^*_j)>2\delta \text {, }  
	\sqrt{{npL}} \max_{j \in {\T}^*_0}(\hat{\theta}^d_j-\theta^*_j-(\hat{\theta}^d_i-\theta^*_i))  \leq \delta \text{ and } \delta \geq  c_W(\alpha,i)\Big)
	\\
	&\qquad \geq 
	\mathbb{P}\Big( \sqrt{{npL}} (\theta^*_i- \max_{j \in {\T}^*_0}\theta^*_j)>2\delta \text {, }  
	\sqrt{{npL}} \max_{j \neq i} (\hat{\theta}^d_j-\theta^*_j-(\hat{\theta}^d_i-\theta^*_i))  \leq \delta \text{ and } \delta \geq  c_W(\alpha,i)\Big).
	\end{aligned}	
	\end{equation}
	By Theorem \ref{thm:bootstrap}, for any fixed $\alpha$ and some sufficiently large $L, n$,
	$$\mathbb{P}(\sqrt{{npL}} \max_{j \neq i} (\hat{\theta}^d_j-\theta^*_j-\hat{\theta}^d_i+\theta^*_i)\geq  c_W(\alpha,i))>\alpha/2. $$ 
	Combining this with the fact that for sufficiently large constant $C$, we have $$\P(\sqrt{{npL}}\max_{j \neq i}  (\hat{\theta}^d_j-\theta^*_j-\hat{\theta}^d_i+\theta^*_i) 
	\leq 2\sqrt{{npL}} \|\hat{\theta}^d-\theta^*\|_\infty 
	\leq C\sqrt{\log n}) > 1- 1/n^5, $$ 
	we further have
	\begin{equation}\nonumber
	c_W(\alpha,i)\leq C \sqrt{\log n}.
	\end{equation}
Letting $\delta =  C\sqrt{\frac{\log n}{npL}}$,  it follows that 
	\begin{equation}\nonumber
	\P \big(\text{reject } H_0\big) = \P \Big(\sqrt{{npL}}(\hat{\theta}^d_i -  \max_{j \in \hat{\T}_0}\hat{\theta}^d_j) > c_W(\alpha, i) \Big) \geq 1 - n^{-5}.
	\end{equation}
	Hence,
	\begin{equation}\nonumber
	\inf_{ \theta^*_{i}-\theta^*_{(K+1)}>\delta} \mathbb{P}\big(\text{Reject } H_0\big) \to 1,
	\end{equation}
	which completes the proof.
\end{proof}	
%
%

\section{Proof in Section  \ref{sec:multi test}}
\label{sec:pf:multi test}

\subsection{Proof of Theorem \ref{thm:FWER}}
\label{sec:pf:FWER}
\begin{proof} We first show our procedure can control FWER with desired level asymptotically. 
Recall that in FWER control procedure, the set of all possible latent scores after perturbation $\tilde{\Theta}$ is 
\begin{equation}\nonumber
\tilde{\Theta}=
\{{\theta}:{\theta}_k \in \big[\hat{\theta}^d_k-C_M(\alpha, [n] \times [n])/\sqrt{npL},\ \hat{\theta}^d_k+C_M(\alpha, [n] \times [n])/\sqrt{npL}\big],
1 \leq k \leq n\}.
\end{equation} 
In specific, we consider $\tilde{\theta}^{(i)} \in \tilde{\Theta}$, $i=1, \cdots, n$,  whose $k$-th entry is, 
\begin{equation}\nonumber
\tilde{\theta}^{(i)}_k=
\left\{\begin{array}{ll}
\hat{\theta}^d_{k}+c_M(\alpha,[n] \times [n]), & k \in H(\gamma^*,i), \\
\hat{\theta}^d_i, &  k = i,\\
\hat{\theta}^d_{k}-c_M(\alpha,[n] \times [n]), & k \in L(\gamma^*,i).
\end{array}\right.
\end{equation}
By our FWER control procedure, we have, 
\begin{equation}\nonumber
\P (\text {making at least one Type I error}) 
= \P \big(\bigcup_{i: \gamma^* \notin \cR_i} \bigcap_{{\theta} \in \tilde{\Theta}}\{\gamma({{\theta}}) \in \cR_i\} \big)
\leq \P \big( 
\bigcup_{i: \gamma^* \notin \cR_i}
\{\gamma(\tilde{\theta}^{(i)}) \in \cR_i
\}
\big).
\end{equation}
Similar as the argument in \eqref{eqn:subseteq}, we have 
$$\big\{\gamma(\tilde{\theta}^{(i)}) \in \cR_i | \gamma^* \notin \cR_i\big\}  
\subseteq
\big\{\max _{k \neq i } \big(\hat{\theta}^d_{i}-\theta_{i}^{*}-\hat{\theta}^d_{k}+\theta_{k}^{*}\big)>c_M(\alpha,[n] \times [n])\big\},$$
which  gives us 

\begin{equation}\nonumber
\begin{aligned}
\P \big( 
\bigcup_{i: \gamma^* \notin \cR_i}
\{\gamma(\tilde{\theta}^{(i)}) \in \cR_i
\}
\big) 
&\leq 
\P \Big( 
\bigcup_{i: \gamma^* \notin \cR_i}
\big\{\max _{k \neq i}\big(\hat{\theta}^d_{i}-\theta_{i}^{*}-\hat{\theta}^d_{k}+\theta_{k}^{*}\big)>c_M(\alpha,[n] \times [n])
\big\}
\Big)\\
&\leq 
\P \Big( 
\max _{i: \gamma^* \notin \cR_i,k \neq i} \big(\hat{\theta}^d_{i}-\theta_{i}^{*}-\hat{\theta}^d_{k}+\theta_{k}^{*}\big)>c_M(\alpha,[n] \times [n])
\Big)\\
&\leq 
\P \Big( 
\max _{i\in[n],k \in[n]}\big(\hat{\theta}^d_{i}-\theta_{i}^{*}-\hat{\theta}^d_{k}+\theta_{k}^{*}\big)>c_M(\alpha,[n] \times [n])
\Big)\\
&\leq \alpha+o(1),
\end{aligned}
\end{equation}
where the last inequality holds by Theorem \ref{thm:bootstrap}, which completes the proof of controlling FWER.

Then, we consider the  Type II error. First, we have
\begin{equation}\nonumber
\begin{aligned}
\mathbb{P}(\text {At least one Type II error}) 
\leq \Sigma_{i=1}^{n}\mathbb{P}(\text {making Type II error in test  }i),
\end{aligned}
\end{equation}
where the test $i$ is $$H_{0}:  \gamma^* \notin \cR_i
\text{ \ \ v.s. \ \  } 
H_{a}: \gamma^* \in \cR_i.$$
Following similar arguments in \eqref{eqn:delta}, we 
let $\Delta({\theta}^*)>\delta=C \sqrt{  \log n}/\sqrt{npL}=C\sqrt{\frac{\log n}{npL}}$, and have
\begin{equation}\nonumber
\mathbb{P}(\text {Type II error in test  }i) \lesssim n^{-5}.
\end{equation}
Thus, we have 
\begin{equation}\nonumber
\begin{aligned}
\mathbb{P}(\text {At least one Type II error}) \lesssim n^{-4} \to 0,
\end{aligned}
\end{equation}
which completes the proof.
\end{proof}

\subsection{Proof of Theorem \ref{thm:FDR}}
\label{sec:pf:FDR}
\begin{proof}
We prove that our FDR procedure in Section \ref{sec:FDR}   controls the FDR asymptotically, in the sense that
 $\text{FDR} \leq \alpha + o(1)$ for any given $\alpha\in(0,1)$.  
By the definition of FDR, we have \begin{equation}\nonumber
  \begin{aligned}
	\text{FDR}
	&=\mathbb{E}\left[\frac{\sum_{i \in \mathcal{H}_{0}} \mathbb{I}\left[p_i \leq \frac{R }{n \cdot N}\alpha\right] \mathbb{I}\left[R>0\right]}{R}\right]
	=\sum_{i \in \mathcal{H}_{0}} \mathbb{E}\left[\frac{\mathbb{I}\left[p_i \leq \frac{R}{n \cdot N}\alpha\right] \mathbb{I}\left[R>0\right]}{R}\right],
	\end{aligned}
	\end{equation}
	where 
$|\mathcal{H}_{0}|$ is the number of true null hypotheses in \eqref{eqn:MT}, and $R$ is the total number of discoveries.
Since $1 / R=\sum_{k=R}^{\infty} \frac{1}{k(k+1)}=\sum_{k=1}^{\infty} \frac{\mathbb{I}\left[k \geq R\right]}{k(k+1)},$  we further have 
\begin{equation}\label{equ:fqr}
\begin{aligned}
\text{FDR}&=\sum_{k=1}^{\infty} \frac{1}{k(k+1)}\sum_{i \in \mathcal{H}_{0}} \mathbb{E}\bigg[\mathbb{I}\Big[p_i \leq \frac{R}{n \cdot N}\alpha\Big] \mathbb{I}\left[R>0\right]\mathbb{I}\left[k \geq R\right] \bigg]\\
&\leq \sum_{k=1}^{\infty} \frac{1}{k(k+1)}\sum_{i \in \mathcal{H}_{0}} \mathbb{P}\Big(p_i \leq \frac{\min(k,n)}{n \cdot N}\alpha\Big) .
\end{aligned}
\end{equation}
In addition, for any $0<\beta<1$ and $i \in \mathcal{H}_{0}$ (i.e., $\gamma^* \notin \cR_i$), we have the following inequality hold by  \eqref{eqn:subseteq},
\begin{equation}\label{equ:quantile}
\begin{aligned}
\mathbb{P}\left(p_i \leq \beta\right)
& \leq  \mathbb{P}\Big(\max_{k \neq i} \sqrt{{npL}} \big(\hat{\theta}^d_i-\theta^*_i-\hat{\theta}^d_k+\theta^*_k\big) >  c_W(\beta,i) \Big)\\
& \leq \beta+C\Big(\frac{1}{L^{c_3}}+\frac{2}{n^5}\Big)
\end{aligned}
\end{equation}
where the last inequality follows from Theorem \ref{thm:bootstrap} and Remark \ref{rmk:c3}, and $C>0$ is a sufficiently large constant. Plugging  \eqref{equ:quantile} into  \eqref{equ:fqr}, we have
\begin{equation}\nonumber
\begin{aligned}
\text{FDR} &\leq \sum_{k=1}^{\infty} \frac{1}{k(k+1)}\sum_{i \in \mathcal{H}_{0}}  \frac{\min(k,n)}{n \cdot N}\alpha+\sum_{k=1}^{\infty} \frac{1}{k(k+1)} \sum_{i \in \mathcal{H}_{0}} C\Big(\frac{1}{L^{c_3}}+\frac{2}{n^5}\Big)\\
&= \sum_{k=1}^{n} \frac{1}{k+1}\sum_{i \in \mathcal{H}_{0}}  \frac{1}{n \cdot N}\alpha
+ \sum_{k=n+1}^{\infty} \frac{1}{k(k+1)}\sum_{i \in \mathcal{H}_{0}}  \frac{1}{ \cdot N}\alpha+\sum_{i \in \mathcal{H}_{0}}C\Big(\frac{1}{L^{c_3}}+\frac{2}{n^5}\Big)\\
&= \sum_{i \in \mathcal{H}_{0}} \frac{\alpha}{n \cdot N} \Big(\sum_{k=1}^{n} \frac{1}{k+1} 
+ \sum_{k=n+1}^{\infty} \frac{n}{k(k+1)}\Big)+\sum_{i \in \mathcal{H}_{0}}C\Big(\frac{1}{L^{c_3}}+\frac{2}{n^5}\Big)\\
&\leq \frac{|\mathcal{H}_{0}|}{n}\alpha+C|\mathcal{H}_0|\Big(\frac{1}{L^{c_3}}+\frac{2}{n^5}\Big),
\end{aligned}
\end{equation}
where the last inequality follows from $\sum_{k=1}^{n} \frac{1}{k+1} 
+ \sum_{k=n+1}^{\infty} \frac{n}{k(k+1)} = \sum _{k=1}^{n} \frac {1}{k} = N$, and our claim holds as desired. 
\end{proof}

\section{Proofs of in Section~\ref{sec:LB_FWER}} \label{sec:pf:LB_FWER}

\subsection{Proof of Lemma \ref{lem:reduction}}
\label{sec:reduction}
\begin{proof}
	In this section, we prove 
	$$\mathfrak{R} = \inf_{\psi} \sup_{\theta^* \in \Xi} \P(d(\psi,\psi^*)\geq 1/2) \geq \mathbb{P}_{e,M}
	= \inf _{\phi} \max _{0 \leq i \leq M} \mathbb{P}_{i}(\phi \neq i), $$
	 where $\mathbb{P}_{i}$  is the probability measure under the $i$-th null hypothesis $\theta^*=\theta^{(i)}$, and $\Xi$, $d(\psi,\psi^*)$, $\phi$ are defined in \eqref{eqn:R}, \eqref{eqn:reduce} and \eqref{eqn:minimax prob error}, respectively. 
%

First, considering preference score vector ${\theta}^* \in \Xi$ within a finite set $\theta^* \in \{\theta^{(0)}, \cdots, \theta^{(M)}\}$, we have
$$
\inf_{\psi} \sup_{{\theta}^* \in \Xi} \P\big(d(\psi,\psi^*) \geq \frac{1}{2}\big)
\geq \inf_{\psi} \sup _{\theta^* \in \{\theta^{(0)}, \cdots, \theta^{(M)}\}} \P_i\big(d(\psi,\psi^{(i)})\geq \frac{1}{2}\big)
= \inf_{\psi} \sup_{ 0 \leq i \leq M} \P_i\big(d(\psi,\psi^{(i)})\geq \frac{1}{2}\big)
.$$ 
Then, for any selection procedure $\psi$, we have 
\begin{equation}\label{eqn:min distance test}
\P_i(d(\psi,\psi^{(i)}) \geq 1/2) 
\geq \P_i(\tilde{\phi} \neq i),
\end{equation} 
where $\tilde{\phi}: \bm{Y} \rightarrow\{0,1, \ldots, M\}$ is the  minimum distance test defined by 
$$\tilde{\phi}=\argmin _{0 \leq k \leq M} d(\psi, \psi^{(k)}).$$
Note that \eqref{eqn:min distance test} holds since
the event $\{\tilde{\phi} \neq i\}$ is equivalent to the event $\{\tilde{\phi} = k \text{ for some } k\neq i\}$. Due to the minimum distance nature of $\tilde{\phi}$, we have $d(\psi, \psi^{(k)}) \leq d(\psi, \psi^{(i)})$ for some $k$. We also have the triangle ineqaulity $1 \leq d(\psi^{(j)}, \psi^{(k)}) \leq d(\psi, \psi^{(j)})+d(\psi, \psi^{(k)})$ where $1 \leq d(\psi^{(j)}, \psi^{(k)})$ follows by our assumption. Combining the above inequalities, we  have $d(\psi, \psi^{(j)}) \geq 1/2$. 
Thus, we have
$$ \inf_{\psi} \sup_{\theta^* \in \Xi} \P(d(\psi,\psi^*) \geq 1/2)
\geq  \sup_{ 0 \leq i \leq M} \P_i(\tilde{\phi} \neq i) 
\geq \inf_{\phi} \sup_{0 \leq i \leq M} \P_i(\phi \neq i) 
= \mathbb{P}_{e,M},
$$
which completes the proof.
\end{proof}

\subsection{Proof of Lemma \ref{lem:fano}} \label{sec:fano}
\begin{proof}	
We prove 	$$
\bar{\P}_{\mathrm{e}, \mathcal{M}} 
\geq 
1-\frac{1}{\log (|\mathcal{M}|+1)}\bigg\{\frac{p L}{(|\mathcal{M}|+1)^{2}} \sum_{H,\tilde{H}\in \mathcal{H}_{\mathcal{M}}} \sum_{i < j} \mathrm{KL}\left(\mathbb{P}_{y_{i, j}^{(1)} \mid H} \| \mathbb{P}_{y_{i, j}^{(1)} \mid \tilde{H}}\right)+\log 2\bigg\},
$$ where the set of hypotheses  $\cH_{\cM} = \{H_0, H_1, \cdots, H_{|\mathcal{M}(\theta)|}\}$ is defined in \eqref{eqn:set hypo}.
 
Following the argument in the proof of Theorem 3 of \cite{chen2015spectral}, because of the partial observation in Erd\"{o}s-R\'{e}nyi graph, we introduce an erased version of observations that
$\boldsymbol{y}_{i, j}:=\left(y_{i, j}^{(1)}, \cdots, y_{i, j}^{(\ell)}\right), i,j \in [n]$ as
\begin{equation}\nonumber
\boldsymbol{z}_{i, j}=\left\{\begin{array}{ll}
\boldsymbol{y}_{i, j}, & \text { with probability } p \\
\text {0, } & \text { else }
\end{array}\right.
\end{equation}
and define the set $\boldsymbol{Z}:=\left\{\boldsymbol{z}_{i, j}\right\}_{1 \leq i<j \leq n}$. Then, we apply the generalized Fano inequality in \cite{verdu1994generalizing}, we have

\begin{equation}\nonumber
\begin{aligned}
\bar{\P}_{\mathrm{e}, \mathcal{M}} 
& \geq 1-\frac{1}{\log (|\mathcal{M}|+1)}\bigg\{\frac{1}{(|\mathcal{M}|+1)^{2}} \sum_{H,\tilde{H}\in \mathcal{H}_{\mathcal{M}}} \mathrm{KL}\left(\mathbb{P}_{\boldsymbol{Z} \mid H} \| \mathbb{P}_{\boldsymbol{Z} \mid \tilde{H}}\right)+\log 2\bigg\} \\
& = 1-\frac{1}{\log (|\mathcal{M}|+1)}\bigg\{\frac{1}{(|\mathcal{M}|+1)^{2}} \sum_{H,\tilde{H}\in \mathcal{H}_{\mathcal{M}}}\sum_{i < j} \mathrm{KL}\left(\mathbb{P}_{\bm{z}_{i, j} \mid H} \| \mathbb{P}_{\bm{z}_{i, j} \mid \tilde{H}}\right)+\log 2\bigg\} \\
&= 1-\frac{1}{\log (|\mathcal{M}|+1)}\bigg\{\frac{p}{(|\mathcal{M}|+1)^{2}} \sum_{H,\tilde{H}\in \mathcal{H}_{\mathcal{M}}} \sum_{i < j} \mathrm{KL}\left(\mathbb{P}_{\bm{y}_{i, j} \mid H} \| \mathbb{P}_{\bm{y}_{i, j}\mid \tilde{H}}\right)+\log 2\bigg\} \\
& = 1-\frac{1}{\log (|\mathcal{M}|+1)}\bigg\{\frac{p L}{(|\mathcal{M}|+1)^{2}} \sum_{H,\tilde{H}\in \mathcal{H}_{\mathcal{M}}} \sum_{i < j} \mathrm{KL}\left(\mathbb{P}_{y_{i, j}^{(1)} \mid H} \| \mathbb{P}_{y_{i, j}^{(1)} \mid \tilde{H}}\right)+\log 2\bigg\},
\end{aligned}
\end{equation}
where the first equality holds by
independence of the $\bz_{i,j}$ for different $(i,j)$, the second equality follows from the relation between $\bz_{i,j}$ and $\by_{i,j}$, and the third equality follows from i.i.d. assumption of $y_{ij}^{(\ell)}$ for different $\ell \in [L]$.
\end{proof}

\subsection{Proof of Lemma \ref{lem:KL}}\label{sec:KL}
\begin{proof}
We show
\begin{equation}\label{eqn:KL2}
	\sum_{H,\tilde{H}\in \mathcal{H}_{\mathcal{M}}} \sum_{i < j} \mathrm{KL}\left(\mathbb{P}_{y_{i, j}^{(1)} \mid H} \| \mathbb{P}_{y_{i, j}^{(1)} \mid \tilde{H}}\right)
	\leq 
	4n|\mathcal{M}|^2 \big(d^2_{\mathcal{M}}({{\omega}})+\cO(d^3_{\mathcal{M}}({{\omega}})) \big),
\end{equation}
where $d_{\mathcal{M}}({{\omega}})=\max_{(k_1, k_2) \in \mathcal{M}} \Big(\frac{\omega_{\min(k_1,k_2)}}{\omega_{\min(k_1,k_2)+1}}+\cdots+\frac{\omega_{\max(k_1,k_2)-1}}{\omega_{\max(k_1,k_2)}}-|k_1-k_2|\Big)$ and $\omega_1 > \cdots > \omega_n$.

We have the following Lemma to bound the KL divergence. Proof is provided in Section~\ref{sec:pf:KL-chi}.
\begin{lemma}\label{lem:KL-chi}
Assume we have two probability  measures $P \sim \text{Bern}(p)$ and $Q \sim \text{Bern}(q)$, and assume $p=\frac{\omega_{b}}{\omega_a+\omega_{b}}$ and $q=\frac{\omega_{d}}{\omega_{c}+\omega_d}$, where $a,b,c,d \in [n]$ and $\omega_1 > \cdots > \omega_n$. If $a<b$, $c<d$ or $a>b$, $c>d$, we have
\begin{small}
	\begin{equation}\nonumber
	\begin{aligned}
	\mathrm{KL}(P \| Q) \leq &  \frac{1}{4}\Big(\frac{\omega_{\min(a,b)}}{\omega_{\min(a,b)+1}}-1
	+\cdots+\frac{\omega_{\max(a,b)-1}}{\omega_{\max(a,b)}}-1-
	(\frac{\omega_{\min(c,d)}}{\omega_{\min(c,d)+1}}-1
	+\cdots+\frac{\omega_{\max(c,d)-1}}{\omega_{\max(c,d)}}-1)\Big)^2 \\
	+& \cO\Bigg(\Big(|\frac{\omega_{\min(a,b)}}{\omega_{\min(a,b)+1}}-1|
	+\cdots+|\frac{\omega_{\max(a,b)-1}}{\omega_{\max(a,b)}}-1|
	+ |\frac{\omega_{\min(c,d)}}{\omega_{\min(c,d)+1}}-1|+\cdots+|\frac{\omega_{\max(c,d)-1}}{\omega_{\max(c,d)}}-1|\Big)^3\Bigg). 
	\end{aligned}
	\end{equation}
\end{small}
Similarly, if $a>b$, $c<d$ or $a<b$, $c>d$, we have
\begin{small}
	\begin{equation}\nonumber
	\begin{aligned}
	\mathrm{KL}(P \| Q) \leq &  \frac{1}{4}\Big(\frac{\omega_{\min(a,b)}}{\omega_{\min(a,b)+1}}-1
	+\cdots+\frac{\omega_{\max(a,b)-1}}{\omega_{\max(a,b)}}-1 +
	(\frac{\omega_{\min(c,d)}}{\omega_{\min(c,d)+1}}-1
	+\cdots+\frac{\omega_{\max(c,d)-1}}{\omega_{\max(c,d)}}-1)\Big)^2 \\
	+& \cO\Bigg(\Big(|\frac{\omega_{\min(a,b)}}{\omega_{\min(a,b)+1}}-1|
	+\cdots+|\frac{\omega_{\max(a,b)-1}}{\omega_{\max(a,b)}}-1|
	+ |\frac{\omega_{\min(c,d)}}{\omega_{\min(c,d)+1}}-1|+\cdots+|\frac{\omega_{\max(c,d)-1}}{\omega_{\max(c,d)}}-1|\Big)^3\Bigg). 
	\end{aligned}
	\end{equation}
\end{small}

\end{lemma}

%
Then to prove \eqref{eqn:KL2}, we decompose the left hand side of  \eqref{eqn:KL2} as
	\begin{equation}\nonumber
	\begin{aligned}
&	\sum_{H,\tilde{H}\in \mathcal{H}_{\mathcal{M}}} \sum_{i < j} \mathrm{KL}\Big(\mathbb{P}_{y_{i, j}^{(1)} \mid H} \| \mathbb{P}_{y_{i, j}^{(1)} \mid \tilde{H}}\Big) \\
&\qquad	= \!\!\! \underbrace{\sum_{\substack{(k_1,k'_1) \in \mathcal{M}\\ (k_2,k'_2) \in \mathcal{M} }} \sum_{i < j} 
	\mathrm{KL}\Big(\mathbb{P}_{y_{i, j}^{(1)} \mid H_{(k_1,k'_1)}} \| \mathbb{P}_{y_{i, j}^{(1)} \mid H_{(k_2,k'_2)}}\Big)}_{J_1} \\
& \qquad\quad	 + \!\!\!\!
	\underbrace{\sum_{\substack{(k,k') \in \mathcal{M}}} \sum_{i < j} \Big(\mathrm{KL}\big(\mathbb{P}_{y_{i, j}^{(1)} \mid H_0} \| \mathbb{P}_{y_{i, j}^{(1)} \mid H_{(k,k')}}\big)
	\!\!+ 
	\mathrm{KL}\big(\mathbb{P}_{y_{i, j}^{(1)} \mid H_{(k,k')}} \| \mathbb{P}_{y_{i, j}^{(1)} \mid H_{0}}\big)
	\Big)}_{J_2}, 
	\end{aligned}
	\end{equation}
where ${H}_{(k,k')} \in \mathcal{H}_{\mathcal{M}}$ denotes the hypothesis with respect to $(k,k') \in \mathcal{M}$.
Then we control $J_1$ and $J_2$ separately. 
For with $J_1$, 
if neither of $i,j$ is equal to one of $k_1,k'_1,k_2,k'_2$ (i.e., $i,j \neq k_1,k'_1,k_2,k'_2$), we have that $\mathrm{KL}\Big(\mathbb{P}_{y_{i, j}^{(1)} \mid H_{(k_1,k'_1)}} \| \mathbb{P}_{y_{i, j}^{(1)} \mid H_{(k_2,k'_2)}}\Big)=0$.
Otherwise,  
we divide $i,j$ into the following cases. Let $p$ and $q$ be the parameters in Bernoulli distributions with respect to probability measures $\mathbb{P}_{y_{i, j}^{(1)} \mid H_{(k_1,k'_1)}}$ and $\mathbb{P}_{y_{i, j}^{(1)} \mid H_{(k_2,k'_2)}}$.
\begin{itemize}
	\item If $i=k_1 , j=k'_1$, we have $p=\frac{\omega_{k_1}}{\omega_{k'_1}+\omega_{k_1}}$ and $q=\frac{\omega_{k'_1}}{\omega_{k_1}+\omega_{k'_1}}$. Thus, by Lemma \ref{lem:KL-chi} we have
\begin{small}
	\begin{equation}\nonumber
\begin{aligned}
&\mathrm{KL}\Big(\mathbb{P}_{y_{i, j}^{(1)} \mid H_{(k_1,k'_1)}} \| \mathbb{P}_{y_{i, j}^{(1)} \mid H_{(k_2,k'_2)}}\Big) 
\leq  \frac{(p-q)^{2}}{q(1-q)}  \\
= &
\frac{1}{4}\Big(\frac{\omega_{\min(k_1,k'_1)}}{\omega_{\min(k_1,k'_1)+1}}-1 + \cdots+\frac{\omega_{\max(k_1,k'_1)-1}}{\omega_{\max(k_1,k'_1)}}-1
+ (\frac{\omega_{\min(k_1,k'_1)}}{\omega_{\min(k_1,k'_1)+1}}-1 + \cdots+\frac{\omega_{\max(k_1,k'_1)-1}}{\omega_{\max(k_1,k'_1)}}-1)\Big)^2 \\
&+ \cO\bigg(\Big(|\frac{\omega_{\min(k_1,k'_1)}}{\omega_{\min(k_1,k'_1)+1}}-1|
+\cdots+|\frac{\omega_{\max(k_1,k'_1)-1}}{\omega_{\max(k_1,k'_1)}}-1|
+ |\frac{\omega_{\min(k_1,k'_1)}}{\omega_{\min(k_1,k'_1)+1}}-1|
+\cdots+|\frac{\omega_{\max(k_1,k'_1)-1}}{\omega_{\max(k_1,k'_1)}}-1|\Big)^3\bigg)\\
= &
\Big(\frac{\omega_{\min(k_1,k'_1)}}{\omega_{\min(k_1,k'_1)+1}}-1 + \cdots+\frac{\omega_{\max(k_1,k'_1)-1}}{\omega_{\max(k_1,k'_1)}}-1\Big)^2 
+ \cO\bigg(\Big(|\frac{\omega_{\min(k_1,k'_1)}}{\omega_{\min(k_1,k'_1)+1}}-1|
+\cdots+|\frac{\omega_{\max(k_1,k'_1)-1}}{\omega_{\max(k_1,k'_1)}}-1|
\Big)^3\bigg)\\
\leq & d^2_{\mathcal{M}}({{\omega}}) + \cO(d^3_{\mathcal{M}}({{\omega}})) 
\end{aligned}
\end{equation}
\end{small}
Note that first, for special case that $k_1=k_2$ or $k'_1=k'_2$ or $k_1=k'_2$ or $k'_1=k_2$, the upper bound $d^2_{\mathcal{M}}({{\omega}}) + \cO(d^3_{\mathcal{M}}({{\omega}}))$ also holds. 
Second, by Remark \ref{rmk:symmetry}, 
the cases $(i=k_2, j=k'_2), (i=k'_1, j=k_1), (i=k'_2, j=k_2)$ all have the same upper bound. 
	
	\item If $i=k_1 , j=k'_2$, we have $p=\frac{\omega_{k'_2}}{\omega_{k'_1}+\omega_{k'_2}}$ and $q=\frac{\omega_{k_2}}{\omega_{k_1}+\omega_{k_2}}$. Similarly, we also have 
	$$\mathrm{KL}\Big(\mathbb{P}_{y_{i, j}^{(1)} \mid H_{(k_1,k'_1)}} \| \mathbb{P}_{y_{i, j}^{(1)} \mid H_{(k_2,k'_2)}}\Big) \leq \frac{(p-q)^{2}}{q(1-q)}
	\leq d^2_{\mathcal{M}}({{\omega}})+\cO(d^3_{\mathcal{M}}({{\omega}})).$$
	Also,  $(i=k_2 , j=k'_1)$, $(i=k'_1 , j=k_2)$, $(i=k'_2, j=k_1)$, $(i=k_1 , j=k_2)$, $(i=k'_1 , j=k'_2)$, $(i=k_2 , j=k_1)$, $(i=k'_2 , j=k'_1)$ all have the same upper bound.
	
	\item if $i=[n]/\{k_1,k'_1,k_2,k'_2\}, j=k'_1$, we have $p=\frac{\omega_{k_1}}{\omega_{i}+\omega_{k_1}}$ and $q=\frac{\omega_{k'_1}}{\omega_i+\omega_{k'_1}}$. Similarly, we  have 
	$$\mathrm{KL}\Big(\mathbb{P}_{y_{i, j}^{(1)} \mid H_{(k_1,k'_1)}} \| \mathbb{P}_{y_{i, j}^{(1)} \mid H_{(k_2,k'_2)}}\Big) \leq \frac{(p-q)^{2}}{q(1-q)}=d^2_{\mathcal{M}}({{\omega}})+\cO(d^3_{\mathcal{M}}({{\omega}})). $$
	Similarly, cases $(i=k_2 , j=[n]/\{k_1,k'_1,k_2,k'_2\}), 
	(i=k'_1 ,  j=[n]/\{k_1,k'_1,k_2,k'_2\}), 
	(i=k'_2 ,  j=[n]/\{k_1,k'_1,k_2,k'_2\}), 
	(i=[n]/\{k_1,k'_1,k_2,k'_2\}, j=k_1),  
	(i=[n]/\{k_1,k'_1,k_2,k'_2\}, j=k'_1),  
	(i=[n]/\{k_1,k'_1,k_2,k'_2\}, j=k_2),  
	(i=[n]/\{k_1,k'_1,k_2,k'_2\}, j=k'_2)$ all have the same upper bound.
	
\end{itemize}

In summary, we have, if $i$ or $j$ equals to one of $k_1,k'_1,k_2,k'_2$, we have 
$$\mathrm{KL}\left(\mathbb{P}_{y_{i, j}^{(1)} \mid H_{(k_1,k'_1)}} \| \mathbb{P}_{y_{i, j}^{(1)} \mid H_{(k_2,k'_2)}}\right)
\leq
d^2_{\mathcal{M}}({{\omega}})+\cO(d^3_{\mathcal{M}}({{\omega}})).$$
Thus, we have $$J_1 
=\sum_{\substack{(k_1,k'_1) \in \mathcal{M}\\ (k_2,k'_2) \in \mathcal{M} }} \sum_{i < j} 
\mathrm{KL}\left(\mathbb{P}_{y_{i, j}^{(1)} \mid H_{(k_1,k'_1)}} \| \mathbb{P}_{y_{i, j}^{(1)} \mid H_{(k_2,k'_2)}}\right)
\leq 4n|\mathcal{M}|(|\mathcal{M}|-1)(d^2_{\mathcal{M}}({{\omega}})+\cO(d^3_{\mathcal{M}}({{\omega}}))),$$
where $4n$ comes from the sum over $i<j$ and one of $i,j$ is equal to one of $k_1,k'_1,k_2,k'_2$, and $|\mathcal{M}|(|\mathcal{M}|-1)$ comes from the summation over $(k_1,k'_1) \in \mathcal{M}, (k_2,k'_2) \in \mathcal{M} $ and $(k_1,k'_1) \neq (k_2,k'_2)$.

We next bound $J_2$. If $i$ or $j$ equals $k$ or $k'$, by similar arguments as above, we have
$$\mathrm{KL}\left(\mathbb{P}_{y_{i, j}^{(1)} \mid H_0} \| \mathbb{P}_{y_{i, j}^{(1)} \mid H_{(k,k')}}\right)
\leq 
d^2_{\mathcal{M}}({{\omega}})+\cO(d^3_{\mathcal{M}}({{\omega}})).$$
Hence, 
\begin{equation}\nonumber
\begin{aligned}
J_2 &=\sum_{\substack{(k,k') \in \mathcal{M}}} \sum_{i < j} \Big(\mathrm{KL}\left(\mathbb{P}_{y_{i, j}^{(1)} \mid H_0} \| \mathbb{P}_{y_{i, j}^{(1)} \mid H_{(k,k')}}\right)
+
\mathrm{KL}\left(\mathbb{P}_{y_{i, j}^{(1)} \mid H_{(k,k')}} \| \mathbb{P}_{y_{i, j}^{(1)} \mid H_{0}}\right)
\Big)\\
&\leq 4n|\mathcal{M}|(d^2_{\mathcal{M}}({{\omega}})+\cO(d^3_{\mathcal{M}}({{\omega}}))).
\end{aligned}
\end{equation}
Combining $J_1$ and $J_2$ together, we have
\begin{small}
	\begin{equation}\nonumber
	\begin{aligned}
	&\sum_{H,\tilde{H}\in \mathcal{H}_{\mathcal{M}}} \sum_{i < j} \mathrm{KL}\left(\mathbb{P}_{y_{i, j}^{(1)} \mid H} \| \mathbb{P}_{y_{i, j}^{(1)} \mid \tilde{H}}\right) \\
	&\qquad \leq 
	4n|\mathcal{M}|(|\mathcal{M}|-1) \big(d^2_{\mathcal{M}}({{\omega}})+\cO(d^3_{\mathcal{M}}({{\omega}})) \big)
	+ 
	4n|\mathcal{M}| \big(d^2_{\mathcal{M}}({{\omega}})+\cO(d^3_{\mathcal{M}}({{\omega}})) \big)\\
	&\qquad \leq  4n|\mathcal{M}|^2 \big(d^2_{\mathcal{M}}({{\omega}})+\cO(d^3_{\mathcal{M}}({{\omega}})) \big),
	\end{aligned}
	\end{equation}
\end{small}
which completes the proof. 
\end{proof}

\subsection{Proof of Lemma \ref{lem:necessary condition}}
\label{sec:necessary condition}
\begin{proof}
We show that if
\begin{equation}\nonumber
\Delta({\theta}) 
\lesssim \sqrt{\frac{\epsilon \log(|\mathcal{M}|+1)-\log 2}{npL}},
\end{equation}
we have	
\begin{equation}\nonumber
d_\mathcal{M}({{\omega}}) 
\lesssim \sqrt{\frac{\epsilon \log(|\mathcal{M}|+1)-\log 2}{npL}},
\end{equation}
where $\theta$ ($\theta_{1} > \theta_{2} > \cdots > \theta_{n}$) is chosen in \eqref{eqn:step2} and $\omega=\exp(\theta)$. 

Let $\delta_i:=\theta_{i-1}-\theta_i \geq 0$  for $1 \leq i \leq n$ where $\theta_0=\theta_1$. Since $\omega=\exp(\theta)$, we have 
$\Delta_{i}:=\frac{\omega_{i-1}}{\omega_i} = \exp (\delta_i) \geq 1$. 
Recall that 
$$d_{\mathcal{M}}({{\omega}})=\max_{(k_1, k_2) \in \mathcal{M}} \Big(\frac{\omega_{\min(k_1,k_2)}}{\omega_{\min(k_1,k_2)+1}}+\cdots+\frac{\omega_{\max(k_1,k_2)-1}}{\omega_{\max(k_1,k_2)}}-|k_1-k_2|\Big).$$ 
WLOG, we assume $d_{\mathcal{M}}({{\omega}})$ is achieved by a pair $(k_1, k_2) \in \mathcal{M}$ and $k_1<k_2$. We then have
$$d_{\mathcal{M}}({{\omega}})= \frac{\omega_{\min(k_1,k_2)}}{\omega_{\min(k_1,k_2)+1}}+\cdots+\frac{\omega_{\max(k_1,k_2)-1}}{\omega_{\max(k_1,k_2)}}-|k_1-k_2|
=  \Delta_{k_1+1} - 1 + \cdots + \Delta_{k_2} - 1.$$ 
Accordingly we have  $$\Delta({\theta})
=|\theta_{k_1} - \theta_{k_2}|
=\delta_{k_1+1} + \cdots + \delta_{k_2}
=\log \Delta_{k_1+1} + \cdots + \log \Delta_{k_2}.$$ 
By the inequality $x - 1 \leq 2\log(x)$ for $1 \leq x \leq 3$, we have $\Delta_{k_1+1} - 1 + \cdots + \Delta_{k_2} - 1 \leq 2(\log \Delta_{k_1+1} + \cdots + \log \Delta_{k_2})$ if $\Delta_{k_1+1}, \cdots, \Delta_{k_2} \leq 3$.
Meanwhile, we have $\Delta({\theta}) \lesssim \sqrt{\frac{\epsilon \log(|\mathcal{M}|+1)-\log 2}{npL}} \lesssim \sqrt{\frac{1}{L}}$ where the last inequality is achieved by the fact $|\mathcal{M}|\leq n^2$ and our assumption that $p \gtrsim \frac{\log n}{n}$. 
We have 
$d_\mathcal{M}({{\omega}}) \lesssim \Delta({\theta})
 \lesssim \sqrt{\frac{\epsilon \log(|\mathcal{M}|+1)-\log 2}{npL}}$, which concludes the proof. 
\end{proof}

\subsection{Proof of Lemma \ref{lem:KL-chi}}\label{sec:pf:KL-chi}
\begin{proof}
	We first provide a bound of the KL divergence by chi-square divergence. Since $P \sim \text{Bern}(p)$ and $Q \sim \text{Bern}(q)$,
	one have 
	\begin{equation}\nonumber
	\mathrm{KL}(P \| Q) \leq \chi^{2}(P \| Q)=\frac{(p-q)^{2}}{q}+\frac{(p-q)^{2}}{1-q}=\frac{(p-q)^{2}}{q(1-q)}.
	\end{equation} 
	
	By the condition that $p=\frac{\omega_{b}}{\omega_a+\omega_{b}}$ and $q=\frac{\omega_{d}}{\omega_{c}+\omega_d}$, we first expand $\frac{(p-q)^{2}}{q(1-q)}$ assuming $a<b$ and $c<d$ (so $\frac{\omega_a}{\omega_{b}} = \frac{\omega_a}{\omega_{a+1}}\times\cdots\times\frac{\omega_{b-1}}{\omega_{b}}$), and then expand $\frac{(p-q)^{2}}{q(1-q)}$ without assumption on $a,b,c,d$. 
	
	In particular, assuming $a<b$ and $c<d$, by Taylor expansion we have 
	\begin{equation}\nonumber
	\begin{aligned}
	\frac{(p-q)^{2}}{q(1-q)} 
	= &\  \frac{\big(\frac{1}{\frac{\omega_a}{\omega_{b}}+1}-\frac{1}{\frac{\omega_c}{\omega_{d}}+1}\big)^{2}}
	{\frac{1}{\frac{\omega_c}{\omega_{d}}+1}
		\big(1-\frac{1}{\frac{\omega_c}{\omega_{d}}+1}\big)} 
	=  \frac{\Big(\frac{1}{\frac{\omega_a}{\omega_{a+1}}\times\cdots\times\frac{\omega_{b-1}}{\omega_{b}}+1}-\frac{1}{\frac{\omega_c}{\omega_{c+1}}\times\cdots\times\frac{\omega_{d-1}}{\omega_{d}}+1}\Big)^{2}}
	{\frac{1}{\frac{\omega_c}{\omega_{c+1}}\times\cdots\times\frac{\omega_{d-1}}{\omega_{d}}+1}
		\Big(1-\frac{1}{\frac{\omega_c}{\omega_{c+1}}\times\cdots\times\frac{\omega_{d-1}}{\omega_{d}}+1}\Big)} \\
	=&\  \frac{1}{4}\Big(\frac{\omega_a}{\omega_{a+1}}-1
	+\cdots+\frac{\omega_{b-1}}{\omega_{b}}-1-(\frac{\omega_c}{\omega_{c+1}}-1+\cdots+\frac{\omega_{d-1}}{\omega_{d}}-1)\Big)^2 \\
	&+ \cO\Bigg(\Big(|\frac{\omega_a}{\omega_{a+1}}-1|
	+\cdots+|\frac{\omega_{b-1}}{\omega_{b}}-1|
	+ |\frac{\omega_c}{\omega_{c+1}}-1|+\cdots+|\frac{\omega_{d-1}}{\omega_{d}}-1|\Big)^3\Bigg). 
	\end{aligned}
	\end{equation}
	Generally, if $a<b$, $c<d$ or $a>b$, $c>d$, we have
	\begin{small}
		\begin{equation}\nonumber
		\begin{aligned}
		\frac{(p-q)^{2}}{q(1-q)} 
		= &  \frac{1}{4}\Big(\frac{\omega_{\min(a,b)}}{\omega_{\min(a,b)+1}}-1
		+\cdots+\frac{\omega_{\max(a,b)-1}}{\omega_{\max(a,b)}}-1-
		(\frac{\omega_{\min(c,d)}}{\omega_{\min(c,d)+1}}-1
		+\cdots+\frac{\omega_{\max(c,d)-1}}{\omega_{\max(c,d)}}-1)\Big)^2 \\
		+& \cO\Bigg(\Big(|\frac{\omega_{\min(a,b)}}{\omega_{\min(a,b)+1}}-1|
		+\cdots+|\frac{\omega_{\max(a,b)-1}}{\omega_{\max(a,b)}}-1|
		+ |\frac{\omega_{\min(c,d)}}{\omega_{\min(c,d)+1}}-1|+\cdots+|\frac{\omega_{\max(c,d)-1}}{\omega_{\max(c,d)}}-1|\Big)^3\Bigg). 
		\end{aligned}
		\end{equation}
	\end{small}
	Similarly, if $a>b$, $c<d$ or $a<b$, $c>d$, we have
	\begin{small}
		\begin{equation}\nonumber
		\begin{aligned}
		\frac{(p-q)^{2}}{q(1-q)} 
		= &  \frac{1}{4}\Big(\frac{\omega_{\min(a,b)}}{\omega_{\min(a,b)+1}}-1
		+\cdots+\frac{\omega_{\max(a,b)-1}}{\omega_{\max(a,b)}}-1 +
		(\frac{\omega_{\min(c,d)}}{\omega_{\min(c,d)+1}}-1
		+\cdots+\frac{\omega_{\max(c,d)-1}}{\omega_{\max(c,d)}}-1)\Big)^2 \\
		+& \cO\Bigg(\Big(|\frac{\omega_{\min(a,b)}}{\omega_{\min(a,b)+1}}-1|
		+\cdots+|\frac{\omega_{\max(a,b)-1}}{\omega_{\max(a,b)}}-1|
		+ |\frac{\omega_{\min(c,d)}}{\omega_{\min(c,d)+1}}-1|+\cdots+|\frac{\omega_{\max(c,d)-1}}{\omega_{\max(c,d)}}-1|\Big)^3\Bigg). 
		\end{aligned}
		\end{equation}
	\end{small}
	
\end{proof}

\begin{remark}\label{rmk:symmetry}
	Based on the above argument, we further have 
	\begin{itemize}
		\item $\frac{(p-q)^{2}}{q(1-q)}$ and $\frac{(p-q)^{2}}{p(1-p)}$ have the same closed-form of expansion, which means $\mathrm{KL}(P \| Q)$ and $\mathrm{KL}(Q \| P)$ (i.e., $\mathrm{KL}\Big(\mathbb{P}_{y_{i, j}^{(1)} \mid H} \| \mathbb{P}_{y_{i, j}^{(1)} \mid \tilde{H}}\Big)$ and $\mathrm{KL}\Big(\mathbb{P}_{y_{i, j}^{(1)} \mid \tilde{H}} \| \mathbb{P}_{y_{i, j}^{(1)} \mid {H}}\Big)$) have the same upper bound,
		\item $\frac{(p-q)^{2}}{q(1-q)} = \frac{((1-p)-(1-q))^{2}}{(1-q)q}$, which means $\mathrm{KL}\Big(\mathbb{P}_{y_{i, j}^{(1)} \mid H} \| \mathbb{P}_{y_{i, j}^{(1)} \mid \tilde{H}}\Big)$ and $\mathrm{KL}\Big(\mathbb{P}_{y_{j,i}^{(1)} \mid H} \| \mathbb{P}_{y_{j,i}^{(1)} \mid \tilde{H}}\Big)$ have the same upper bound.
	\end{itemize}
\end{remark}

\section{Proofs of Auxiliary Lemmas in Section~\ref{sec:auciliary lem - 2} }
\label{sec:auciliary lem - 3}

Before proving the auxiliary lemmas in Section~\ref{sec:auciliary lem - 2}, we first present the following lemma for self-completeness, which is used throughout the proof in this section.

\begin{lemma}[Degree concentration, Lemma 1  in \cite{chen2019spectral}]\label{lem:degree_concentration}
	
	Let $d_i$ be the degree of node $i$ in Erd\"{o}s-R\'{e}nyi graph $\mathcal{G}(n,p)$ where $p \geq C_0 \frac{\log n}{n}$ for sufficiently large constant $C_0>0$, we have \begin{equation}\nonumber
	\frac{1}{2}np\leq d_{\min}\leq d_{\max} \leq \frac{3}{2}np
	\end{equation} 
	is satisfied with probability at least $1-\cO ( n^{-10})$, where $d_{\min}=\min_{1\leq i \leq n} d_i$
	and $d_{\max}=\max_{1\leq i \leq n} d_i$.

\end{lemma}

\subsection{Proof of Lemma \ref{lem:consistency}}
\label{sec:pf:consistency}	
\begin{proof}
	For the proof that  $
\|\hat{{\theta}}-{{\theta}^*}\|_\infty\lesssim \sqrt{\frac{\log n}{npL}}$ for regularized MLE, see Section~6  in \cite{chen2019spectral} for the detailed proof.

For the proof of the spectral method, \cite{chen2019spectral} show that 
\begin{equation}\nonumber
\frac{\left\|{\pi}-{\pi}^{*}\right\|_{\infty}}
{\left\|{\pi}^{*}\right\|_{\infty}} 
\lesssim \sqrt{\frac{\log n}{n p L}}
\end{equation}
with probability at least $1-\cO ( n^{-5})$ where $\bm{1}^\top{\pi}^{*}=1$ and  $\bm{1}^\top{\pi}=1$ obtained from spectral method. 

We then show that the same rate can be achieved by $\hat{{\theta}}=(\mathbf{I}-\mathbf{1}\mathbf{1}^{T} / n)\log \pi$ under the constraint parameter set $\mathcal{C}=\{{\theta}:\bm{1}^\top{\theta}=0\}$, where the projection matrix $\mathbf{I}-\mathbf{1}\mathbf{1}^{T} / n$ project $\log \pi$ onto the constraint parameter space, that
\begin{equation}\nonumber
\|\hat{{\theta}}-{{\theta}^*}\|_\infty\lesssim \sqrt{\frac{\log n}{npL}}.
\end{equation} 
Letting ${\eta}=\log({\pi})$ and ${\eta}^*=\log({\pi}^*)$, we first show that $\|\eta-{\eta}^*\|_\infty\lesssim \sqrt{\frac{\log n}{npL}}$. In particular, we have 
\begin{equation}\nonumber
\begin{aligned}
\frac{\left\|{\pi}-{\pi}^{*}\right\|_{\infty}}
{\left\|{\pi}^{*}\right\|_{\infty}} 
=  \frac{\left\|e^{{\eta}}-e^{{\eta}^*}\right\|_{\infty}}
{\left\|e^{{\eta}^*}\right\|_{\infty}} 
=  
\frac{\max_{k}|e^{{\eta}_k}-e^{{\eta}_k^*}|}
{e^{{\eta}^*_{\max}}} 
=  
\frac{\max_{k}|e^{c_k}({\eta}_k-{\eta}_k^*)|}
{e^{{\eta}^*_{\max}}},
\end{aligned}
\end{equation}
where $c_k$ is some real value betwen ${\eta}_k$ and ${\eta}_k^*$.
Thus, we have $e^{c_k} \geq e^{\eta_{\min}^*-\left\|{{\eta}}-{{\eta}^*}\right\|_{\infty}}$ for any $k$, which immediately gives us 
\begin{equation}\nonumber
\begin{aligned}
\frac{\max_{k}|e^{c_k}({\eta}_k-{\eta}_k^*|)}
{e^{{\eta}^*_{\max}}}
\geq  \frac{\max_{k}|e^{\eta_{\min}^*-\left\|{{\eta}}-{{\eta}^*}\right\|_{\infty}}({\eta}_k-{\eta}_k^*)|}
{e^{{\eta}^*_{\max}}} 
\geq  \frac{e^{\eta_{\min}^*-\left\|{{\eta}}-{{\eta}^*}\right\|_{\infty}}\left\|{{\eta}}-{{\eta}^*}\right\|_{\infty}}
{e^{{\eta}^*_{\max}}}
\geq 
\frac{1}{\kappa}
\frac{\left\|{{\eta}}-{{\eta}^*}\right\|_{\infty}}
{e^{\left\|{{\eta}}-{{\eta}^*}\right\|_{\infty}}},
\end{aligned}
\end{equation}
where the last inequality comes from $\frac{e^{\eta_{\min}^*}}{e^{{\eta}^*_{\max}}} 
\geq \frac{\omega_{\min}}{\omega_{\max}} 
= \frac{1}{\kappa}$. 
Combining the above inequalities,  we have 
\begin{equation}\nonumber
\begin{aligned}
\frac{\left\|{{\eta}}-{{\eta}^*}\right\|_{\infty}}
{e^{\left\|{{\eta}}-{{\eta}^*}\right\|_{\infty}}}
\lesssim \frac{\left\|{\pi}-{\pi}^{*}\right\|_{\infty}}
{\left\|{\pi}^{*}\right\|_{\infty}}
\lesssim 
\sqrt{\frac{\log n}{n p L}},
\end{aligned}
\end{equation}
which immediately gives   $\left\|{{\eta}}-{{\eta}^*}\right\|_{\infty} 
\lesssim 
\sqrt{\frac{\log n}{n p L}}$. 
Projecting ${{\eta}}$ and ${{\eta}^*}$ onto constrained parameter space, we finally have 
$$\|\hat{{\theta}}-{{\theta}^*}\|_\infty
=\|(\mathbf{I}-\mathbf{1}\mathbf{1}^{T} / n) ({{\eta}}-{{\eta}^*})\|_\infty
\lesssim \|{{\eta}}-{{\eta}^*}\|_\infty
\lesssim \sqrt{\frac{\log n}{npL}},$$
which completes the proof.
\end{proof}

\subsection{Proof of Lemma \ref{lem:gradient}}
\label{sec:pf:gradient}
\begin{proof}
	We aim to show that, with probability exceeding $1-\cO ( n^{-6})$, 
\begin{equation}\nonumber	\label{equ:gradient}
\|{\nabla \cL({{\theta}}^*)}\|_\infty
=\Big\|\sum_{i>j}{\mathcal{E}_{ji}}(-y_{ji}+\frac{e^{\theta_i^*}}{e^{\theta_i^*}+e^{\theta_j^*}})(\bm{e}_i-\bm{e}_j)\Big\|_\infty
\lesssim np \sqrt{\frac{\log n}{L}}.
\end{equation}
First, note that $ y_{ji}=\sum_{\ell=1}^{L}y_{ji}^{(\ell)}/L$. 
By Hoeffding's inequality, we have $$\mathbb{P}\Big(\Big| -y_{ji}+\frac{e^{\theta_i^*}}{e^{\theta_i^*}+e^{\theta_j^*}}\Big| \geq t \Big)\leq 2e^{-2t^2L},$$ which immediately gives $ \left|-y_{ji}+\frac{e^{\theta_i^*}}{e^{\theta_i^*}+e^{\theta_j^*}}\right|\leq 2\sqrt{\frac{\log n}{L}}$  with probability at least $1-\cO( n^{-8})$ by taking $t=2\sqrt{\frac{\log n}{L}}$. 
Thus, we have, with probability larger than $1-\cO ( n^{-6})$,
$$ \max_{i,j}\Big|-y_{ji}+\frac{e^{\theta_i^*}}{e^{\theta_i^*}+e^{\theta_j^*}}\Big|\lesssim \sqrt{\frac{\log n}{L}}.
$$
Finally by  Lemma \ref{lem:degree_concentration}, we have
\begin{equation}\nonumber
\|{\nabla \cL({{\theta}}^*)}\|_\infty
\lesssim \max_{i,j}
\Big|-y_{ji}
+
\frac{e^{\theta_i^*}}{e^{\theta_i^*}+e^{\theta_j^*}}\Big|
\Bigl\|\sum_{i>j}{\mathcal{E}_{ji}}(\bm{e}_i-\bm{e}_j)\Big\|_\infty
\lesssim np \sqrt{\frac{\log n}{L}},
\end{equation}
which concludes the proof.
\end{proof}

\subsection{Proof of Lemma \ref{lem:smoothness} }
\label{sec:pf:smoothness}
\begin{proof}
We aim to show that, with probability exceeding $1-\cO ( n^{-5})$, 
\begin{equation}\label{equ:smoothness}
\|\nabla \cL(\hat{{\theta}})-\nabla \cL({{\theta}}^*)
- \nabla^2 \cL({\theta}^*)(\hat{{\theta}}-{\theta}^*)\|_\infty	\lesssim \frac{\log n}{L}.
\end{equation}
Consider the term $	\nabla \cL(\hat{{\theta}})-\nabla \cL({{\theta}}^*)$. 
%
%
By the mean-value theorem, we have 
\begin{equation}\label{eqn:grad_conc}
\begin{aligned}
\nabla \cL(\hat{{\theta}})-\nabla \cL({{\theta}}^*)
&=\sum_{i>j}{\mathcal{E}_{ji}}\left(\frac{e^{\hat{\theta}_i}}{e^{\hat{\theta}_i}+e^{\hat{\theta}_j}}-\frac{e^{\theta_i^*}}{e^{\theta_i^*}+e^{\theta_j^*}}\right) (\bm{e}_i-\bm{e}_j) \\
&=-\sum_{i>j}{\mathcal{E}_{ji}} \frac{e^{c_{ji}}}{\left(1+e^{c_{ji}}\right)^{2}}\left[\hat{\theta}_j-\hat{\theta}_i -\left(\theta_{j}^{*}-\theta_{i}^{*}\right)\right](\bm{e}_i-\bm{e}_j)\\
&=\sum_{i>j}{\mathcal{E}_{ji}}\frac{e^{c_{ji}}}{\left(1+e^{c_{ji}}\right)^{2}}(\bm{e}_i-\bm{e}_j)(\bm{e}_i-\bm{e}_j)^\top(\hat{{\theta}}-{\theta}^*),
\end{aligned}
\end{equation}
where $c_{ji}$ is some real number  between $\hat{\theta}_j-\hat{\theta}_i$ and $\theta_{j}^{*}-\theta_{i}^{*}$.
Substituting  \eqref{eqn:grad_conc} back into  \eqref{equ:smoothness}, and by the definition of $\nabla^2\cL(\theta^*)$ in \eqref{eqn:hess}, we have 
\begin{equation}\nonumber
\begin{aligned}
& \|\nabla \cL(\hat{{\theta}})-\nabla \cL({{\theta}}^*)- \nabla^2 \cL({{\theta}}^*)(\hat{{\theta}}-{\theta}^*)\|_\infty \\
\leq & \Big\|\sum_{i>j}{\mathcal{E}_{ji}} \bigg(\frac{e^{c_{ji}}}{\left(1+e^{c_{ji}}\right)^{2}}
-\frac{e^{{\theta}^*_j-{\theta}^*_i}} {(1+e^{{\theta}^*_j-{\theta}^*_i} )^{2}} \bigg)
(\bm{e}_i-\bm{e}_j)(\bm{e}_i-\bm{e}_j)^\top\Big\|_\infty \|\hat{{\theta}}-{\theta}^*\|_\infty.
\end{aligned}
\end{equation}
Then, we  bound $\Big\|\sum_{i>j}{\mathcal{E}_{ji}} \bigg(\frac{e^{c_{ji}}}{\left(1+e^{c_{ji}}\right)^{2}}
-\frac{e^{{\theta}^*_j-{\theta}^*_i}} {(1+e^{{\theta}^*_j-{\theta}^*_i} )^{2}} \bigg)
(\bm{e}_i-\bm{e}_j)(\bm{e}_i-\bm{e}_j)^\top\Big\|_\infty$. 
Denote the term by $$M=\sum_{i>j}{\mathcal{E}_{ji}}\left(\frac{e^{c_{ji}}}{\left(1+e^{c_{ji}}\right)^{2}}
-\frac{e^{{\theta}^*_j-{\theta}^*_i}} {(1+e^{{\theta}^*_j-{\theta}^*_i} )^{2}}\right)(\bm{e}_i-\bm{e}_j)(\bm{e}_i-\bm{e}_j)^\top.$$
Since  
$\left|\frac{e^{c_{ji}}}{\left(1+e^{c_{ji}}\right)^{2}}
-\frac{e^{{\theta}^*_j-{\theta}^*_i}} {(1+e^{{\theta}^*_j-{\theta}^*_i} )^{2}}\right| 
\leq \frac{1}{4}|c_{ji}-({\theta}^*_j-{\theta}^*_i)|
\lesssim \|\hat{{\theta}}-{\theta}^*\| _\infty$ by Lipschitz continuity,  the $k$-th diagonal entry  of $M$ is bounded by 
\begin{equation}\nonumber
\begin{aligned}
|M_{kk}|=&\ \bigg|\sum_{k>j}{\mathcal{E}_{jk}}\Big(\frac{e^{c_{jk}}}{\left(1+e^{c_{jk}}\right)^{2}}
-\frac{e^{{\theta}^*_j-{\theta}^*_i}} {(1+e^{{\theta}^*_j-{\theta}^*_i} )^{2}}\Big)
+\sum_{i>k}{\mathcal{E}_{ki}}\Big(\frac{e^{c_{ki}}}{\left(1+e^{c_{ki}}\right)^{2}}
-\frac{e^{{\theta}^*_j-{\theta}^*_i}} {(1+e^{{\theta}^*_j-{\theta}^*_i} )^{2}}\Big)\bigg|\\
\lesssim &\ \sum_{k>j}{\mathcal{E}_{jk}}\|\hat{{\theta}}-{\theta}^*\| _\infty+\sum_{i>k}{\mathcal{E}_{ki}}\|\hat{{\theta}}-{\theta}^*\| _\infty.
\end{aligned}
\end{equation}
For the off-diagonal entries in the $k$-th row of $M$, we have $$k>j,\  |M_{kj}|=\bigg|-{\mathcal{E}_{jk}}\Big(\frac{e^{c_{jk}}}{\left(1+e^{c_{jk}}\right)^{2}}
-\frac{e^{{\theta}^*_j-{\theta}^*_k}} {(1+e^{{\theta}^*_j-{\theta}^*_k} )^{2}}\Big)\bigg|\lesssim{\mathcal{E}_{jk}}\|\hat{{\theta}}-{\theta}^*\| _\infty,$$
$$i>k,\  |M_{ki}|=\bigg|-{\mathcal{E}_{ki}}\Big(\frac{e^{c_{ki}}}{\left(1+e^{c_{ki}}\right)^{2}}
-\frac{e^{{\theta}^*_k-{\theta}^*_i}} {(1+e^{{\theta}^*_k-{\theta}^*_i} )^{2}}\Big)\bigg|\lesssim{\mathcal{E}_{ki}}\|\hat{{\theta}}-{\theta}^*\| _\infty,$$
which immediately leads to  
\begin{equation}\label{eqn:Minf}
\begin{aligned}
\|M\|_\infty &=\max_{1\leq k\leq n}\{|M_{k1}|+\cdots+|M_{kk}|+\cdots+|M_{kn}|\} \\
&\lesssim
2\sum_{k>j}{\mathcal{E}_{jk}}\|\hat{{\theta}}-{\theta}^*\| _\infty
+2\sum_{i>k}{\mathcal{E}_{ki}}\|\hat{{\theta}}-{\theta}^*\| _\infty
\lesssim np\|\hat{{\theta}}-{\theta}^*\| _\infty,
\end{aligned}
\end{equation}
where the last inequality holds by Lemma \ref{lem:degree_concentration}.  
Putting the previous bounds together, we have
\begin{equation}\nonumber
\begin{aligned}
\|\nabla \cL(\hat{{\theta}})-\nabla \cL({{\theta}}^*)
- \nabla^2 \cL({{\theta}}^*)(\hat{{\theta}}-{\theta}^*)\|_\infty 
\lesssim  \|M\|_\infty \|\hat{{\theta}}-{\theta}^*\|_\infty
\lesssim  np\|\hat{{\theta}}-{\theta}^*\| ^2_\infty
\lesssim  \frac{\log n}{L} 
\end{aligned}
\end{equation}
with probability exceeding $1-\cO (n^{-5})$, where the last inequality follows from Lemma \ref{lem:consistency}.
\end{proof}

\begin{remark}
\label{rmk:grad_conc}
	From the above proof, we have 
	\begin{equation}\label{eqn:grad_conc11}
	\begin{aligned}
	\nabla \cL(\hat{{\theta}})-\nabla \cL({{\theta}}^*)
	=\sum_{i>j}{\mathcal{E}_{ji}}\frac{e^{c_{ji}}}{\left(1+e^{c_{ji}}\right)^{2}}(\bm{e}_i-\bm{e}_j)(\bm{e}_i-\bm{e}_j)^\top(\hat{{\theta}}-{\theta}^*) 
	\end{aligned}
	\end{equation}
	by \eqref{eqn:grad_conc11}, which gives \begin{equation}\nonumber
	\|\nabla \cL(\hat{{\theta}})-\nabla \cL({{\theta}}^*)\|_{\infty} \lesssim np \sqrt{\frac{\log n}{npL}}
	\end{equation}
	with probability exceeding $1-\cO (n^{-5})$.
\end{remark}

\subsection{Proof of Lemma \ref{lem:Hessian_concentartion}}
\label{sec:pf:Hessian_concentartion}
\begin{proof}
	We aim to show, with probability exceeding $1-\cO (n^{-5})$, 
\begin{equation}\nonumber
\| \nabla^2 {\cL(\hat{{\theta})}}-\nabla^2 {\cL({{\theta}}^*)}\|_\infty\lesssim np\sqrt{\frac{\log n}{npL}}.
\end{equation}			
By the closed-form of Hessian in \eqref{eqn:hess}, we have, with probability exceeding $1-\cO (n^{-5})$,
\begin{equation}\nonumber
\begin{aligned}
\| 
\nabla^2 {\cL(\hat{{\theta})}}-\nabla^2 {\cL({{\theta}}^*)}\|_\infty 
\lesssim& \left\|\sum_{i>j}{\mathcal{E}_{ji}}
\left(
\frac{e^{\hat{\theta}_i}e^{\hat{\theta}_j}}
{(e^{\hat{\theta}_i}+e^{\hat{\theta}_j})^2}-\frac{e^{\theta_i^*}e^{\theta_j^*}}{(e^{\theta_i^*}+e^{\theta_j^*})^2}\right)(\bm{e}_i-\bm{e}_j)(\bm{e}_i-\bm{e}_j)^\top \right\|_\infty\\
\lesssim&np \|\hat{{\theta}}-{\theta}^*\| _\infty  
\lesssim np\sqrt{\frac{\log n}{npL}},
\end{aligned}
\end{equation}
where the second inequality  follows from \eqref{eqn:Minf}, and the third inequality follows from Lemma~\ref{lem:degree_concentration}.

\end{proof}

\subsection{Proof of Lemma \ref{lem:eigenvalue_inv_Hessian}}
\label{sec:pf:eigenvalue_inv_Hessian}
\begin{proof}
We aim to show that, with probability exceeding $1-\cO (n^{-10})$, for $\theta = \hat{{\theta}}$ or ${\theta}^*$,
\begin{equation}\label{equ:eigenvalue_inv_Hessian}
\begin{aligned}
\left\|{
	\left( \begin{array}{ccc}
	\nabla^2 {\cL({{{\theta}})}} & \bm{1}\\
	\bm{1}^\top & 0
	\end{array} 
	\right )^{-1}}\right\|_2
=\frac{1}{\lambda_{n-1}\left(\nabla^{2} \cL({{{\theta}}})\right)}
\lesssim \frac{1}{np}.
\end{aligned}
\end{equation}
To prove \eqref{equ:eigenvalue_inv_Hessian}, first, we first bound the eigenvalues of $\nabla^2 {\cL({{{\theta}})}}$ in the following Lemma. The proof can be found in Section~\ref{sec:pf:eigenvalue_Hessian}.


\begin{lemma}\label{lem:eigenvalue_Hessian}
	
	Under the conditions of Theorem \ref{thm:asy}, for all ${\theta}$ satisfying $\|{\theta}-{\theta}^*\|_\infty \leq C$ for some $C>0$, assume $\nabla^2 {\cL({{{\theta}})}}$ has eigenvalues $\lambda_1 \geq\cdots \geq \lambda_{n-1} \geq \lambda_n=0$ with corresponding eigenvectors $\bm{v}_1, \cdots , \bm{v}_{n-1}$ and $\bm{1}$, which  follows from the closed form of $\nabla^2 \cL({\theta})$. We have, with probability at least $1-\cO (n^{-10})$,
	\begin{equation}\nonumber
	{np} \lesssim \lambda_{n-1} \leq \lambda_{1} \lesssim {np}.
	\end{equation}
\end{lemma}

Under the conditions of Lemma \ref{lem:eigenvalue_Hessian},  we study the eigenvalues and eigenvectors of expanded matrix $\Bigl( \begin{smallmatrix} \nabla^2 {\cL({{\theta})}} & \bm{1}\\\bm{1}^\top & 0 \end{smallmatrix} \Bigr)$ based on the eigenvalues and eigenvectors of $\nabla^2 {\cL({{{\theta}})}}$. In particular, 
for $j=1,\cdots,n-1$, $\lambda_j\neq 0$, we have  $$0=\bm{1}^\top\nabla^2 {\cL({{{\theta}})}}\bm{v}_j=\bm{1}^\top\lambda_j\bm{v}_j, $$ where the first equality follows from the closed-form of Hessian in \eqref{eqn:hess}, so we have \begin{equation}\label{eqn:ortho}
\bm{1}^\top\bm{v}_j=0.
\end{equation}
Hence,  
\begin{equation}\nonumber
\begin{aligned}
{
	\left( \begin{array}{ccc}
	\nabla^2 {\cL({{\theta})}} & \bm{1}\\
	\bm{1}^\top & 0
	\end{array} 
	\right )}
{
	\left( \begin{array}{ccc}
	\bm{v}_j \\
	0
	\end{array} 
	\right )}
={
	\left( \begin{array}{ccc}
	\nabla^2 {\cL({{\theta})}}\bm{v}_j\\
	\bm{1}^\top\bm{v}_j
	\end{array} 
	\right )}
=\lambda_j{
	\left( \begin{array}{ccc}
	\bm{v}_j\\
	0
	\end{array} 
	\right )},
\end{aligned}
\end{equation}
which implies $\lambda_j$ and $\Bigl( \begin{smallmatrix} \bm{v}_j\\ 0 \end{smallmatrix} \Bigr)$   are the eigenvalues and the corresponding eigenvectors of $\Bigl( \begin{smallmatrix} \nabla^2 {\cL({{\theta})}} & \bm{1}\\\bm{1}^\top & 0 \end{smallmatrix} \Bigr)$.

Meanwhile, observe that 
\begin{equation}\nonumber
\begin{aligned}
{
	\left( \begin{array}{ccc}
	\nabla^2 {\cL({{\theta})}} & \bm{1}\\
	\bm{1}^\top & 0
	\end{array} 
	\right )}
{
	\left( \begin{array}{ccc}
	\bm{1} \\
	\sqrt{n}
	\end{array} 
	\right )}
={
	\left( \begin{array}{ccc}
	\nabla^2 {\cL({{\theta})}}\bm{1}+\sqrt{n}\bm{1}\\
	\bm{1}^\top \bm{1}
	\end{array} 
	\right )}
=\sqrt{n}{
	\left( \begin{array}{ccc}
	\bm{1} \\
	\sqrt{n}
	\end{array} 
	\right )},
\end{aligned}
\end{equation}
and 
\begin{equation}\nonumber
\begin{aligned}
{
	\left( \begin{array}{ccc}
	\nabla^2 {\cL({{\theta})}} & \bm{1}\\
	\bm{1}^\top & 0
	\end{array} 
	\right )}
{
	\left( \begin{array}{ccc}
	\bm{1} \\
	-\sqrt{n}
	\end{array} 
	\right )}
={
	\left( \begin{array}{ccc}
	\nabla^2 {\cL({{\theta})}}\bm{1}-\sqrt{n}\bm{1}\\
	\bm{1}^\top \bm{1}
	\end{array} 
	\right )}
=-\sqrt{n}{
	\left( \begin{array}{ccc}
	\bm{1} \\
	-\sqrt{n}
	\end{array} 
	\right )},
\end{aligned}
\end{equation}
which implies $\pm\sqrt{n}$, $\bigl( \begin{smallmatrix} \bm{1}\\ \pm\sqrt{n} \end{smallmatrix} \bigr)$  are also the  eigenvalues and corresponding eigenvectors of $\Bigl( \begin{smallmatrix} \nabla^2 {\cL({{\theta})}} & \bm{1}\\\bm{1}^\top & 0 \end{smallmatrix} \Bigr).$

In summary, $\Bigl( \begin{smallmatrix} \nabla^2 {\cL({{\theta})}} & \bm{1}\\\bm{1}^\top & 0 \end{smallmatrix} \Bigr)$ has $n+1$ eigenvalues, which are $\sqrt{n}$, $\lambda_1$, $\cdots$,  $\lambda_{n-1}$, $-\sqrt{n}$ in descending order.  
Thus,
\begin{equation}\nonumber
\begin{aligned}
\left\|{
	\left( \begin{array}{ccc}
	\nabla^2 {\cL({{{\theta}})}} & \bm{1}\\
	\bm{1}^\top & 0
	\end{array} 
	\right )^{-1}}\right\|_2
= \frac{1}{\lambda_{n-1}}
\lesssim  \frac{1}{np},
\end{aligned}
\end{equation}
which completes the proof.
\end{proof}

\begin{remark}\label{rmk:inv}
From the above proof about the eigenvalues of 	$\Bigl( \begin{smallmatrix} \nabla^2 {\cL(\theta)} & \bm{1}\\\bm{1}^\top & 0 \end{smallmatrix} \Bigr)$, we show that $\Bigl( \begin{smallmatrix} \nabla^2 {\cL(\theta)} & \bm{1}\\\bm{1}^\top & 0 \end{smallmatrix} \Bigr)$ is invertible.
\end{remark}
We further have the following corollary about bounding the diagonal entries of $\Bigl( \begin{smallmatrix} \nabla^2 {\cL(\theta)} & \bm{1}\\\bm{1}^\top & 0 \end{smallmatrix} \Bigr)^{-1}$. 
\begin{corollary}\label{lem:eigenvalue_Q}
	Under the conditions of Theorem \ref{thm:asy}, letting 
	\begin{equation}\label{eqn:Theta}
	{
		\left(  \begin{array}{ccc}
		\Theta_{11}  & \frac{1}{n}\bm{1}\\
		\frac{1}{n}\bm{1}^\top & 0
		\end{array} 
		\right) } :=
	{
		\left(  \begin{array}{ccc}
		\nabla^2 {\cL({{\theta})}} & \bm{1}\\
		\bm{1}^\top & 0
		\end{array} 
		\right) ^{-1}},
	\end{equation}
	we have, with probability exceeding $1-\cO( n^{-10})$,
	\begin{equation}\nonumber
	\|\Theta_{11}\|_2 \asymp \frac{1}{np}
	\ \text{ and }\ 
	[\Theta_{11}]_{kk}  \asymp \frac{1}{np} \ \ \text{ for } 1 \leq k \leq n,
	\end{equation}
	where $[\Theta_{11}]_{kk}$ is the $k$-th diagonal entry of $\Theta_{11}$.
\end{corollary}
\begin{proof}
	First, we show that \eqref{eqn:Theta} holds, i.e., we show the off-diagonal block of inverse is $\frac{1}{n}\bm{1}$, and the right bottom block is $0$. 
	Denoting  $\Bigl( \begin{smallmatrix} \Theta_{11} & \bm{\alpha}\\ \bm{\alpha}^\top & \beta \end{smallmatrix} \Bigr)
	=
	\Bigl( \begin{smallmatrix} \nabla^2 {\cL({{\theta})}} & \bm{1}\\\bm{1}^\top & 0 \end{smallmatrix} \Bigr)^{-1}$,
	we have the following equations
	\begin{equation}\nonumber
	\left\{\begin{array}{ll}
	\Theta_{11} \nabla^2 {\cL({{\theta})}} +\bm{\alpha}\bm{1}^\top=I   & (1)\\
	\Theta_{11}\bm{1}=\bm{0} & (2) \\
	\bm{\alpha}^\top \nabla^2 {\cL({{\theta})}} + \beta \bm{1}^\top=\bm{0}^\top & (3) \\
	\bm{\alpha}^\top \bm{1}=1 & (4)
	\end{array}\right.
	\end{equation}
	Then, by $(3)$, we have $(\bm{\alpha}^\top \nabla^2 {\cL({{\theta})}} + \beta \bm{1}^\top)\bm{1}=0+\beta \bm{1}^\top\bm{1}=0$ , which gives $\beta=0$. 
	By $(1)$, we have $(\Theta_{11} \nabla^2 {\cL({{\theta})}} +\bm{\alpha}\bm{1}^\top)\bm{1}=0+n\bm{\alpha}=\bm{1}$ , which gives $\bm{\alpha}=\bm{1}/n$.
Thus, we have
	\begin{equation}\nonumber
	{
		\left(  \begin{array}{ccc}
		\Theta_{11} & \frac{1}{n}\bm{1}\\
		\frac{1}{n}\bm{1}^\top & 0
		\end{array} 
		\right) }=
	{
		\left(  \begin{array}{ccc}
		\nabla^2 {\cL({{\theta})}} & \bm{1}\\
		\bm{1}^\top & 0
		\end{array} 
		\right) ^{-1}}.
	\end{equation}
By the proof of Lemma \ref{lem:eigenvalue_inv_Hessian},   $\Bigl( \begin{smallmatrix} \nabla^2 {\cL({{\theta})}} & \bm{1}\\\bm{1}^\top & 0 \end{smallmatrix} \Bigr)^{-1}$ has $n+1$ eigenvalues   $\frac{1}{\lambda_{n-1}}$, $\cdots$,  $\frac{1}{\lambda_{1}}$, $\frac{1}{\sqrt{n}}$, $-\frac{1}{\sqrt{n}}$ in descending order, and the corressponding eigenvectors are 
	$\Bigl( \begin{smallmatrix} \bm{v}_{n-1}\\ 0 \end{smallmatrix} \Bigr), \cdots, \Bigl( \begin{smallmatrix} \bm{v}_{1}\\ 0 \end{smallmatrix} \Bigr)$,  $\bigl( \begin{smallmatrix} \bm{1}\\ \sqrt{n} \end{smallmatrix} \bigr)$, $\Bigl( \begin{smallmatrix} \bm{1}\\ -\sqrt{n} \end{smallmatrix} \Bigr)$. 
	Following similar argument in the proof of Lemma \ref{lem:eigenvalue_inv_Hessian} in Section \ref{sec:pf:eigenvalue_inv_Hessian}, we immediatly obatin eigenvalues of matrix $\Theta_{11}$ are $\frac{1}{\lambda_{n-1}}$, $\cdots$,  $\frac{1}{\lambda_{1}}$, 0 with eigenvectors  $\bm{v}_{n-1}, \cdots, \bm{v}_{1}$, $\bm{1}$.
	
	Combining the above argument with Lemma \ref{lem:eigenvalue_Hessian} , we have
	\begin{equation}\nonumber
	\|\Theta_{11}\|_2 = \frac{1}{\lambda_{n-1}} \asymp \frac{1}{np},
	\end{equation}
	which immediately gives us 
	\begin{equation}\nonumber
	\max_{k} \ [\Theta_{11}]_{kk}  \lesssim \frac{1}{np}.
	\end{equation} 
	Next, we show that $ [\Theta_{11}]_{kk}  \gtrsim \frac{1}{np}$. Each natural basis $\bm{e}_k$ can be represented as $$\bm{e}_k=\alpha_1\bm{v}_{1}+\cdots+\alpha_{n-1}\bm{v}_{n-1}+\alpha_n\bm{1}.$$
 Combining the above equation with  $\bm{1}^\top\bm{v}_j=0$ from \eqref{eqn:ortho}, we have
 $$1=\bm{e}_k^\top \bm{1} = \alpha_n\bm{1}^\top\bm{1}=\alpha_n n,$$  which gives $\alpha_n=\frac{1}{n}$.
 
	Furthermore, 
	$$1=\bm{e}_k^\top \bm{e}_k = \alpha_1^2\bm{v}_{1}^\top \bm{v}_{1}+\cdots+\alpha_{n-1}^2\bm{v}_{n-1}^\top\bm{v}_{n-1}+\alpha_n^2 \bm{1}^\top\bm{1} =\alpha_1^2+\cdots+\alpha_{n-1}^2+\alpha_n^2 n,$$ which gives $\alpha_1^2+\cdots+\alpha_{n-1}^2=1-\frac{1}{n}$.
	Thus, we have
	\begin{equation}\nonumber
	\begin{aligned}
	[\Theta_{11}]_{kk}
	=&\ \bm{e}_k^\top \Theta_{11} \bm{e}_k
	=(\alpha_1\bm{v}_{1}+\cdots+\alpha_{n-1}\bm{v}_{n-1}+\alpha_n\bm{1})^\top \Theta_{11} (\alpha_1\bm{v}_{1}+\cdots+\alpha_{n-1}\bm{v}_{n-1}+\alpha_n\bm{1}) \\
	=&\ \alpha_1^2\frac{1}{\lambda_1}+\cdots+\alpha_{n-1}^2\frac{1}{\lambda_{n-1}} 
	\geq  (\alpha_1^2+\cdots+\alpha_{n-1}^2)\frac{1}{\lambda_1}  
	\gtrsim  \frac{1}{np},
	\end{aligned}
	\end{equation} 
	which finishes the proof. 
\end{proof}

\begin{remark}
	\label{rmk:1}
	The smallest eigenvalue of $\Theta_{11}$ restricted to vectors orthogonal to $\bm{1}$ is
	$\lambda_{\min , \perp}(\Theta_{11})=\frac{1}{\lambda_{1}}\gtrsim \frac{1}{np}$, so we have $\bm{z}^\top \Theta_{11} \bm{z} \geq \lambda_{\min,\perp}\left(\Theta_{11}\right)\bm{z}^\top\bm{z}$ for all $\bm{z}$ with $\bm{1}^\top\bm{z}=0$.
	We have
	\begin{equation}\nonumber
	(\bm{e}_j-\bm{e}_k)^\top\Theta_{11}(\bm{e}_j-\bm{e}_k) 
	\geq 2\frac{1}{\lambda_{1}} \gtrsim \frac{1}{np},
	\end{equation} 
	and
	\begin{equation}\nonumber
	(\bm{e}_j-\bm{e}_k)^\top \Theta_{11}(\bm{e}_j-\bm{e}_k)
	\leq 2\frac{1}{\lambda_{n-1}} 
	\lesssim \frac{1}{np}.
	\end{equation} 
	Hence, combining with Corollary \ref{lem:eigenvalue_Q}, we  have 
	\begin{equation}\nonumber
	[\Theta_{11}]_{jk} \asymp \frac{1}{np}.
	\end{equation}
\end{remark}

\subsection{Proof of Lemma \ref{lem:qhat-qstar}}
\label{sec:pf:qhat-qstar}
\begin{proof}
We aim to show that, with probability at least $1-\cO(n^{-5})$,
	\begin{equation}\nonumber
\left\| 
{\left( \begin{array}{ccc}
	\nabla^2 {\cL(\hat{{\theta})}} & \bm{1}\\
	\bm{1}^\top & 0
	\end{array} 
	\right )}^{-1}-{
	\left( \begin{array}{ccc}
	\nabla^2 {\cL({{\theta}}^*)} & \bm{1}\\
	\bm{1}^\top & 0
	\end{array} 
	\right )}^{-1} \right\|_2
\lesssim \frac{1}{np}\sqrt{\frac{\log n}{pL}} .
\end{equation}
Following Lemma \ref{lem:Hessian_concentartion} and Lemma \ref{lem:eigenvalue_inv_Hessian}, we have 
\begin{equation}\nonumber
\begin{aligned}
&\left\| 
{\left( \begin{array}{ccc}
	\nabla^2 {\cL(\hat{{\theta}})} & \bm{1}\\
	\bm{1}^\top & 0
	\end{array} 
	\right )}^{-1}-{
	\left( \begin{array}{ccc}
	\nabla^2 {\cL({{\theta}}^*)} & \bm{1}\\
	\bm{1}^\top & 0
	\end{array} 
	\right )}^{-1} \right\|_2\\
 &\qquad \leq \left\| 
{\left( \begin{array}{ccc}
	\nabla^2 {\cL(\hat{{\theta}})} & \bm{1}\\
	\bm{1}^\top & 0
	\end{array} 
	\right )}^{-1} \right\|_2
\Bigg\| 
{\left( \begin{array}{ccc}
	\nabla^2 {\cL(\hat{{\theta}})} & \bm{1}\\
	\bm{1}^\top & 0
	\end{array} 
	\right )}-{
	\left( \begin{array}{ccc}
	\nabla^2 {\cL({{\theta}}^*)} & \bm{1}\\
	\bm{1}^\top & 0
	\end{array} 
	\right )} \Bigg\|_2
\left\| 
{\left( \begin{array}{ccc}
	\nabla^2 {\cL({{\theta}^*)}} & \bm{1}\\
	\bm{1}^\top & 0
	\end{array} 
	\right )}^{-1} \right\|_2\\
 &\qquad \lesssim \Big(\frac{1}{np}\Big)^2 \times \sqrt{n}\big\| 
\nabla^2 {\cL(\hat{{\theta}})}-\nabla^2 {\cL({{\theta}}^*)}\big\|_\infty   
\lesssim \Big(\frac{1}{np}\Big)^2 \times \sqrt{n}\times np\sqrt{\frac{\log n}{npL}}
\lesssim \dfrac{1}{np}\sqrt{\frac{\log n}{pL}},
\end{aligned}
\end{equation}
which concludes the proof.
\end{proof}
\begin{remark}\label{rmk:eqn:4} 
	Note that from the above proof, we immediately have \begin{equation}\nonumber
	\big\|\hat{\Theta}_{11}-{\Theta}^*_{11}\big\|_2\lesssim \dfrac{1}{np}\sqrt{\frac{\log n}{pL}}.
	\end{equation}
\end{remark}

\subsection{Proof of Lemma \ref{lem:eigenvalue_Hessian}}
\label{sec:pf:eigenvalue_Hessian}
\begin{proof} 
	For the control of $\lambda_{n-1}$,  we refer the readers to the proof of Lemma 9 and Lemma 10 in \cite{chen2019spectral}. By these two lemmas, for all ${\theta}$ satisfying $\|{\theta}-{\theta}^*\|_\infty \leq C$ for some $C>0$, $\lambda_{n-1}=\lambda_{\min , \perp}\left(\nabla^{2} \cL({{{\theta}}})\right)  \gtrsim np$ with probability exceeding $1-\cO( n^{-10})$, which also shows that graph $\mathcal{G}$ is connected.
	
	To control $\lambda_1$,  we first decompose $\nabla^2 {\cL({\theta})}$ as 
	$$\begin{aligned}
	\nabla^2 {\cL({\theta})}= &\sum_{i>j}{\mathcal{E}_{ji}}\frac{e^{\theta_i}e^{\theta_j}}{(e^{\theta_i}+e^{\theta_j})^2}(\bm{e}_i-\bm{e}_j)(\bm{e}_i-\bm{e}_j)^\top\\
	=&\underbrace{\sum_{i>j}{\mathcal{E}_{ji}}\frac{1}{4}(\bm{e}_i-\bm{e}_j)(\bm{e}_i-\bm{e}_j)^\top}_{M}
	 - \underbrace{\sum_{i>j}{\mathcal{E}_{ji}}\left(\frac{1}{4} - \frac{e^{\theta_i}e^{\theta_j}}{(e^{\theta_i}+e^{\theta_j})^2}\right)(\bm{e}_i-\bm{e}_j)(\bm{e}_i-\bm{e}_j)^\top}_{P},
	\end{aligned}$$
	where $M$ has eigenvalues $\mu_1 \geq\cdots \geq \mu_{n-1} \geq \mu_n=0$ and $P$ has eigenvalues $\rho_1 \geq\cdots \geq \rho_{n-1} \geq \rho_n=0.$\\
	By Weyl’s Monotonicity Theorem (Corollary 2.3 in \cite{zheng2020inertia}), we have 
	$\lambda_1\leq \mu_1.$ 
	
We then bound $\mu_1$.	Since $\mathbb{E}[M]=\mathbb{E} \sum_{i>j}{\mathcal{E}_{ji}}\frac{1}{4}(\bm{e}_i-\bm{e}_j)(\bm{e}_i-\bm{e}_j)^\top=\frac{1}{4}p\sum_{i>j}(\bm{e}_i-\bm{e}_j)(\bm{e}_i-\bm{e}_j)^\top$, the largest eigenvalue of $\mathbb{E}[M]$ is $\frac{1}{4}np.$
	By matrix Chernoff inequality (Theorem 5.1.1 in \cite{tropp2015introduction}), we obtain that
	$\mu_1 \leq 3\times\frac{1}{4}np$  with probability exceeding $1-\cO( n^{-10})$ as long as  $p>\frac{C_0\log n}{n}$ for some   constant $C_0>0$.

	Combining the results above, we have  $\lambda_1\leq \mu_1  \lesssim np$  with probability at least $1-\cO (n^{-10})$.
\end{proof}

\subsection{Berry-Esseen Theorem}
\label{sec:CLT}
We provide the Berry-Esseen Theorem here for self-completeness. Proof can be found in Theorem 5.4 of \cite{petrov1995limit}.
\begin{lemma}[Berry-Esseen Theorem]
	\label{lem:BE}
Assume $X_{1}, \ldots, X_{n}$ are independent random variables with $\mathbb{E}\left(X_{i}\right)=0, \mathbb{E}\left(X_{i}^{2}\right)=\sigma_{i}^{2}$ and $\mathbb{E}\left(\left|X_{i}\right|^{3}\right)=\rho_{i} .$ Then it holds that
$$
\sup _{x \in \mathbb{R}}\left|\mathbb{P}\left(\frac{1}{s_{n}} \sum_{i=1}^{n} X_{i} \leqslant x\right)-\Phi(x)\right| \leqslant C\left(\sum_{i=1}^{n} \sigma_{i}^{2}\right)^{-3 / 2}\left(\sum_{i=1}^{n} \rho_{i}\right)
$$
where $C>0$ is a constant and $s_{n}^{2}=\sum_{i=1}^{n} \sigma_{i}^{2}$.

\end{lemma}


%
